\documentclass{article}

\PassOptionsToPackage{numbers}{natbib}
\usepackage[preprint]{neurips_2025}

\usepackage[utf8]{inputenc} 
\usepackage[T1]{fontenc}    
\usepackage{hyperref}       
\usepackage{url}            
\usepackage{booktabs}       
\usepackage{amsfonts}       
\usepackage{nicefrac}       
\usepackage{microtype}      
\usepackage{xcolor}  
\usepackage{enumitem}
\usepackage{xargs}
\usepackage{amsmath}
\usepackage{cleveref}
\usepackage{amssymb}
\usepackage{amsthm}
\usepackage{tikz-cd}
\usepackage{comment}
\usepackage{bbm}
\usepackage{enumitem}
\usepackage{mathtools}
\usepackage{autonum}

\usepackage{multirow}
\usepackage{booktabs}

\title{Statistical inference for Linear Stochastic Approximation with Markovian Noise}

\author{%
  Sergey Samsonov\\
  HSE University \\
  svsamsonov@hse.ru \\
  \And
  Marina Sheshukova\\
  HSE University\\
  msheshukova@hse.ru
  \AND
  Eric Moulines\\
  Ecole Polytechnique,\\
  MBUZAI \\
  eric.moulines@polytechnique.edu
  \And
  Alexey Naumov\\
  HSE University,\\
  Steklov Mathematical Institute \\
  of Russian Academy of Sciences \\
  anaumov@hse.ru
}

\newtheorem{assum}{A\hspace{-2pt}}

\newtheorem{theorem}{Theorem}
\crefname{theorem}{theorem}{Theorems}
\Crefname{theorem}{Theorem}{Theorems}

\newtheorem{lemma}{Lemma}
\crefname{lemma}{lemma}{lemmas}
\Crefname{lemma}{Lemma}{Lemmas}

\newtheorem{remark}{Remark}
\crefname{remark}{remark}{remarks}
\Crefname{remark}{Remark}{Remarks}

\newtheorem{corollary}{Corollary}
\crefname{corollary}{corollary}{corollaries}
\Crefname{corollary}{Corollary}{Corollaries}

\newtheorem{proposition}{Proposition}
\crefname{proposition}{proposition}{propositions}
\Crefname{proposition}{Proposition}{Propositions}

\newtheorem{definition}{Definition}
\crefname{definition}{definition}{definitions}
\Crefname{Definition}{Definition}{Definitions}

\crefname{example}{example}{examples}
\Crefname{Example}{Example}{Examples}

\crefname{figure}{figure}{figures}
\Crefname{Figure}{Figure}{Figures}

\crefname{table}{table}{tables}
\Crefname{Table}{Table}{Tables}

\crefname{assum}{A\hspace{-2pt}}{A\hspace{-2pt}}
\crefname{assumb}{B\hspace{-2pt}}{B\hspace{-2pt}}
\crefname{assumUGE}{UGE\hspace{-1pt}}{UGE\hspace{-1pt}}
\crefname{assumID}{IND\hspace{-1pt}}{IND\hspace{-1pt}}
\crefname{assumUE}{UE\hspace{-1pt}}{UE\hspace{-1pt}}
\crefname{assumSUP}{M\hspace{-1pt}}{M\hspace{-1pt}}

\newlist{renumerate}{enumerate}{3}
\setlist[renumerate]{wide, labelwidth=!, labelindent=0pt,label=(\roman*)}

\newlist{aenumerate}{enumerate}{3}
\setlist[aenumerate]{wide, labelwidth=!, labelindent=0pt,label=(\arabic*)}

\newlist{aaenumerate}{enumerate}{3}
\setlist[aaenumerate]{wide, labelwidth=!, labelindent=0pt,label=(\alph*)}

\newlist{aenumerateSpace}{enumerate}{3}
\setlist[aenumerateSpace]{wide, labelwidth=!,label=(\arabic*)}

\newlist{benumerate}{enumerate}{3}
\setlist[benumerate]{wide, labelwidth=!, labelindent=0pt,label=$\bullet$}

\newtheorem{assumprime}{\textbf{A'}\hspace{-1pt}}
\Crefname{assumprime}{\textbf{A'}\hspace{-1pt}}{\textbf{A'}\hspace{-1pt}}
\crefname{assumprime}{\textbf{A'}}{\textbf{A'}}
\def\supconsteps{\supnorm{\funnoisew}}

\def\bbeta{\boldsymbol\beta}
\newcommand{\PE}{\mathbb{E}}

\def\mart{\mathcal{M}}
\def\boundconstmart{\varkappa}
\def\Sphere{\mathcal{S}}

\newcommand{\PP}{\mathbb{P}}
\newcommandx{\genericb}[1][1=]{b_{#1}}
\newcommandx{\ConstK}[1][1=]{\operatorname{C}_{\operatorname{K},#1}}
\newcommandx{\Constros}[1][1=]{\operatorname{C}_{\operatorname{Ros},#1}}
\newcommandx{\Constburk}[1][1=]{\operatorname{C}_{\operatorname{Burk}}}
\newcommandx{\driftW}[1][1=]{W_{#1}}

\newcommandx{\metricd}[1][1=]{\mathsf{d}_{#1}}

\newcommandx\invmeasure[1][1=]{\Pi_{#1}}
\newcommandx{\PPjoint}[1][1=]{\PP^{\MKjoint[#1]}}
\newcommandx{\PEjoint}[1][1=]{\PE^{\MKjoint[#1]}}
\newcommandx{\PEMID}[1][1=\alpha]{\PE^{\MK[#1]}}
\newcommandx{\PPMID}[1][1=\alpha]{\PP^{\MK[#1]}}

\newcommand{\supnorm}[1]{\norm{ #1 }[\infty]}
\newcommandx{\MKjoint}[1][1=]{\bar{\operatorname{P}}_{#1}}
\newcommandx\costw[1][1=]{\mathsf{c}_{#1}}

\newcommandx\Intergrdist[1][1=]{\mathbb{M}_{1}(#1)}

\newcommandx{\mmarkov}[1][1=0]{m^{(\Markov)}_{#1}}

\def\Conv{\operatorname{Conv}}
\def\seta{\mathcal{A}}

\def\F{\mathcal{F}}

\def\Zset{\mathsf{Z}}
\def\Zsigma{\mathcal{Z}}

\def\rset{\mathbb{R}}

\def\nset{\ensuremath{\mathbb{N}}}
\def\nsets{\ensuremath{\mathbb{N}^*}}

\newcommand{\round}[1]{\ensuremath{\lfloor#1\rfloor}}

\newcommand{\msi}{\mathsf{I}}

\newcommand{\Mat}[1]{{\bf{#1}}}
\def\MatB{B}
\def\MatP{P}

\renewcommand{\S}{\mathcal{S}}
\newcommand{\A}{\mathcal{A}}

\def\PMDP{\MKQ}

\newcommand{\Const}[1]{\operatorname{C}_{{#1}}}

\newcommand{\Constupd}[1]{\operatorname{C}_{{#1}}^{\mathsf{D}}}

\newcommand{\Constupdboot}[1]{\operatorname{C}_{{#1}}^{\mathsf{D,b}}}

\newcommand{\bConst}[1]{\operatorname{C}_{{\bf #1}}}

\newcommandx\sequence[4][2=,3=,4=]
{\ifthenelse{\equal{#3}{}}{\ensuremath{\{ #1_{#2 #4}\}}}{\ensuremath{\{ #1_{#2 #4} \}_{#2 \in #3}}}}

\newcommandx\sequenceD[2][2=]
{\ifthenelse{\equal{#2}{}}{\ensuremath{\{ #1\}}}{\ensuremath{\{ #1\!~:\!~#2  \}}}}
\newcommandx\sequenceDouble[4][3=,4=]
{\ifthenelse{\equal{#3}{}}{\ensuremath{\{ (#1_{#3},#2_{#3})\}}}{\ensuremath{\{ (#1_{#3},#2_{#3})\}_{#3 \in #4}}}}

\newcommandx{\sequencen}[2][2=n\in\nset]{\ensuremath{\{ #1, \eqsp #2 \}}}
\newcommandx\sequencens[2][2=n]
{\ensuremath{\{ #1_{#2} \!~:\!~#2\in\nsets\}}}
\newcommandx\sequencet[4]
{\ensuremath{\{ #1{#2_{#3}} \, : \, \eqsp #3 \in #4 \}}}
\def\PE{\mathbb{E}}
\def\P{\mathbb{P}}
\def\ProdB{\Gamma}

\newcommandx{\PVar}[1][1=]{\ensuremath{\operatorname{Var}_{#1}}}

\def\noisecov{\Sigma_\varepsilon}

\newcommandx{\MK}[1][1=\alpha]{\mathrm{P}_{#1}}
\newcommandx\MKK[1][1=\alpha]{\mathrm{K}_{#1}}
\def\MKQ{\mathrm{P}}

\newcommandx{\PEtilde}[1][1=]{\PE^{\mathrm{K}_{#1}}}
\newcommandx{\PPtilde}[1][1=]{\PP^{\mathrm{K}_{#1}}}

\newcommandx{\norm}[2][2=]{\Vert#1 \Vert_{{#2}}}
\newcommandx{\normLigne}[2][2=]{\Vert#1 \Vert_{{#2}}}
\newcommandx{\normLine}[2][2=]{\Vert#1 \Vert_{{#2}}}
\newcommandx{\normop}[2][2=]{\Vert{#1}\Vert_{{#2}}}
\newcommandx{\normopLigne}[2][2=]{\Vert{#1}\Vert_{{#2}}}
\newcommandx{\normopLine}[2][2=]{\Vert{#1}\Vert_{{#2}}}
\newcommandx{\osc}[2][1=]{\mathrm{osc}_{#1}(#2)}

\newcommandx{\normlip}[2][2=\operatorname{Lip}]{\Vert#1 \Vert_{{#2}}}
\newcommand{\lip}{\operatorname{L}}
\newcommandx{\lipspace}[1]{\lip_{#1}}

\newcommandx{\CPP}[3][1=]
{\ifthenelse{\equal{#1}{}}{{\mathbb P}\left(\left. #2 \, \right| #3 \right)}{{\mathbb P}_{#1}\left(\left. #2 \, \right | #3 \right)}}
\newcommandx{\CPPtilde}[3][1=]
{\ifthenelse{\equal{#1}{}}{{\tilde{\mathbb P}}\left(\left. #2 \, \right| #3 \right)}{{\tilde{\mathbb P}}_{#1}\left(\left. #2 \, \right | #3 \right)}}

\def\iid{i.i.d.}

\newcommandx{\as}[1][1=\PP]{\ensuremath{#1\, -\mathrm{a.s.}}}

\newcommand{\eqsp}{\;}

\newcommand{\Id}{\mathrm{I}}
\def\prtheta{\bar{\theta}}

\def\utheta{\tilde{\theta}^{\sf (tr)}}
\def\vtheta{\tilde{\theta}^{\sf (fl)}}

\newcommand{\ConstD}{\mathsf{D}}
\newcommand{\ConstDM}{\ConstD^{(\operatorname{M})}}

\newcommandx{\boundmetric}[1][1=]{\kappa_{\MKK[#1]}}

\newcommand{\Jnalpha}[2]{J_{#1}^{(#2)}}
\newcommand{\Hnalpha}[2]{H_{#1}^{(#2)}}

\newcommand{\ocintLine}[1]{(#1]}

\newcommandx{\Nnorm}[2][1=V]{[ #2]_{#1}}
\newcommandx{\lipnorm}[2][1=g]{[ #1]_{#2}}

\newcommandx{\CPE}[3][1=]{{\mathbb E}^{#3}_{#1}\left[#2\right]}
\newcommandx{\CPEext}[3][1=]{\tilde{\mathbb E}^{#3}_{#1}\left[#2\right]}
\newcommandx{\CPEtilde}[3][1=]{{\tilde{\mathbb E}}^{#3}_{#1}\left[#2\right]}
\newcommandx{\CPEs}[3][1=]{{\mathbb E}^{#3}_{#1}[#2]}
\def\ProdBa{\ProdB^{(\alpha)}}

\def\thetalim{\theta^\star}

\newcommand{\rme}{\mathrm{e}}
\newcommand{\rmd}{\mathrm{d}}
\def\funcAw{\mathbf{A}}
\def\funcBw{\mathbf{B}}
\newcommand{\funcA}[1]{\funcAw(#1)}

\def\funcbw{\mathbf{b}}
\newcommand{\funcb}[1]{\funcbw(#1)}
\newcommandx{\zmfuncA}[2][1=]{\tilde{\funcAw}^{#1}(#2)}
\newcommandx{\zmfuncAw}[1][1=]{\tilde{\funcAw}_{#1}}
\newcommandx{\zmfuncb}[2][1=]{\tilde{\funcbw}^{#1}(#2)}
\def\funnoisew{\varepsilon}
\newcommand{\funcnoise}[1]{\funnoisew(#1)}

\newcommandx{\funcct}[2][1=]{\funcctilde^{#1}(#2)}

\def\qcond{\kappa_{Q}}

\def\State{Z}

\def\taumix{t_{\operatorname{mix}}}

\newcommand{\1}{\boldsymbol{1}}

\newcommandx{\CovC}[1][1=u]{\operatorname{C}_{#1}}

\def\msz{\mathsf{Z}}
\def\mcz{\mathcal{Z}}

\def\plusinfty{+\infty}

\DeclareMathAlphabet{\mathpzc}{OT1}{pzc}{m}{it}

\def\lyapW{\mathpzc{W}}

\newcommandx{\bias}[1][1=\alpha]{\operatorname{B}_{#1}}

\newcommandx\probaMarkovTilde[2][2=]
{\ifthenelse{\equal{#2}{}}{{\widetilde{\mathbb{P}}_{#1}}}{\widetilde{\mathbb{P}}_{#1}\left[ #2\right]}}

\def\mcf{\mathcal{F}}

\newcommand{\indi}[1]{\1_{#1}}

\def\bA{\bar{\mathbf{A}}}

\def\X{{\bf X}}
\def\Y{{\bf Y}}

\def\thetas{\thetalim}
\def\sphere{\mathbb{S}}

\def\Am{{\bf A}}
\def\bm{{\bf b}}
\def\funcctilde{\tilde{c}_u}

\def\RemCov{\mathcal{R}}

\def\barb{\bar{\mathbf{b}}}
\newcommandx{\driftb}[1][1=p]{\bar{b}_{#1}}

\def\Zbf{\mathbf{Z}}

\def\eps{\varepsilon}

\def\tvdist{\mathsf{d}_{\operatorname{tv}}}

\newcommandx{\boldb}[1][1={q}]{\mathsf{b}_{#1}}

\newcommandx{\ConstGW}[1][1={n,\lyapW}]{\operatorname{G}_{#1}}

\newcommandx{\ConstMW}[1][1={n,\lyapW}]{\operatorname{M}_{#1}}

\newtheorem{assumTD}{\textbf{TD}\hspace{-1pt}}
\Crefname{assumTD}{\textbf{TD}\hspace{-1pt}}{\textbf{TD}\hspace{-1pt}}
\crefname{assumTD}{\textbf{TD}}{\textbf{TD}}

\Crefname{assumptionC}{\textbf{C}\hspace{-1pt}}{\textbf{C}\hspace{-1pt}}
\crefname{assumptionC}{\textbf{C}}{\textbf{C}}

\Crefname{assumptionM}{\textbf{UGE}\hspace{-1pt}}{\textbf{UGE}\hspace{-1pt}}
\crefname{assumptionM}{\textbf{UGE}}{\textbf{UGE}}

\def\distance{\mathsf{d}}

\newcommandx{\vartconstwas}[1][1=V]{c_{#1}}

\newcommandx{\deltawas}[1][1=*]{\delta_{#1}}

\newcommandx{\wasser}[4][1=\distance,4=]{\mathbf{W}_{#1}^{#4}\left(#2,#3\right)}
\newcommandx{\covcoeff}[2]{\rho_{#1}^{(#2)}}

\newcommand{\dobrush}{\mathsf{\Delta}}
\newcommandx{\dobru}[3][1=,3=]{\dobrush_{#1}^{#3}( #2)}  

\def\mtt{\mathtt{m}}

\def\qexponent{q}
\def\ppexponent{p}
\def\Markov{\mathrm{M}}

\newcommandx{\dlim}[1]{\ensuremath{\stackrel{#1}{\Longrightarrow}}}

\def\boot{\mathsf{b}}
\newcommand{\PPb}{\mathbb{P}^\boot}
\newcommand{\PEb}{\mathbb{E}^\boot}

\newcommand{\kolmogorov}[1]{\mathsf{d}_{K}\bigl(#1\bigr)}

\def\tmix{{\boldsymbol{\tau}}}

\begin{document}

\maketitle

\begin{abstract}
In this paper we derive non-asymptotic Berry–Esseen bounds for Polyak–Ruppert averaged iterates of the Linear Stochastic Approximation (LSA) algorithm driven by the Markovian noise. Our analysis yields $O(n^{-1/4})$ convergence rates to the Gaussian limit in the Kolmogorov distance. We further establish the non-asymptotic validity of a multiplier block bootstrap procedure for constructing the confidence intervals, guaranteeing consistent inference under Markovian sampling. Our work provides the first non-asymptotic guarantees on the rate of convergence of bootstrap-based confidence intervals for stochastic approximation with Markov noise. Moreover, we recover the classical rate of order $\mathcal{O}(n^{-1/8})$ up to logarithmic factors for estimating the asymptotic variance of the iterates of the LSA algorithm.  
\end{abstract}

\section{Introduction}
\label{sec:intro}
 Stochastic approximation (SA) 
 has become foundational to modern machine learning, especially its reinforcement learning (RL) domain. Many classical RL algorithms, including $Q$-learning \cite{watkins:qlearn:1992,sutton:book:2018}, the actor–critic algorithm \cite{konda1999actor}, and policy evaluation algorithms, such as temporal difference (TD) learning  \cite{sutton:book:2018} are special instances of SA. Recent research has extensively studied both the asymptotic \cite{borkar:sa:2008} and non-asymptotic behavior \cite{moulines2011non} of these algorithms. It is important not only to establish the convergence of SA estimators, but also to quantify their uncertainty, which is typically done through the asymptotic normality of corresponding estimates, see \cite{polyak1992acceleration,mokkadem2006convergence}. Recent works have focused on deriving non-asymptotic convergence rates for SA methods \cite{shao2022berry,samsonov2024gaussian,wu2024statistical}. Notably, most existing results consider settings with independent and identically distributed (\iid) noise. In contrast, many practical SA applications involve dependent noise, often forming a Markov chain. This additionally complicates the problem. Indeed, even the problem of deriving precise Berry–Esseen type convergence rates for additive functionals of Markov chains is  challenging compared to \iid\ setting, where quantitative results are well-established, starting from Bentkus' influential work \cite{bentkus2004}.
\par 
Asymptotic normality of SA estimates is particularly important in practice, as it allows one to construct approximate confidence intervals for the parameters of interest. Different approaches either directly utilize the asymptotic normality and aim to estimate the asymptotic covariance matrix directly (the plug-in or batch-mean methods, see \cite{chen2020aos,chen2021statistical,roy2023online}), or rely on the non-parametric methods, based on the bootstrap \cite{efron1992bootstrap,rubin1981bayesian}. When the latter approach is applied to the dependent observations, standard bootstrap methods need to be carefully adjusted to account for the dependence structure, see \cite{Fang2018,liu2023statistical}.
\par 
In this paper, we study a linear stochastic approximation (LSA) procedure. This setting covers several important scenarios, such as the classical TD learning algorithm \cite{sutton:book:2018} with linear function approximation, while allowing for sharper theoretical analysis. The LSA procedure aims to find an approximate solution to the linear system $\bA \thetalim = \barb$ with a unique solution $\thetalim$, based on observations $\{( \funcA{Z_k}, \funcb{Z_k})\}_{k \in \nset}$. Here $\Am: \msz \to \rset^{d \times d}$ and $\bm: \msz \to \rset^d$ are measurable functions, and $(Z_k)_{k \in \nset}$ is a sequence of noise variables taking values in a measurable space $(\msz,\mcz)$ with a distribution $\pi$ satisfying $\PE_{\pi}[ \funcA{Z_k} ] = \bA$ and $\PE_{\pi}[ \funcb{Z_k} ] = \barb$. We focus on the setting where $\{\State_k\}_{k \in \nset}$ is an ergodic Markov chain. Given a sequence of decreasing step sizes $(\alpha_k)_{k \in \nset}$ and an initial point $\theta_0 \in \rset^{d}$, we consider the estimates $\{ \prtheta_{n} \}_{n \in \nset}$ given by
\begin{equation}
\label{eq:lsa}
\textstyle 
\theta_{k} = \theta_{k-1} - \alpha_{k} \{ \funcA{Z_k} \theta_{k-1} - \funcb{Z_k} \} \eqsp,~~ k \geq 1, \quad \prtheta_{n} = n^{-1} \sum_{k=0}^{n-1} \theta_k \eqsp, ~~n \geq 1 \eqsp.
\end{equation}
The sequence $\{ \theta_k \}_{k \in \nset}$ corresponds to the standard LSA iterates, while $\{ \prtheta_{n} \}_{n \in \nset}$ corresponds to the Polyak-Ruppert (PR) averaged iterates \cite{ruppert1988efficient,polyak1992acceleration}. It is known that, under appropriate technical conditions on the step sizes $\{\alpha_k\}$ and the noisy observations $\{Z_k\}$ (see \cite{polyak1992acceleration} and \cite{fort:clt:markov:2015} for a discussion),
\begin{equation}
\label{eq:CLT_fort_prelim}
\textstyle 
\sqrt{n}(\bar{\theta}_{n} - \thetas) \overset{d}{\rightarrow} \mathcal{N}(0,\Sigma_{\infty})\eqsp.
\end{equation}
Here the asymptotic covariance matrix $\Sigma_{\infty}$ is defined in \Cref{sec:normal_approximation_main}; see \eqref{eq:sigma_infty_def}. Recent works have provided a number of non‐asymptotic guarantees for the averaged LSA iterates of $\prtheta_{n}$, in particular, \cite{lakshminarayanan:2018a, srikant:1tsbounds:2019,mou2020linear,durmus2021tight,mou2021optimal,durmus2022finite}, which study the mean-squared error, high-order moment bounds, and concentration bounds for $\sqrt{n}(\bar{\theta}_{n} - \thetas)$. However, these results are not enough to establish explicit convergence rates in \eqref{eq:CLT_fort_prelim} in some appropriate probability metric $\mathsf{d}$:
\begin{equation}
\label{eq:berry-esseen}
\mathsf{d}(\sqrt{n} (\bar{\theta}_{n} - \thetas),  \Sigma_{\infty}^{1/2} Y)\eqsp,
\end{equation}
where $Y \sim \mathcal N(0, \Id)$. Notable exceptions are recent papers \cite{samsonov2024gaussian,wu2024statistical,sheshukova2025gaussian,wu2025uncertainty}, which study non-asymptotic convergence rates in the LSA and stochastic gradient descent (SGD) algorithms, as well as in the specific setting of the TD learning algorithm. Our paper aims to complement these findings by providing both a non-asymptotic analysis of the convergence rate in \eqref{eq:CLT_fort_prelim} and an analysis of an appropriate procedure for constructing confidence sets for $\thetas$. Our primary contributions are as follows:
\begin{itemize}[noitemsep,topsep=0pt,leftmargin=3em]
\item We derive a novel non-asymptotic bound for normal approximation for projected Polyak–Ruppert averaged LSA iterates $\sqrt{n}u^{\top}(\bar{\theta}_{n} - \thetas)$ under Markovian noise in \eqref{eq:CLT_fort_prelim}. Here $u$ is a vector on the unit sphere $\Sphere_{d-1}$. Precisely, we establish a convergence rate of order $\mathcal{O}(n^{-1/4})$ in Kolmogorov distance. The rate $\mathcal{O}(n^{-1/4})$ matches the recent result of \cite{wu2025uncertainty}, which considered the particular setting of the TD learning algorithm, and improves over the previous results from \cite{srikant2024rates}. Our proof strategy differs from that of \cite{wu2025uncertainty} and is based on a Poisson decomposition for Markov chains, combined with an appropriate version of the Berry–Esseen bound for martingales, building on the results of Fan \cite{fan2019exact} and Bolthausen \cite{bolthausen1982martingale}.
\item We provide a non-asymptotic analysis of the multiplier subsample bootstrap approach \cite{MR4733879} for the LSA algorithm under Markovian noise. Our bounds imply that the coverage probabilities of the true parameter $\thetas$ can be approximated at a rate of $\mathcal{O}(n^{-1/10})$, where $n$ is the number of samples used in the procedure. To the best of our knowledge, this is the first non-asymptotic bound on the accuracy of bootstrap approximation for SA algorithms with Markov noise. As a byproduct of our analysis, we also recover the rate of $\mathcal{O}(n^{-1/8})$ (up to logarithmic factors) for estimating the asymptotic variance of projected LSA iterates using the overlapping batch means (OBM) estimator, previously obtained in \cite{chen2020aos,zhu2023online_cov_matr}. 
\item We apply the proposed methodology to the temporal difference learning (TD) algorithm for
policy evaluation in reinforcement learning. 
\end{itemize} 
The rest of the paper is organized as follows. In \Cref{sec:related-work}, we review the literature on Berry–Esseen type results for stochastic approximation (SA) algorithms and methods for constructing confidence intervals for the parameter $\thetas$ in these settings, with particular emphasis on the LSA. In \Cref{sec:normal_approximation_main}, we obtain a quantitative version of the Berry–Esseen theorem for the projected error of the Polyak–Ruppert averaged LSA under the Kolmogorov distance. In \Cref{sec:bootstrap}, we discuss the multiplier subsample bootstrap approach (as proposed in \cite{MR4733879}) for LSA and provide corresponding non-asymptotic bounds on the accuracy of approximating the exact distribution of $\sqrt{n}u^{\top}(\bar{\theta}_{n} - \thetas)$ with its bootstrap counterpart. Then in \Cref{sec:applications_td} we apply the proposed methodology to the temporal difference learning (TD) algorithm, based on the sequence of states, which form a geometrically ergodic Markov chain. Proofs are postponed to appendix.
\par 
\textbf{Notations.} We denote by $\mathcal{P}(\msz)$ the set of probability measures on a measurable space $(\msz,\mcz)$. For probability measures $\mu$ and $\nu$ on $(\msz,\mcz)$ we denote by $\tvdist(\mu,\nu)$ the total variation distance between them $\mu$ and $\nu$, that is, $\tvdist(\mu,\nu) = \sup_{C \in \mcz}|\mu(C) - \nu(C)|$. For matrix $A \in \rset^{d \times d}$ we denote by $\norm{A}$ its operator norm. Given a sequence of matrices $\{A_{\ell}\}_{\ell \in \nset}$, $A_{\ell} \in \rset^{d \times d}$, we use the following convention for matrix products: $\prod_{\ell=m}^{k}A_{\ell} = A_{k} A_{k-1} \ldots A_{m}$, where $m \leq k$. For the matrix $Q \in \rset^{d \times d}$, which is symmetric and positive-definite, and $x \in \rset^{d}$, we define the corresponding norm $\|x\|_Q = \sqrt{x^\top Q x}$, and define the respective matrix $Q$-norm of the matrix $B \in \rset^{d \times d}$ by $\normop{B}[Q] = \sup_{x \neq 0} \norm{Bx}[Q]/\norm{x}[Q]$. For sequences $a_n$ and $b_n$, we write $a_n \lesssim b_n$ if there exist a (numeric) constant $c > 0$ such that $a_n \leq c b_n$. For simplicity of presentation, we state the main results of the paper up to absolute constants. We also use a notation $a_n \lesssim_{\log{n}} b_n$, if there exists $\beta>0$, such that $a_n \leq c (1+\log{n})^{\beta} b_n$ fo any $n$. Additionally, we use the standard abbreviations "\iid\," for "independent and identically distributed" and "w.r.t." for "with respect to".

\vspace{-10pt}
\section{Related works}
\label{sec:related-work}
The analysis of linear stochastic approximation (LSA) algorithms under Markovian noise has a long history, with classical works establishing almost sure convergence and asymptotic normality under broad conditions \cite{polyak1992acceleration,borkar:sa:2008,kushner2003stochastic,benveniste2012adaptive}. These asymptotic results, however, lack explicit finite-sample error bounds. Recently, non-asymptotic performance analysis of stochastic approximation and gradient methods has gained attention. In the i.i.d.~setting of stochastic gradient descent (SGD), finite-time error bounds were provided by \cite{nemirovski2009robust,rakhlin2012making}. For Markovian LSA contexts, convergence rates were analyzed by \cite{bhandari2018finite,lakshminarayanan:2018a}, focusing on temporal-difference learning, and instance-dependent bounds were derived by \cite{mou2021optimal} and \cite{durmus2022finite}. These studies establish guarantees on mean-square error or high-probability deviations but do not address the distributional approximation of the estimator.
\par 
Rates of convergence in CLT are widely studied in probability theory \cite{bentkus2004}, primarily for sums of random variables or univariate martingale difference sequences \cite{petrov1975sums, bolthausen1982martingale}. Berry-Esseen bounds and Edgeworth expansions for Markov chains under various ergodicity assumptions are also available \cite{bolthausen1980berry,bolthausen1982markov,bertail2004edgeworth,kloeckner2019}, though typically restricted to the $1$-dimensional case. Our analysis leverages recent concentration results on quadratic forms of Markov chains \cite{moulines2025note}, building upon martingale decomposition techniques for $U$-statistics developed in \cite{atchade2014martingale}.
\par 
Recent studies have analyzed convergence rates of SA iterates to their limiting normal distributions. \cite{samsonov2024gaussian} derived an $\mathcal{O}(n^{-1/4})$ convergence rate in \eqref{eq:CLT_fort_prelim} in convex distance for TD learning under \iid\ noise, improved to $\mathcal{O}(n^{-1/3})$ by \cite{wu2024statistical} specifically for TD learning. \cite{pmlr-v99-anastasiou19a} applied Stein’s method to averaged SGD and LSA iterates in smooth Wasserstein distance, again under \iid\ conditions. Recently, \cite{srikant2024rates} analyzed rates for martingale CLT in $1$-st order Wasserstein distance with Markovian samples but achieved slower convergence than our result (see \Cref{th:normal_approximation_markov_with_sigma_n} and discussion after it). Lastly, \cite{agrawalla2023high} provided last-iterate bounds for SGD for high-dimensional linear regression, also with \iid\ noise.
\par 
Originally introduced for \iid\ data \cite{efron1992bootstrap}, the bootstrap method has been adapted to complex settings, including high-dimensional tests \cite{Chernozhukov2013,chernozhukov2017central}, linear regression \cite{spokoiny2015}, and covariance matrix spectral projectors \cite{Naumov2019,Jirak2022}. For SA methods, an online bootstrap for SGD was proposed in \cite{Fang2018}, with recent non-asymptotic analysis showing approximation rates up to $\mathcal{O}(n^{-1/2})$ for coverage probabilities in case of strongly convex objectives with independent noise \cite{sheshukova2025gaussian}. A similar bootstrap analysis for linear stochastic approximation (LSA) yielded rates of order $\mathcal{O}(n^{-1/4})$ \cite{samsonov2024gaussian}. Extensions to Markovian settings by \cite{ramprasad2023online} proved inconsistent, as demonstrated by \cite[Proposition~1]{liu2023statistical}. Meanwhile, \cite{liu2023statistical} proposed consistent mini-batch SGD estimators for $\varphi$-mixing noise, focusing solely on asymptotic consistency. They also suggested a consistent procedure for constructing the confidence intervals, but studied only its asymptotic properties using the independent block trick \cite{yu1994rates}. Lastly, multiplier bootstrap techniques for online non-convex SGD were considered by \cite{zhong2023online}. 
\par 
Among other approaches for constructing confidence intervals for dependent data, we mention methods based on (asymptotically) pivotal statistics \cite{li2023online,li2023statistical}. The authors of \cite{li2023statistical} considered the Polyak–Ruppert averaged $Q$-learning algorithm under the \iid\ noise assumption (generative model), while \cite{li2023online} generalized this approach to nonlinear stochastic approximation under Markov noise. Another group of methods are based on estimating the asymptotic covariance matrix $\Sigma_{\infty}$ appearing in the central limit theorem \eqref{eq:CLT_fort_prelim}. Among these approaches we mention the plug-in estimator of \cite{chen2020aos} and batch-mean estimators, see e.g. \cite{chen2021statistical,zhu2023online_cov_matr,roy2023online}. Specifically, \cite{roy2023online} treats SGD with Markov noise, contrasted with the independent-noise SGD analyses in \cite{chen2020aos,chen2021statistical,zhu2023online_cov_matr}, and \cite{zhong2023online} investigates both multiplier bootstrap and batch-mean estimators for nonconvex objectives. These methods yield non-asymptotic error bounds in expectation, $\PE[\norm{\hat{\Sigma}_n - \Sigma_{\infty}}]$, but do not study convergence rates in \eqref{eq:CLT_fort_prelim}. A notable exception is \cite{wu2024statistical}, which delivers a non-asymptotic analysis of TD-learning under \iid\ noise, using plug-in covariance estimates to achieve an $\mathcal{O}(n^{-1/3})$ approximation rate for coverage probabilities of $\thetas$.
\section{Accuracy of Gaussian approximation for LSA with Markov noise}
\label{sec:normal_approximation_main}
We first study the rate of normal approximation for the Polyak-Ruppert averaged LSA procedure. When there is no risk of ambiguity, we use the simplified notations $\funcAw_n = \funcA{\State_n}$ and $\funcbw_n = \funcb{\State_n}$. Starting from the definition \eqref{eq:lsa}, we get with elementary transformations that
\begin{equation}
\label{eq:main_recurrence_1_step}
\theta_{n} - \thetas = (\Id - \alpha_{n} \funcAw_n)(\theta_{n-1} - \thetas) - \alpha_{n} \funnoisew_{n}\eqsp,
\end{equation}
where we have set $\funnoisew_n= \funcnoise{\State_n}$ with $\funcnoise{z} =  \zmfuncA{z} \thetas - \zmfuncb{z}$, $\zmfuncA{z}  = \funcA{z} - \bA$, $\zmfuncb{z} = \funcb{z} - \barb$.
Here the random variable $\funcnoise{\State_n}$ can be viewed as a noise, measured at the optimal point $\thetas$. We also define the noise covariance matrix under the stationary distribution $\pi$:
\begin{equation}
\label{eq:def_noise_cov}
\textstyle \noisecov = \PE_\pi[\funcnoise{Z_0}\{ \funcnoise{Z_0}\}^\top] + 2 \sum_{\ell=1}^\infty \PE_\pi[\funcnoise{Z_0} \{\funcnoise{Z_\ell}\}^\top] \eqsp. 
\end{equation} 
We now impose the following conditions:
\begin{assum}
\label{assum:UGE}
The sequence $(Z_k)_{k \in \nset}$ is a Markov chain taking values in a Polish space $(\Zset,\Zsigma)$ with  Markov kernel $\MKQ$. Moreover, $\MKQ$ admits $\pi$ as a unique invariant distribution and is uniformly geometrically ergodic, that is, there exists $\taumix \in \nset$, such that for any $k \in \nset$, it holds that 
\begin{equation}
\label{eq:tau_mix_contraction}
\textstyle 
\dobrush(\MKQ^{k}) := \sup_{z,z' \in \Zset} \tvdist(\MKQ^{k}(z,\cdot),\MKQ^{k}(z',\cdot)) \leq (1/4)^{\lceil k / \taumix \rceil}\eqsp.
\end{equation}
\end{assum}
Parameter $\taumix$ in \eqref{eq:tau_mix_contraction} is referred to as \emph{mixing time}, see \cite{paulin_concentration_spectral}. We also impose the following assumptions on the noise variables $\funcnoise{z}$:
\begin{assum}
\label{assum:noise-level}
$\int_{\Zset}\funcA{z}\rmd \pi(z) = \bA$ and $\int_{\Zset}\funcb{z}\rmd \pi(z) = \barb$, with the matrix $-\bA$ being Hurwitz. Moreover, $\supconsteps = \sup_{z \in \msz}\normop{\funcnoise{z}} < \plusinfty$, and the mapping $z \to \funcA{z}$ is bounded, that is, 
\begin{equation}
\label{eq:a_matr_bounded}
\bConst{A} = \sup_{z \in \msz} \normop{\funcA{z}} \vee \sup_{z \in \msz} \normop{\zmfuncA{z}} < \infty\eqsp.
\end{equation}
Moreover, we assume that $\lambda_{\min}(\noisecov) > 0$.
\end{assum}
To motivate the introduction of $\noisecov$ in this particular form, note that under assumption \Cref{assum:noise-level}, $\PE_\pi[\funcnoise{Z_0}] = 0$, the following central limit theorem holds (see e.g. \cite[Chapter~21]{douc:moulines:priouret:soulier:2018}):
\begin{equation}
\label{eq:clt_eps_motivation}
\frac{1}{\sqrt{n}}\sum_{\ell=0}^{n-1}\funcnoise{Z_{\ell}} \to \mathcal{N}(0,\noisecov)\eqsp.
\end{equation}
The fact that the matrix $-\bA$ is Hurwitz implies that the linear system $\bA \theta = \barb$ has a unique solution $\thetalim$. Moreover, this fact is sufficient to show that the matrix $\Id - \alpha \bA$ is a contraction in an appropriate matrix $Q$-norm for small enough $\alpha > 0$. Precisely, the following result holds:
\begin{proposition}[Proposition 1 in \cite{samsonov2024gaussian}]
\label{prop:hurwitz_stability}
Let $-\bA$ be a Hurwitz matrix. Then for any $P = P^{\top} \succ 0$, there exists a unique matrix $Q = Q^{\top} \succ 0$, satisfying the Lyapunov equation $\bA^\top Q + Q \bA = P$. Moreover, setting
\begin{equation}
\label{eq:alpha_infty_def}
\textstyle 
a = \frac{\lambda_{\min}(P)}{2\normop{Q}}\eqsp, \quad
\text{and} \quad \alpha_\infty = \frac{\lambda_{\min}(P)}{2\qcond \normop{\bA}[Q]^{2}}\wedge \frac{\normop{Q}}{\lambda_{\min}(P)} \eqsp,
\end{equation}
where $\qcond = \lambda_{\max}(Q)/\lambda_{\min}(Q)$, it holds for any $\alpha \in [0, \alpha_{\infty}]$ that $\alpha a \leq 1/2$, and
\begin{equation}
\label{eq:contractin_q_norm}
\normop{\Id - \alpha \bA}[Q]^2 \leq 1 - \alpha a\eqsp.    
\end{equation}
\end{proposition}
For completeness we provide the proof of this result in \Cref{appendix:tehnical}. Now, we make an assumption about the form of step sizes $\alpha_{k}$ in \eqref{eq:main_recurrence_1_step} and the total number of observations $n$:
\begin{assum}
\label{assum:step-size}
Step sizes $\{\alpha_{k}\}_{k \in \nset}$ have a form $\alpha_{k} = c_{0} / (k+k_0)^{\gamma}$, where $\gamma \in [1/2;1)$, and the constant $c_{0} \leq 1/(2a)$. Moreover, we assume that $n$ is large enough and 
\[k_0 > g(a, \taumix, c_0, \bConst{A} \qcond, \alpha_{\infty})(\log n)^{1/\gamma}.\] 
Precise expressions for $g(a, \taumix, c_0, \bConst{A} \qcond, \alpha_{\infty})$ and for $n$ are given in \Cref{sec:Extended_version_of_A3} (see \Cref{assum:step-size_extended}).
\end{assum}
The main aim of lower bounding $n$ is to ensure that the error related to the choice of initial condition $\theta_0$ becomes small enough. Note also that our bound of \Cref{assum:step-size} requires to known the number of observations $n$ in advance and adjust $k_0$ accordingly. The same problem can be traced in the existing high-probability results for LSA \cite{mou2020linear,durmus2022finite,wu2024statistical}, in the regime when the confidence parameter $\delta$ in high-probability bounds scales with the number of iterations $n$.
\par 
It is known (see e.g. \cite{fort:clt:markov:2015}), that the assumptions \Cref{assum:UGE} - \Cref{assum:step-size} guarantee that the sequence $\prtheta_{n}$ is asymptotically normal, that is, $\sqrt{n}(\bar{\theta}_{n} - \thetas) \overset{d}{\rightarrow} \mathcal{N}(0,\Sigma_{\infty})$, 
where the covariance matrix $\Sigma_{\infty}$ has a form
\begin{equation}
\label{eq:sigma_infty_def}
\Sigma_{\infty} = \bA^{-1} \noisecov \bA^{-\top}\eqsp.
\end{equation}
For a fixed $u \in \sphere^{d-1}$ we consider projection of $\sqrt{n}(\bar{\theta}_{n} - \thetas)$
on $u$ and quantify the rate of convergence in the Kolmogorov distance, i.e., we consider the quantity
\begin{equation}
\label{eq:kolmogorov_dist_main}
\kolmogorov{\sqrt{n} u^\top (\bar{\theta}_{n} - \thetas) / \sigma(u),\mathcal{N}(0,1)} = \sup_{x \in \rset}|\PP(\sqrt{n} u^\top (\bar{\theta}_{n} - \thetas)/\sigma(u) \leq x) - \Phi(x)| \eqsp,
\end{equation}
where $\Phi(x)$ is the c.d.f. of the standard normal law $\mathcal{N}(0,1)$, and 
$\sigma^2(u) = u^\top \bA^{-1} \noisecov \bA^{-\top} u$.
To control the quantity \eqref{eq:kolmogorov_dist_main}, we will first present an auxiliary result. Define
\begin{equation}
\label{eq:Q_ell_definition_main}
G_{m:k} = \prod_{\ell=m}^k (\Id - \alpha_\ell \bA), \,
Q_\ell = \alpha_\ell \sum_{k=\ell}^{n-1} G_{\ell+1:k}\eqsp, \, \Sigma_n = \frac{1}{n}\sum_{\ell=2}^{n-1} Q_{\ell} \noisecov Q_{\ell}^{\top}\eqsp, \, \sigma^2_n(u) = u^\top \Sigma_n u\eqsp.
\end{equation}

\begin{theorem}
\label{th:normal_approximation_markov_with_sigma_n}
Assume \Cref{assum:UGE}, \Cref{assum:noise-level}, and \Cref{assum:step-size}. Then for any $u \in \sphere_{d-1}$, $\theta_0 \in \rset^d$, and initial distribution $\xi$ on $(\Zset,\Zsigma)$, it holds that
\begin{equation}
\label{eq:kolmogorov_distance_sigma_n}
\kolmogorov{\sqrt{n}  u^\top (\bar{\theta}_{n} - \thetas)/\sigma_n(u),\mathcal{N}(0,1)} \leq \mathrm{B}_n\eqsp,
\end{equation}
where we set
\begin{equation}
\label{eq:normal_approximation_markov}
\mathrm{B}_n = \frac{\ConstK[1] \log^{3/4}{n}}{n^{1/4}} + \frac{\ConstK[2] \log{n}}{n^{1/2}} + \frac{\Constupd{1}\|\theta_0 - \thetas \| + \Constupd{2}}{\sqrt{n}} + \Constupd{3}\frac{(\log n)^2}{n^{\gamma-1/2}} + \Constupd{4}\frac{(\log n)^{5/2}}{n^{\gamma-1/2}}\eqsp,
\end{equation}    
where $\ConstK[1]$ and $\ConstK[2]$ are defined in \Cref{appendix:Markov_technical}, and $\{\Constupd{i}\}_{i=1}^4$ are defined in \Cref{appendix:proofs}. 
\end{theorem}
\begin{proof}
We provide below the main elements of the proof and refer the reader to \Cref{appendix:proofs} for a complete derivations. We use the expansion \eqref{eq:decomp_fluctuation} and \eqref{eq:error_decomposition_LSA}, detailed in the appendix:
$\theta_k - \thetas = \ProdB_{1:k}(\theta_0 - \thetas) + J_k^{(0)} + H_k^{(0)}$,
where $\ProdB_{m:k} = \prod_{\ell = m}^k (\Id - \alpha_\ell \funcA{\State_{\ell}})$ for $m \le k$, $J_0^{(0)} = H_0^{(0)} = 0$, 
$J_k^{(0)} = -\sum_{\ell = 1}^{k} \alpha_
\ell G_{\ell+1:{k}} \funnoisew(Z_\ell)$, and $H_k^{(0)} = -\sum_{\ell = 1}^k \alpha_
\ell \ProdB_{\ell+1:k} (\funcA{\State_{\ell}} - \bA ) J_{\ell-1}^{(0)}$.
Taking average and changing the order of summation we get that $\sqrt{n} (\bar{\theta}_{n} - \thetas) = W + D_1$, with  
\begin{equation}
\label{eq: linear part and non linear}
W =- \frac{1}{\sqrt{n}} 
\sum_{\ell=1}^{n-1} Q_{\ell} \funnoisew(\State_\ell) , \quad D_1 =  \frac{1}{\sqrt n} \sum_{k = 0}^{n-1} \ProdB_{1:k} (\theta_0 - \thetas) +\frac{1}{\sqrt{n}} \sum_{k=1}^{n-1} H_k^{(0)}\eqsp. 
\end{equation}
Note that $W$ is a linear statistic of the Markov chain $\{Z_\ell\}_{\ell \in \nset}$. We further expand it into a sum of martingale-differences and remainder terms.  Note that under \Cref{assum:UGE} the function 
$\hat \funnoisew(z) = \sum_{k=0}^\infty \MKQ^k \funnoisew(z)$ is a solution to \emph{Poisson equation}
$\hat \funnoisew(z) - \MKQ \hat \funnoisew(z) = \funnoisew(z)$. We rewrite the term $W$ in the following way
$W  = n^{-1/2} M + D_2$, with $M = - \sum_{\ell=2}^{n-1} \Delta M_\ell$, $ = Q_{\ell}\bigl(\hat \funnoisew(\State_\ell) - \MKQ \hat \funnoisew(Z_{\ell-1})\bigr)$ and
\begin{equation}
\textstyle
D_2 = -\frac{1}{\sqrt n} Q_1 \hat \funnoisew (\State_1) +\frac{1}{\sqrt n}  Q_{n-1} \MKQ \hat \funnoisew (\State_{n-1}) + \sum_{\ell=1}^{n-2}(Q_{\ell}- Q_{\ell+1})\MKQ\hat \funnoisew(\State_{\ell})\eqsp.
\end{equation}
Note that $\{\Delta M_\ell\}_{\ell = 2}^{n-1}$ is a martingale-difference sequence w.r.t. $\mathcal F_\ell= \sigma(\State_k, k \leq \ell)$. Note
\begin{equation}
\label{eq:M_quadratic_characteristic}
\textstyle
\PE_{\pi}^{\mathcal F_{\ell-1}}[\Delta M_\ell \Delta M_\ell^\top] = Q_{\ell} \tilde \funnoisew(\State_{\ell-1}) Q_{\ell}^{\top}, \qquad \tilde \funnoisew(z) = \MKQ \hat \funnoisew \hat \funnoisew^\top(z) - \MKQ \hat \funnoisew(z) \MKQ \hat \funnoisew^\top(z)  \eqsp. 
\end{equation}
Furthermore, we have
$\pi\bigl(\tilde \funnoisew\bigr) = \noisecov \eqsp.$
The term $M$ is a martingale, whose quadratic characteristic is given  by
$\langle M \rangle_n = \sum_{\ell=1}^{n-2} Q_{\ell+1} \tilde \funnoisew(\State_\ell) Q_{\ell+1}^{\top}$. Denote $D = D_1 + D_2$. With these notations, we get
$\sqrt{n}(\bar{\theta}_{n} - \thetas) = n^{-1/2} M + D$.
Applying \Cref{prop:nonlinearapprox} with $X = n^{-1/2} u^\top M/\sigma_n(u)$ and $Y = u^\top D/\sigma_n(u)$, we obtain for any $p \geq 2$,
\begin{equation}
\label{eq:D_remainder_term_main}
\textstyle 
\kolmogorov{\frac{\sqrt{n}  u^\top (\bar{\theta}_{n} - \thetas)}{\sigma_n(u)},\mathcal{N}(0,1)} \leq \kolmogorov{\frac{u^\top M}{\sqrt{n}\sigma_n(u)},\mathcal{N}(0,1)} + 2 \bigl\{\PE[\bigl|\frac{u^\top D}{\sigma_n(u)}\bigr|^{p}]\bigr\}^{1/(p+1)} \eqsp. 
\end{equation}
To obtain the rate of Gaussian approximation for $\sqrt{n}  u^\top (\bar{\theta}_{n} - \thetas)/\sigma_n(u)$ it remains to control the moments of the term $|u^\top D/\sigma_n(u)|$, which is done in \Cref{prop:D_term_bound}, and to control $\kolmogorov{u^\top M/(\sqrt{n}\sigma_n(u))}$. To bound the latter term we use a normal approximation result for sums of martingale-difference sequences, which builds upon the arguments of \cite{bolthausen1982martingale} and \cite{fan2019exact} - see  \Cref{th:normal_approximation_martingales}, applied with $p = \log{n}$. 
\end{proof}

Note that the result of \Cref{th:normal_approximation_markov_with_sigma_n} yields an approximation of $\sqrt{n}  u^\top (\bar{\theta}_{n} - \thetas)$ with $\mathcal{N}(0,\sigma^2_n(u))$, and not the limiting quantity $\sigma^2(u)$ from the CLT \eqref{eq:CLT_fort_prelim}. In order to complete the result, we need an additional result on the Gaussian comparison between $\mathcal{N}(0,\sigma^2_n(u))$ and $\mathcal{N}(0,\sigma^2(u))$. This result is based is based on the quantitative estimates provided first in \cite{BarUly86}, and then revised in \cite{devroye2018total}.

\begin{corollary}
\label{cor:normal_approximation_markov}
Under assumptions of \Cref{th:normal_approximation_markov_with_sigma_n} it holds, with $\mathrm{B}_n$ given in \eqref{eq:normal_approximation_markov}, that 
\begin{align}
\kolmogorov{\sqrt{n}  u^\top (\bar{\theta}_{n} - \thetas)/\sigma(u),\mathcal{N}(0,1)} 
\leq \mathrm{B}_n + C_{\infty}n^{\gamma-1}\eqsp,
\end{align}
where  $C_{\infty}$ is defined in  \eqref{def:C_infty_const}.
\end{corollary}
\begin{remark}
\label{rem:optimal_step_size}
The bound of \Cref{cor:normal_approximation_markov} predicts the optimal error of normal approximation for Polyak-Ruppert averaged estimates of order $n^{-1/4}$ up to a logarithmic factors in $n$, which is achieved with the step size $\alpha_{k} = c_{0}/(k+k_0)^{3/4}$, that is, when setting $\gamma = 3/4$ in \eqref{eq:normal_approximation_markov}. In this case we obtain the optimized bound:
\begin{equation}
\label{eq:kolmogorov_bound_optimized}
\textstyle 
\kolmogorov{\sqrt{n}  u^\top (\bar{\theta}_{n} - \thetas)/\sigma(u),\mathcal{N}(0,1)}  \lesssim_{pr} \frac{(\log n)^{5/2}}{n^{1/4}} \eqsp.
\end{equation}
where $\lesssim_{pr}$ stands for inequality up to absolute and problem-specific constants (such as $\bConst{A}, \qcond, a, \taumix$), but not $n$.
\end{remark}
\paragraph{Discussion.} The proof of \Cref{cor:normal_approximation_markov} is given in \Cref{sec:cor:normal_approximation_markov}. Results similar to the one of \Cref{cor:normal_approximation_markov} have been recently obtained in the literature in \cite{samsonov2024gaussian}, \cite{srikant2024rates} and \cite{pmlr-v99-anastasiou19a}. \cite{samsonov2024gaussian} considered the LSA algorithm based on \iid\ noise variables $\{Z_{k}\}_{1 \leq k \leq n}$ and a randomized concentration approach based on \cite{shao2022berry}. This result was later refined in \cite{wu2024statistical}. To the best of our knowledge, this technique does not extend to Markovian noise $\{Z_{k}\}_{1 \leq k \leq n}$. 
\cite{wu2025uncertainty} obtain for the setting of TD learning with Markov noise the bound of the form
\begin{equation}
\label{eq:wu_bound}
\textstyle 
\metricd[C](\sqrt{n} (\bar{\theta}_{n} - \thetas),\Sigma_{\infty}^{1/2}\eta) \lesssim \frac{\log{n}}{n^{1/4}}\eqsp,
\end{equation}
where $\eta \sim \mathcal{N}(0,\Id_d)$, and $\metricd[C](X, Y) = \sup_{B \in \Conv(\rset^{d})}\left|\P\bigl(X \in B\bigr) - \P(Y \in B)\right|$ is the convex distance. However, it is not clear, if this result applies to the general LSA setting. The authors in \cite{srikant2024rates} obtained the convergence rate for the general LSA procedure with Markov noise 
\begin{equation}
\textstyle
\metricd[W](\sqrt{n}(\bar{\theta}_{n} - \thetas),\Sigma_{\infty}^{1/2}\eta) \lesssim \frac{\sqrt{\log{n}}}{n^{1/6}}\eqsp,
\end{equation}
where $\metricd[W]$ is the $1$-st order Wasserstein distance. This result implies, due to the classical relations of between Kolmogorov and Wasserstein distance (see \cite{Ross}), the final rate approximation on Kolmogorov distance of order $\mathcal{O}(n^{-1/12})$, which is worse compared to the rate of \Cref{th:normal_approximation_markov_with_sigma_n}. On the other hand, the result of \cite{srikant2024rates} holds in a $d$-dimensional setting, whereas our analysis is restricted to one-dimensional projections of the estimation error. Nevertheless, the essential part of our analysis, namely controlling the remainder term $D$ in \eqref{eq:D_remainder_term_main}, can be generalized to the $d$-dimensional case through bounds on $\PE_{\xi}^{1/p}\left[ \norm{D}^{p}\right]$, following an approach similar to that presented in \Cref{appendix:proofs}.

\section{Multiplier subsample bootstrap for LSA}
\label{sec:bootstrap}
We will apply the multiplier subsample bootstrap (MSB) procedure,  a block-based approach that constructs the bootstrap statistic via a blockwise scheme; see \cite{kunsch1989jackknife,MR4733879}. Below we describe in details the MSB approach, closely following \cite{MR4733879}. Let $b_n$ be the length of block, and for each $t = 0,\ldots, n - b_n$, define $\bar{\theta}_{b_n,t} = (1/b_n)\sum_{\ell=t}^{t+b_n-1} \theta_\ell$, the  ''scale $b_n$'' version of $\bar{\theta}_{n}$. To imitate $\bar{\theta}_{n}$, the MSB estimator of $\bar{\theta}_{n}$ is given by
\begin{equation}
\label{eq:OBM_estimator_def}
\bar{\theta}_{n, b_n}(u) = \frac{\sqrt{b_n}}{\sqrt{n-b_n+1}} \sum_{t = 0}^{n-b_n} w_t (\bar{\theta}_{b_n,t} - \bar{\theta}_{n})^\top u\eqsp, 
\end{equation}
where $\Xi_n = \{Z_{\ell}\}_{\ell=1}^{n}$, and $\{w_{\ell}\}_{0 \leq \ell \leq n-b_n}$, the multiplier weights, are \iid\, $\mathcal N(0,1)$ random variables, which are independent of $\Xi_n$. We write, respectively, $\PPb = \PP(\cdot | \Xi_n)$ and $\PEb = \PE(\cdot | \Xi_n)$ for the corresponding conditional probability and expectation. For simplicity, we do not "subsample" the blocks: the theory extends readily, but we do not want to add another layer of notations.
The key idea of the MSB procedure \eqref{eq:OBM_estimator_def} is that the "bootstrap world" distribution $\PPb\big(\bar{\theta}_{n, b_n}(u) \leq x\big)$ should be close to their "real world" counterparts $\PP\big(\sqrt{n} (\bar{\theta}_n - \thetas) \leq x\big)$ for any $x \in \rset$. Formally, the procedure \eqref{eq:OBM_estimator_def} is said to be \emph{asymptotically valid}, if  the quantity 
\begin{equation}
\label{eq:boot_validity_supremum}
\sup_{x \in \rset} |\PP(\sqrt n (\bar{\theta}_{n} - \thetas)^\top u \le x ) - \PPb(\bar{\theta}_{n, b_n}(u) \le x)| 
\end{equation}
converges to $0$ in $\PP$-probability. Typically the authors consider the asymptotic validity of the procedures for constructing the confidence intervals (either with multiplier bootstrap \cite{Fang2018,MR4733879} or with direct estimation of the asymptotic covariance \cite{chen2020aos,zhu2023online_cov_matr}). Our aim in this section is to provide \emph{fully non-asymptotic} bounds on the rate at which the supremum in \eqref{eq:boot_validity_supremum} decays as a function of $n$.
\par 
The MSB estimator \eqref{eq:OBM_estimator_def} $\bar{\theta}_{n, b_n}(u)$ is normally distributed w.r.t. $\PPb$, that is,
\[
\textstyle 
\bar{\theta}_{n, b_n}(u) \sim \mathcal{N}(0,\hat \sigma^2_\theta(u))\eqsp,
\]
with the variance $\hat \sigma^2_\theta(u)$ given by
\begin{equation}
\label{eq:sigma_hat_variance}
\hat \sigma^2_\theta(u) = \frac{b_n}{n-b_n+1} \sum_{t = 0}^{n-b_n}  (( \bar{\theta}_{b_n,t} - \bar{\theta}_{n})^\top u )^2\eqsp.
\end{equation}
The parameter $b_n$ is commonly referred to as \emph{lag window}, or \emph{bandwidth} (see \cite{flegal:2010} and references therein). Under the bootstrap probability $\PPb$, $\bar{\theta}_{n, b_n}(u)$ (see  \eqref{eq:OBM_estimator_def}) is a Gaussian approximation of $\bar{\theta}_n$ with estimated variance $\hat \sigma^2_\theta(u)$. Up to a multiplicative factor tending to $1$ as $n$ goes to $\infty$, the variance formula in \eqref{eq:sigma_hat_variance} coincides with the \emph{overlapping batch mean} estimator (OBM), a well-known technique for estimating the asymptotic variance of the Markov chain, suggested in \cite{meketon1984overlapping}. The properties of OBM estimator are studied, see e.g. \cite{welch1987relationship,damerdji1991strong,flegal:2010}. In particular, it is known (see \cite{flegal:2010}), that the OBM is a consistent estimator of the asymptotic variance of a Markov chain under suitable ergodicity assumptions, provided that $b_n \to \infty$ as $n \to \infty$. 
\par 
Note that even if $\{Z_k\}_{k \in \nset}$ is a Markov chain, the iterates $\{\theta_k\}_{k \in \nset}$ alone do not form a Markov chain (one must rather consider the joint process $(\theta_k,Z_k)$). Consequently, the classical consistency results for overlapping‐batch‐means variance estimators do not apply directly to \eqref{eq:sigma_hat_variance}. Fortunately, we show below that applying the block bootstrap for the sequence $\{\theta_{\ell}\}$, is equivalent, up to a suitable correction, to the block bootstrap procedure applied to the non-observable random variables $\{\funcnoise{Z_{\ell}}\}$. To make this precise, we define for each block start $t \in \{0,\ldots,n-b_n\}$ the quantities:
\begin{equation}
\label{eq:W_n_b_t_def}
\bar{W}_{b_n,t} = \frac{1}{b_n}\sum_{\ell=t+1}^{t+b_n-1}\funcnoise{Z_{\ell}}\eqsp, \;\bar{W}_{n} = \frac{1}{n} \sum_{\ell=1}^{n-1}\funcnoise{Z_{\ell}}, \;
\hat \sigma_{\varepsilon}^2(u) = \frac{b_n}{n-b_n+1} \sum_{t = 0}^{n-b_n}  
((\bar{W}_{n,b_n,t} - \bar{W}_{n})^{\top} u )^2\eqsp.
\end{equation}
Then the following proposition holds:
\begin{proposition}
\label{prop:block_bootstrap_batch_mean}
Assume \Cref{assum:UGE}, \Cref{assum:noise-level}, and \Cref{assum:step-size}. Then for any $u \in \sphere_{d-1}$, it holds that 
\begin{equation}
\label{eq:block_bootstrap_batch_mean}
\hat \sigma^2_\theta(u) = \hat \sigma_{\varepsilon}^2(u) + \RemCov_{var}(u)\eqsp,
\end{equation}
where $\RemCov_{var}(u)$ is a remainder term defined in \Cref{sec:block_bootstrap_batch_mean_proof} (see \eqref{def:remcov_car}).
Moreover, for any $2 \leq p \leq \log{n}$, and any initial distribution $\xi$ on $(\Zset,\Zsigma)$, it holds that 
\begin{equation}
\label{eq:remainder_term_variance_bound}
\begin{split}
\PE_{\xi}^{1/p}\bigl[ \bigl| \RemCov_{var}(u) \bigr|^{p} \bigr] &\lesssim M_1pb_n^{1/2}n^{\gamma/2-1} + M_2p^{4}(\log n)b_n^{1/2}n^{-\gamma} \\&\qquad \qquad + M_3pb_{n}^{-1/2}n^{\gamma/2} + M_4p^4(\log n)n^{-1} + M_5pn^{2\gamma-2}
\end{split} 
\end{equation}
and the constants $M_i$ are defined in \Cref{sec:block_bootstrap_batch_mean_proof} (see \eqref{def:const_M_i}).
\end{proposition}
The proof of \Cref{prop:block_bootstrap_batch_mean} is given in \Cref{sec:block_bootstrap_batch_mean_proof}. The bound of \Cref{prop:block_bootstrap_batch_mean} has some noteworthy properties. First, our bound requires that the block size $b_n$ to grow at least like $n^{\gamma}$ to ensure that the residual term $\RemCov_{var}(u)$ is small. Careful inspection of the bootstrap estimator $\bar{\theta}_{n, b_n}(u)$ from \eqref{eq:OBM_estimator_def} explains this dependence. Indeed, one can expect that the decay rate of the covariance between $\theta_k$ and $\theta_{k+b_n}$ depends on the quantity $\sum_{k=t+1}^{t+b_n}\alpha_{k}$, see e.g. \cite{durmus2020biassgd}. At the same time,  for $b_n \simeq n^{\beta}, 0 < \beta < 1$, and $\alpha_k = c_0/(k_0+k)^{\gamma}$, which is set according to \Cref{assum:step-size}, we obtain that \[
\textstyle 
\sum_{k=t+1}^{t+b_n}\alpha_{k} \lesssim b_n/t^{\gamma}\eqsp. 
\] 
Hence, if $b_n \ll n^{\gamma}$, the sequence $\{\theta_k\}$ does not mix within each block $t \leq k \leq t+b_n$. 
\par 
Given that the remainder term $\RemCov_{var}(u)$ in \eqref{eq:remainder_term_variance_bound} is negligible, it remains to analyze the concentration properties of $\sigma_{\varepsilon}^2(u)$. Indeed, under \Cref{assum:UGE} and \Cref{assum:noise-level}, the (normalized) linear statistic $n^{-1/2}\sum_{\ell=0}^{n-1} u^{\top} \funcnoise{Z_{\ell}}$ is asymptotically normal with the asymptotic variance equal to $\sigma^2(u)$ defined in \eqref{eq:kolmogorov_dist_main}. Moreover, $\hat \sigma_{\varepsilon}^2(u)$ coincides with the overlapping batch mean estimator of $\sigma^2(u)$. In order to prove concentration bounds for $\hat \sigma_{\varepsilon}^2(u)$ around $\sigma^2(u)$, we apply the result of \cite[Theorem~1]{moulines2025note}.
\begin{proposition} 
\label{prop:concentration_OBM}
Assume \Cref{assum:UGE} and \Cref{assum:noise-level}. Then for any $p \geq 2$, and $n \geq 2b_n + 1$, and any initial distribution $\xi$ on $(\Zset,\Zsigma)$, it holds that
\begin{equation}
\PE_{\xi}^{1/p}[|\hat \sigma_{\varepsilon}^2(u) - \sigma^2(u)|^p] \lesssim \frac{p \taumix^3 \supconsteps^2}{\sqrt{n}} + \frac{p^2 \taumix^2 \sqrt{b_n} \supconsteps^2}{\sqrt{n}} + \frac{p \taumix^{2}\supconsteps^2}{\sqrt{b_n}}\eqsp.
\end{equation}
\end{proposition}
The result above is based on martingale decomposition, associated with the Poisson equation for quadratic forms of Markov chains, as introduced in \cite{atchade2014martingale}. Now, combining the estimates of \Cref{prop:block_bootstrap_batch_mean,prop:concentration_OBM} and applying Markov's inequality with $p=\log n$, we obtain the following result:
\begin{corollary}
\label{coro:first-concentration_OBM}
Let $n$ be large enough. Set \(b_n = \lceil n^{3/4} \rceil\), $\varepsilon \in (0; 1/\log{n})$,  and let $\alpha_k = c_0/(k_0 + k)^{1/2 + \varepsilon}$. Then, with probability at least \(1 - n^{-1}\), it holds that \begin{equation}
\label{eq:coro-conenctration-OBM-formula}
\bigl|\hat{\sigma}^2_{\theta}(u) - \sigma^2(u)\bigr| \lesssim_{\log{n}} n^{-1/8 + \varepsilon/2} \eqsp.
\end{equation}
\end{corollary}
\paragraph{Discussion.} The version of \Cref{coro:first-concentration_OBM} with explicit constants and explicit power of $\log{n}$ is provided in \Cref{sec:proof_concentration_OBM}, together with the proof of \Cref{coro:first-concentration_OBM}. Note that the fastest decay rate in the r.h.s of \eqref{eq:coro-conenctration-OBM-formula} is achieved when we set $b_n = \mathcal{O}(n^{3/4})$ and aggressive step sizes $\alpha_k = c_0/(k_0 + k)^{1/2 + \varepsilon}$. The same choice of hyperparameters appears to be optimal in the recent work of \cite{roy2023online} for batch-mean estimators of the asymptotic variance for the SGD algorithm, even in case of controlled Markov chain. We also recover the rate $n^{-1/8}$ (up to logarithmic factors), that previously appeared in \cite{roy2023online}. The authors of \cite{chen2020aos,zhu2023online_cov_matr} also considered the batch mean estimators $\widehat{\Sigma}_n$ of the asymptotic variance $\Sigma_{\infty}$ in case of SGD algorithms with independent noise and obtained the same (up to logarithmic factors) convergence rate $\PE[\norm{\widehat{\Sigma}_n - \Sigma_{\infty}}] \lesssim n^{-1/8}$. Moreover, the optimal rate for recovering $\Sigma_{\infty}$ in \cite[Corollary~4.5]{chen2020aos} and \cite[Corollary~3.4]{zhu2023online_cov_matr} is attained for the step sizes $\alpha_k = 1/k^{\gamma}$ with $\gamma \to 1/2$. This is on par with our findings of \Cref{coro:first-concentration_OBM}. At the same time, one should note that this choice of step sizes yields extremely slow convergence rates in the CLT in \Cref{th:normal_approximation_markov_with_sigma_n}. This introduces an additional trade-off, that needs to be taken into account when considering the decay rate of \eqref{eq:boot_validity_supremum}. Namely, one needs to balance not only the right-hand sides of \Cref{coro:first-concentration_OBM} and \Cref{prop:block_bootstrap_batch_mean}, but also to take into account the convergence rate in \Cref{th:normal_approximation_markov_with_sigma_n}. The respective trade-off yields 
\begin{equation}
\label{eq:n_b_gamma_choice}
b_n = \lceil n^{4/5} \rceil\eqsp, \quad \alpha_k = c_0 / (k_0+k)^{3/5}\eqsp. 
\end{equation}
The corresponding main theorem writes as follows:
\begin{theorem}
\label{th:bootstrap_validity_main}
Assume \Cref{assum:UGE}, \Cref{assum:noise-level}, and \Cref{assum:step-size}, let $n$ be large enough, set $b_n = \lceil n^{4/5} \rceil$, $\alpha_k = c_0 / (k_0+k)^{3/5}$. Then for any $u \in \sphere_{d-1}$, and any initial distribution $\xi$ on $(\Zset,\Zsigma)$, it holds  with $\PP$ -- probability at least $1 - 1/n$  that
\begin{align}
\sup_{x \in \rset} |\PP(\sqrt n (\bar{\theta}_{n} - \thetas)^\top u \le x ) - \PPb(\bar{\theta}_{n, b_n}(u) \le x)| \lesssim_{\log{n}} n^{-1/10} \eqsp.
\end{align}
\end{theorem}

\begin{proof}
We provide the detailed proof of \Cref{th:bootstrap_validity_main} in \Cref{sec:proof_bootstrap_validity_main} together with the explicit form of condition on $n$. The proof is based on the following scheme: 
\begin{tikzcd}[column sep = 110pt]
\text{Real world:} \sqrt n u^\top (\bar{\theta}_{n} - \thetas) \arrow[<->]{r}{\text{Gaussian approximation, Cor.}~\ref{cor:normal_approximation_markov}}  &
  \xi \sim \mathcal N(0,  \sigma^2(u) )  \arrow[<->]{d}{\text{Gaussian comparison, Lemma }\ref{lem:Pinsker}} 
\\
  \text{Bootstrap world: \quad} \bar{\theta}_{n, b}(u) \arrow[<->]{r}{\text{exactly matches the distribution}} &
\xi^\boot \sim \mathcal N(0,  \hat \sigma_{\theta}^2(u)) 
\end{tikzcd}

Due to \Cref{cor:normal_approximation_markov}, it holds that 
\begin{equation}
\textstyle
\kolmogorov{\sqrt{n} u^\top (\bar{\theta}_{n} - \thetas), \xi}
\lesssim_{\log{n}}   n^{-1/4} + n^{1/2-\gamma} + n^{\gamma-1}\eqsp,
\end{equation}
where $\xi \sim \mathcal{N}(0,1)$. This result allows for the first horizontal bar above. We now use the Gaussian comparison (see \cite{BarUly86}, \cite{devroye2018total}), which is specified in \Cref{lem:Pinsker}, which states that if $\xi_i \sim \mathcal{N}(0, \sigma_i^2)$, $i = 1,2$, are such that $|\sigma_1^2/\sigma_2^2 - 1| \le \delta$,
for some $\delta \geq 0$,
$\sup_{x \in \rset} |\PP(\xi_1 \le x) - \PP(\xi_2 \le x)| \le \frac{3}{2}\delta.$ We apply the Gaussian comparison between the limiting Gaussian $\xi \sim \mathcal{N}(0,\sigma^2(u))$ and a Gaussian random variable $\xi^b \sim \mathcal{N}(0,\hat{\sigma}^2_\theta(u))$. For this purpose, we need to obtain a high probability bound for $|\hat{\sigma}_\theta^2(u) - \sigma^2(u)|/\sigma^2(u)$. Combining \Cref{prop:block_bootstrap_batch_mean}, \Cref{prop:concentration_OBM} and Markov's inequality, we get that with probability at least $1-1/n$, 
\begin{align}
    |\hat{\sigma}_\theta^2(u) - \sigma^2(u)|/\sigma^2(u) &\lesssim_{\log{n}} b_n^{1/2}n^{\gamma/2-1} +b_n^{1/2}n^{-\gamma} + b_{n}^{-1/2}n^{\gamma/2} +n^{2\gamma-2} +  b^{1/2} n^{-1/2} \eqsp.
\end{align}
Now note that by construction $ \bar{\theta}_{n,b} \sim \mathcal{N}(0,\hat{\sigma}^2_\theta(u))$ under the bootstrap probability. Hence, with probability at least $1-1/n$, it holds that
\begin{align}
\kolmogorov{\sqrt{n}  u^\top (\bar{\theta}_{n} - \thetas),\bar{\theta}_{n,b}} 
\lesssim_{\log{n}} \frac{b_n^{1/2}}{n^{1-\gamma/2}} +\frac{b_n^{1/2}}{n^{\gamma/2}} + \frac{b_n^{1/2}}{n^{1/2}} + \frac{1}{n^{1/4}} + \frac{1}{n^{\gamma-1/2}} + \frac{1}{n^{1-\gamma}}\eqsp.
\end{align}
To complete the proof, it remains to optimize our choice of $\gamma$ and $b_n$, which yields to \eqref{eq:n_b_gamma_choice}.
\end{proof}

\paragraph{Discussion.} Non-asymptotic analysis of coverage probabilities has been carried out for the modifications of multiplier bootstrap approach of \cite{Fang2018} in recent papers \cite{samsonov2024gaussian,sheshukova2025gaussian}. These approaches showed that coverage probabilities of $\thetas$ can be approximated by their bootstrap counterparts with the order up to $\mathcal{O}(n^{-1/2})$. Yet for this bootstrap approach it is crucial to work in the independent noise setting. The attempt of \cite{ramprasad2023online} to generalize it for the case of Markovian noise yields inconsistent procedure, as shown in \cite[Proposition~1]{liu2023statistical}. That is why we prefer to start from the asymptotically consistent MSB procedure of \cite{MR4733879}. Our bootstrap validity proof relies on direct approximation of $\sqrt{n}  u^\top (\bar{\theta}_{n} - \thetas)$ by the limiting Gaussian distribution $\mathcal{N}(0,\sigma^2(u))$. At the same time, it is known (see \cite{sheshukova2025gaussian}) that in the \iid\ setting it can be more advantageous to approximate $\sqrt{n}  u^\top (\bar{\theta}_{n} - \thetas)$ by another Gaussian distribution $\mathcal{N}(0,\sigma_n^2(u))$. It remains an open question if this reasoning can be applied in case of dependent random variables. 

\vspace{-5pt}
\section{Application to the TD learning}
\label{sec:applications_td}
We illustrate our theoretical findings via the temporal‐difference (TD) learning algorithm \cite{sutton1988learning,sutton:book:2018} applied to policy evaluation in a discounted Markov decision process (MDP) $(\S,\A,\PMDP,r,\lambda)$. Here, $\S$ and $\A$ are complete metric spaces with their Borel $\sigma$-algebras; $\PMDP(\cdot\mid s,a)$ denotes the state–action transition kernel; $r:\S\times\A\to[0,1]$ is the deterministic reward; and $\lambda\in[0,1)$ is the discount factor. A policy $\pi(\cdot\mid s)$ is a Markov kernel from states to actions. Under $\pi$, the value function is given by 
$V_\pi(s)=\PE[\sum_{k=0}^{\infty}\lambda^k\,r(S_k,A_k)\,|\,S_0=s]$ with $A_k\sim\pi(\cdot\mid S_k)$, $S_{k+1}\sim\PMDP(\cdot\mid S_k,A_k)$. The induced state process $(S_k)$ is itself a Markov chain with kernel
$\PMDP_\pi(s'\mid s)=\int_\A \PMDP(s'\mid s,a)\,\pi(\rmd a\mid s)$. Equivalently
$V_\pi(s)=\sum_{k=0}^\infty\lambda^k\,\delta_s P_\pi^k\bar r$ with $\bar r(s)=\int r(s,a)\,\pi(\rmd a\mid s)$.
We approximate $V^\pi$ in a linear feature space:$V_\theta^\pi(s)=\varphi(s)^\top\theta$, $\varphi:\S\to\rset^d$,
and seek $\theta_\star\in\rset^d$. We impose two standard assumptions:
\begin{assumTD}\label{assum:markov_model}
$\PMDP_\pi$ is uniformly geometrically ergodic with unique invariant measure $\mu$ and mixing time $\tmix$.
\end{assumTD}
We also define the design matrix $\Sigma_{\varphi} = \PE_\mu[\varphi(s)\varphi(s)^\top]$ and require that it is non-degenerate:
\begin{assumTD}\label{assum:feature_design}
$\Sigma_{\varphi}$ is non-degenerate, i.e.\ $\lambda_{\min}(\Sigma_{\varphi})>0$, and $\sup_{s\in\S}\|\varphi(s)\|\le1$.
\end{assumTD}
Under these, the TD-learning is an instance of LSA with $A_k = \varphi(S_k)\bigl[\varphi(S_k)-\lambda\varphi(S_{k+1})\bigr]^\top$ and $b_k = \varphi(S_k)\,r(S_k,A_k)$ (see \Cref{appendix:td_learning}). Refining constants as in \cite{patil2023finite,samsonov2023finite}, one shows:
\begin{proposition}
\label{prop:assumption_check_TD}
Under \Cref{assum:markov_model} and \Cref{assum:feature_design}, the TD updates satisfy the noise‐level condition \ref{assum:noise-level} with
$\bConst{A}=2(1+\lambda)$ and $\supconsteps=2(1+\lambda)(\|\theta_\star\|+1)$. Moreover, for step‐size $\alpha\le\alpha_\infty=(1-\lambda)/(1+\lambda)^2$,$\|\Id-\alpha\bar A\|^2\;\le\;1-\alpha\,a$, where $a=(1-\lambda)\,\lambda_{\min}(\Sigma_{\varphi})$.
\end{proposition}
By \cite[Proposition~~2]{samsonov2024gaussian}, this ensures \Cref{prop:hurwitz_stability} holds with $Q=\Id$, and hence \Cref{th:bootstrap_validity_main} applies directly to the TD scheme. We provide numerical simulations in \Cref{appendix:td_learning}.

\section{Conclusion}
\label{sec:conclusion}
We presented the non-asymptotic Berry–Esseen bounds for Polyak–Ruppert averaged iterates of linear stochastic approximation algorithm under Markovian noise, achieving convergence rates of order up to $n^{-1/4}$ in Kolmogorov distance. Additionally, we established the theoretical validity of a multiplier subsample bootstrap procedure, enabling reliable uncertainty quantification in the setting of the LSA algorithm with Markovian noise. Our paper suggest a number of further research directions. One of them is related with the generalizations of \Cref{th:normal_approximation_markov_with_sigma_n} to the setting of non-linear SA algorithms, as well as with obtaining multivariate version of it. It is also an interesting and important question if our non-asymptotic bounds on coverage probabilities provided in \Cref{th:bootstrap_validity_main} can be further improved.

\clearpage
\newpage
\bibliographystyle{plain}
\bibliography{references}

\begin{thebibliography}{10}

\bibitem{agrawalla2023high}
Bhavya Agrawalla, Krishnakumar Balasubramanian, and Promit Ghosal.
\newblock High-dimensional central limit theorems for linear functionals of
  online least-squares sgd.
\newblock {\em arXiv preprint arXiv:2302.09727}, 2023.

\bibitem{aguech2000perturbation}
Rafik Aguech, Eric Moulines, and Pierre Priouret.
\newblock On a perturbation approach for the analysis of stochastic tracking
  algorithms.
\newblock {\em SIAM Journal on Control and Optimization}, 39(3):872--899, 2000.

\bibitem{pmlr-v99-anastasiou19a}
Andreas Anastasiou, Krishnakumar Balasubramanian, and Murat~A. Erdogdu.
\newblock Normal approximation for stochastic gradient descent via
  non-asymptotic rates of martingale {CLT}.
\newblock In Alina Beygelzimer and Daniel Hsu, editors, {\em Proceedings of the
  Thirty-Second Conference on Learning Theory}, volume~99 of {\em Proceedings
  of Machine Learning Research}, pages 115--137. PMLR, 25--28 Jun 2019.

\bibitem{archibald1995generation}
TW~Archibald, KIM McKinnon, and LC~Thomas.
\newblock On the generation of markov decision processes.
\newblock {\em Journal of the Operational Research Society}, 46(3):354--361,
  1995.

\bibitem{atchade2014martingale}
Yves~F Atchad{\'e} and Matias~D Cattaneo.
\newblock A martingale decomposition for quadratic forms of {M}arkov chains
  (with applications).
\newblock {\em Stochastic Processes and their Applications}, 124(1):646--677,
  2014.

\bibitem{BarUly86}
S.~{Barsov} and V.~{Ulyanov}.
\newblock {Estimates for the closeness of Gaussian measures}.
\newblock {\em Dokl. Akad. Nauk SSSR}, 291(2):273--277, 1986.

\bibitem{bentkus2004}
V.~Bentkus.
\newblock On the dependence of the {B}erry–{E}sseen bound on dimension.
\newblock {\em Journal of Statistical Planning and Inference}, 113(2):385--402,
  2003.

\bibitem{benveniste2012adaptive}
A.~Benveniste, M.~M{\'e}tivier, and P.~Priouret.
\newblock {\em Adaptive algorithms and stochastic approximations}, volume~22.
\newblock Springer Science \& Business Media, 2012.

\bibitem{bertail2004edgeworth}
Patrice Bertail and St{\'e}phan Cl{\'e}men{\c c}on.
\newblock Edgeworth expansions of suitably normalized sample mean statistics
  for atomic markov chains.
\newblock {\em Probability Theory and Related Fields}, 130(3):388--414, 2004.

\bibitem{bhandari2018finite}
J.~Bhandari, D.~Russo, and R.~Singal.
\newblock A finite time analysis of temporal difference learning with linear
  function approximation.
\newblock In {\em Conference On Learning Theory}, pages 1691--1692, 2018.

\bibitem{bolthausen1982markov}
E.~Bolthausen.
\newblock The berry-esse{\'e}n theorem for strongly mixing harris recurrent
  markov chains.
\newblock {\em Zeitschrift f{\"u}r Wahrscheinlichkeitstheorie und Verwandte
  Gebiete}, 60(3):283--289, 1982.

\bibitem{bolthausen1982martingale}
E.~Bolthausen.
\newblock {Exact Convergence Rates in Some Martingale Central Limit Theorems}.
\newblock {\em The Annals of Probability}, 10(3):672 -- 688, 1982.

\bibitem{bolthausen1980berry}
Erwin Bolthausen.
\newblock The berry-esseen theorem for functionals of discrete markov chains.
\newblock {\em Zeitschrift f{\"u}r Wahrscheinlichkeitstheorie und verwandte
  Gebiete}, 54(1):59--73, 1980.

\bibitem{borkar:sa:2008}
Vivek~S Borkar.
\newblock {\em Stochastic Approximation: A Dynamical Systems Viewpoint}.
\newblock Cambridge University Press, 2008.

\bibitem{chen2021statistical}
Haoyu Chen, Wenbin Lu, and Rui Song.
\newblock Statistical inference for online decision making via stochastic
  gradient descent.
\newblock {\em Journal of the American Statistical Association},
  116(534):708--719, 2021.

\bibitem{chen2020aos}
Xi~Chen, Jason~D. Lee, Xin~T. Tong, and Yichen Zhang.
\newblock {Statistical inference for model parameters in stochastic gradient
  descent}.
\newblock {\em The Annals of Statistics}, 48(1):251 -- 273, 2020.

\bibitem{Chernozhukov2013}
Victor Chernozhukov, Denis Chetverikov, and Kengo Kato.
\newblock Gaussian approximations and multiplier bootstrap for maxima of sums
  of high-dimensional random vectors.
\newblock {\em Ann. Statist.}, 41(6):2786--2819, 2013.

\bibitem{chernozhukov2017central}
Victor Chernozhukov, Denis Chetverikov, and Kengo Kato.
\newblock Central limit theorems and bootstrap in high dimensions.
\newblock {\em Annals of Probability}, 45(4):2309--2352, 2017.

\bibitem{damerdji1991strong}
Halim Damerdji.
\newblock Strong consistency and other properties of the spectral variance
  estimator.
\newblock {\em Management Science}, 37(11):1424--1440, 1991.

\bibitem{devroye2018total}
Luc Devroye, Abbas Mehrabian, and Tommy Reddad.
\newblock The total variation distance between high-dimensional gaussians with
  the same mean.
\newblock {\em arXiv preprint arXiv:1810.08693}, 2018.

\bibitem{durmus2020biassgd}
Aymeric Dieuleveut, Alain Durmus, and Francis Bach.
\newblock {Bridging the gap between constant step size stochastic gradient
  descent and Markov chains}.
\newblock {\em The Annals of Statistics}, 48(3):1348 -- 1382, 2020.

\bibitem{douc:moulines:priouret:soulier:2018}
R.~Douc, E.~Moulines, P.~Priouret, and P.~Soulier.
\newblock {\em Markov chains}.
\newblock Springer Series in Operations Research and Financial Engineering.
  Springer, 2018.

\bibitem{durmus2022finite}
Alain Durmus, Eric Moulines, Alexey Naumov, and Sergey Samsonov.
\newblock Finite-time high-probability bounds for {P}olyak–{R}uppert averaged
  iterates of linear stochastic approximation.
\newblock {\em Mathematics of Operations Research}, 50(2):935--964, 2025.

\bibitem{durmus2021tight}
Alain Durmus, Eric Moulines, Alexey Naumov, Sergey Samsonov, Kevin Scaman, and
  Hoi-To Wai.
\newblock Tight high probability bounds for linear stochastic approximation
  with fixed stepsize.
\newblock In M.~Ranzato, A.~Beygelzimer, K.~Nguyen, P.~S. Liang, J.~W. Vaughan,
  and Y.~Dauphin, editors, {\em Advances in Neural Information Processing
  Systems}, volume~34, pages 30063--30074. Curran Associates, Inc., 2021.

\bibitem{durmus2023rosenthal}
Alain Durmus, Eric Moulines, Alexey Naumov, Sergey Samsonov, and Marina
  Sheshukova.
\newblock Rosenthal-type inequalities for linear statistics of markov chains.
\newblock {\em arXiv preprint arXiv:2303.05838}, 2023.

\bibitem{durmus2021stability}
Alain Durmus, Eric Moulines, Alexey Naumov, Sergey Samsonov, and Hoi-To Wai.
\newblock On the stability of random matrix product with markovian noise:
  Application to linear stochastic approximation and td learning.
\newblock In Mikhail Belkin and Samory Kpotufe, editors, {\em Proceedings of
  Thirty Fourth Conference on Learning Theory}, volume 134 of {\em Proceedings
  of Machine Learning Research}, pages 1711--1752. PMLR, 15--19 Aug 2021.

\bibitem{efron1992bootstrap}
Bradley Efron.
\newblock Bootstrap methods: another look at the jackknife.
\newblock In {\em Breakthroughs in statistics: Methodology and distribution},
  pages 569--593. Springer, 1992.

\bibitem{fan2019exact}
Xiequan Fan.
\newblock Exact rates of convergence in some martingale central limit theorems.
\newblock {\em Journal of Mathematical Analysis and Applications},
  469(2):1028--1044, 2019.

\bibitem{Fang2018}
Yuan Fang, Jing Xu, and Lijun Yang.
\newblock Online bootstrap confidence intervals for the stochastic gradient
  descent estimator.
\newblock {\em Journal of Machine Learning Research}, 19(78):1--21, 2018.

\bibitem{flegal:2010}
James~M. Flegal and Galin~L. Jones.
\newblock Batch means and spectral variance estimators in {M}arkov chain
  {M}onte {C}arlo.
\newblock {\em Ann. Statist.}, 38(2):1034--1070, 04 2010.

\bibitem{fort:clt:markov:2015}
{G. Fort}.
\newblock Central limit theorems for stochastic approximation with controlled
  {M}arkov chain dynamics.
\newblock {\em ESAIM: PS}, 19:60--80, 2015.

\bibitem{gabcke2015neue}
Wolfgang Gabcke.
\newblock Neue herleitung und explizite restabsch{\"a}tzung der
  riemann-siegel-formel.
\newblock 2015.

\bibitem{geist2014off}
Matthieu Geist, Bruno Scherrer, et~al.
\newblock Off-policy learning with eligibility traces: a survey.
\newblock {\em J. Mach. Learn. Res.}, 15(1):289--333, 2014.

\bibitem{huang2020matrix}
De~Huang, Jonathan Niles-Weed, Joel~A Tropp, and Rachel Ward.
\newblock Matrix concentration for products.
\newblock {\em Foundations of Computational Mathematics}, pages 1--33, 2021.

\bibitem{Jirak2022}
Moritz Jirak and Martin Wahl.
\newblock Quantitative limit theorems and bootstrap approximations for
  empirical spectral projectors.
\newblock 2022.
\newblock arXiv:2202.08669.

\bibitem{kloeckner2019}
Beno\^it Kloeckner.
\newblock Effective {B}erry-{E}sseen and concentration bounds for {M}arkov
  chains with a spectral gap.
\newblock {\em Ann. Appl. Probab.}, 29(3):1778--1807, 2019.

\bibitem{konda1999actor}
Vijay Konda and John Tsitsiklis.
\newblock Actor-critic algorithms.
\newblock {\em Advances in neural information processing systems}, 12, 1999.

\bibitem{kunsch1989jackknife}
Hans~R Kunsch.
\newblock The jackknife and the bootstrap for general stationary observations.
\newblock {\em The annals of Statistics}, 17(3):1217--1241, 1989.

\bibitem{kushner2003stochastic}
Harold Kushner and G~George Yin.
\newblock {\em Stochastic approximation and recursive algorithms and
  applications}, volume~35.
\newblock Springer Science \& Business Media, 2003.

\bibitem{lakshminarayanan:2018a}
C.r Lakshminarayanan and Csaba Szepesvari.
\newblock Linear stochastic approximation: How far does constant step-size and
  iterate averaging go?
\newblock In Amos Storkey and Fernando Perez-Cruz, editors, {\em Proceedings of
  the Twenty-First International Conference on Artificial Intelligence and
  Statistics}, volume~84 of {\em Proceedings of Machine Learning Research},
  pages 1347--1355. PMLR, 2018.

\bibitem{li2023online}
Xiang Li, Jiadong Liang, and Zhihua Zhang.
\newblock Online statistical inference for nonlinear stochastic approximation
  with {M}arkovian data.
\newblock {\em arXiv preprint arXiv:2302.07690}, 2023.

\bibitem{li2023statistical}
Xiang Li, Wenhao Yang, Jiadong Liang, Zhihua Zhang, and Michael~I Jordan.
\newblock A statistical analysis of {P}olyak-{R}uppert averaged {Q}-learning.
\newblock In {\em International Conference on Artificial Intelligence and
  Statistics}, pages 2207--2261. PMLR, 2023.

\bibitem{liu2023statistical}
Ruiqi Liu, Xi~Chen, and Zuofeng Shang.
\newblock Statistical inference with {S}tochastic {G}radient {M}ethods under
  {$\phi$}-mixing {D}ata.
\newblock {\em arXiv preprint arXiv:2302.12717}, 2023.

\bibitem{MR4733879}
Ruru Ma and Shibin Zhang.
\newblock Multiplier subsample bootstrap for statistics of time series.
\newblock {\em J. Statist. Plann. Inference}, 233:Paper No. 106183, 15, 2024.

\bibitem{ElMachkouriOuchti2007}
Mohamed~El Machkouri and Lahcen Ouchti.
\newblock Exact convergence rates in the central limit theorem for a class of
  martingales.
\newblock {\em Bernoulli}, 13(4):981--999, 2007.

\bibitem{meketon1984overlapping}
Marc~S Meketon and Bruce Schmeiser.
\newblock Overlapping batch means: Something for nothing?
\newblock Technical report, Institute of Electrical and Electronics Engineers
  (IEEE), 1984.

\bibitem{mokkadem2006convergence}
Abdelkader Mokkadem, Mariane Pelletier, et~al.
\newblock Convergence rate and averaging of nonlinear two-time-scale stochastic
  approximation algorithms.
\newblock {\em The Annals of Applied Probability}, 16(3):1671--1702, 2006.

\bibitem{mou2020linear}
Wenlong Mou, Chris~Junchi Li, Martin~J Wainwright, Peter~L Bartlett, and
  Michael~I Jordan.
\newblock On linear stochastic approximation: Fine-grained polyak-ruppert and
  non-asymptotic concentration.
\newblock In {\em Conference on Learning Theory}, pages 2947--2997. PMLR, 2020.

\bibitem{mou2021optimal}
Wenlong Mou, Ashwin Pananjady, Martin~J Wainwright, and Peter~L Bartlett.
\newblock Optimal and instance-dependent guarantees for markovian linear
  stochastic approximation.
\newblock {\em arXiv preprint arXiv:2112.12770}, 2021.

\bibitem{moulines2011non}
Eric Moulines and Francis Bach.
\newblock Non-asymptotic analysis of stochastic approximation algorithms for
  machine learning.
\newblock {\em Advances in neural information processing systems}, 24:451--459,
  2011.

\bibitem{moulines2025note}
Eric Moulines, Alexey Naumov, and Sergey Samsonov.
\newblock A note on concentration inequalities for the overlapped batch mean
  variance estimators for {M}arkov chains.
\newblock {\em arXiv preprint arXiv:2505.08456}, 2025.

\bibitem{Naumov2019}
Alexey Naumov, Vladimir Spokoiny, and Vladimir Ulyanov.
\newblock Bootstrap confidence sets for spectral projectors of sample
  covariance.
\newblock {\em Probability Theory and Related Fields}, 174(3-4):1091--1132,
  2019.

\bibitem{nemirovski2009robust}
Arkadi Nemirovski, Anatoli Juditsky, Guanghui Lan, and Alexander Shapiro.
\newblock Robust stochastic approximation approach to stochastic programming.
\newblock {\em SIAM Journal on optimization}, 19(4):1574--1609, 2009.

\bibitem{osekowski:2012}
A.~Osekowski.
\newblock {\em Sharp Martingale and Semimartingale Inequalities}.
\newblock Monografie Matematyczne 72. Birkhäuser Basel, 1 edition, 2012.

\bibitem{patil2023finite}
Gandharv Patil, LA~Prashanth, Dheeraj Nagaraj, and Doina Precup.
\newblock Finite time analysis of temporal difference learning with linear
  function approximation: Tail averaging and regularisation.
\newblock In {\em International Conference on Artificial Intelligence and
  Statistics}, pages 5438--5448. PMLR, 2023.

\bibitem{paulin_concentration_spectral}
Daniel Paulin.
\newblock {Concentration inequalities for Markov chains by Marton couplings and
  spectral methods}.
\newblock {\em Electronic Journal of Probability}, 20(none):1 -- 32, 2015.

\bibitem{petrov1975sums}
V.~Petrov.
\newblock {\em Sums of Independent Random Variables}.
\newblock Ergebnisse der Mathematik und ihrer Grenzgebiete. 2. Folge. Springer
  Berlin Heidelberg, 1975.

\bibitem{pinelis_1994}
Iosif Pinelis.
\newblock {Optimum Bounds for the Distributions of Martingales in Banach
  Spaces}.
\newblock {\em The Annals of Probability}, 22(4):1679 -- 1706, 1994.

\bibitem{polyak1992acceleration}
Boris~T Polyak and Anatoli~B Juditsky.
\newblock Acceleration of stochastic approximation by averaging.
\newblock {\em SIAM journal on control and optimization}, 30(4):838--855, 1992.

\bibitem{rakhlin2012making}
Alexander Rakhlin, Ohad Shamir, and Karthik Sridharan.
\newblock Making gradient descent optimal for strongly convex stochastic
  optimization.
\newblock In {\em Proceedings of the 29th International Coference on
  International Conference on Machine Learning}, pages 1571--1578, 2012.

\bibitem{ramprasad2023online}
Pratik Ramprasad, Yuantong Li, Zhuoran Yang, Zhaoran Wang, Will~Wei Sun, and
  Guang Cheng.
\newblock Online bootstrap inference for policy evaluation in reinforcement
  learning.
\newblock {\em J. Amer. Statist. Assoc.}, 118(544):2901--2914, 2023.

\bibitem{riobook}
E.~Rio.
\newblock {\em Asymptotic Theory of Weakly Dependent Random Processes}.
\newblock Springer, 2017.

\bibitem{Rosenthal1970}
H.~Rosenthal.
\newblock On the subspaces of {$L^{p}$} {$(p>2)$} spanned by sequences of
  independent random variables.
\newblock {\em Israel J. Math.}, 8:273--303, 1970.

\bibitem{Ross}
Nathan Ross.
\newblock Fundamentals of {S}tein's method.
\newblock {\em Probab. Surv.}, 8:210--293, 2011.

\bibitem{roy2023online}
Abhishek Roy and Krishnakumar Balasubramanian.
\newblock Online covariance estimation for stochastic gradient descent under
  {M}arkovian sampling.
\newblock {\em arXiv preprint arXiv:2308.01481}, 2023.

\bibitem{rubin1981bayesian}
Donald~B Rubin.
\newblock The bayesian bootstrap.
\newblock {\em The annals of statistics}, pages 130--134, 1981.

\bibitem{ruppert1988efficient}
David Ruppert.
\newblock Efficient estimations from a slowly convergent robbins-monro process.
\newblock Technical report, Cornell University Operations Research and
  Industrial Engineering, 1988.

\bibitem{ROLLIN2018171}
Adrian Röllin.
\newblock On quantitative bounds in the mean martingale central limit theorem.
\newblock {\em Statistics \& Probability Letters}, 138:171--176, 2018.

\bibitem{samsonov2024gaussian}
Sergey Samsonov, Eric Moulines, Qi-Man Shao, Zhuo-Song Zhang, and Alexey
  Naumov.
\newblock Gaussian {A}pproximation and {M}ultiplier {B}ootstrap for
  {P}olyak-{R}uppert {A}veraged {L}inear {S}tochastic {A}pproximation with
  {A}pplications to {TD} {L}earning.
\newblock In {\em Advances in Neural Information Processing Systems},
  volume~37, pages 12408--12460. Curran Associates, Inc., 2024.

\bibitem{samsonov2023finite}
Sergey Samsonov, Daniil Tiapkin, Alexey Naumov, and Eric Moulines.
\newblock Improved {H}igh-{P}robability {B}ounds for the {T}emporal
  {D}ifference {L}earning {A}lgorithm via {E}xponential {S}tability.
\newblock In Shipra Agrawal and Aaron Roth, editors, {\em Proceedings of Thirty
  Seventh Conference on Learning Theory}, volume 247 of {\em Proceedings of
  Machine Learning Research}, pages 4511--4547. PMLR, 30 Jun--03 Jul 2024.

\bibitem{shao2022berry}
Qi-Man Shao and Zhuo-Song Zhang.
\newblock Berry--{E}sseen bounds for multivariate nonlinear statistics with
  applications to {M}-estimators and stochastic gradient descent algorithms.
\newblock {\em Bernoulli}, 28(3):1548--1576, 2022.

\bibitem{sheshukova2025gaussian}
Marina Sheshukova, Sergey Samsonov, Denis Belomestny, Eric Moulines, Qi-Man
  Shao, Zhuo-Song Zhang, and Alexey Naumov.
\newblock Gaussian approximation and multiplier bootstrap for stochastic
  gradient descent.
\newblock {\em arXiv preprint arXiv:2502.06719}, 2025.

\bibitem{spokoiny2015}
Vladimir Spokoiny and Mayya Zhilova.
\newblock {Bootstrap confidence sets under model misspecification}.
\newblock {\em The Annals of Statistics}, 43(6):2653 -- 2675, 2015.

\bibitem{srikant2024rates}
R~Srikant.
\newblock Rates of {C}onvergence in the {C}entral {L}imit {T}heorem for
  {M}arkov {C}hains, with an {A}pplication to {TD} {L}earning.
\newblock {\em arXiv preprint arXiv:2401.15719}, 2024.

\bibitem{srikant:1tsbounds:2019}
R.~{Srikant} and L.~{Ying}.
\newblock {Finite-Time Error Bounds For Linear Stochastic Approximation and TD
  Learning}.
\newblock In {\em Conference on Learning Theory}, 2019.

\bibitem{sutton1988learning}
R.~S Sutton.
\newblock Learning to predict by the methods of temporal differences.
\newblock {\em Machine learning}, 3(1):9--44, 1988.

\bibitem{sutton:book:2018}
R.~S. Sutton and Andrew~G. Barto.
\newblock {\em Reinforcement Learning: An Introduction}.
\newblock The MIT Press, second edition, 2018.

\bibitem{tsitsiklis:td:1997}
J.~N. {Tsitsiklis} and B.~{Van Roy}.
\newblock An analysis of temporal-difference learning with function
  approximation.
\newblock {\em IEEE Transactions on Automatic Control}, 42(5):674--690, May
  1997.

\bibitem{vanHandel2016}
Ramon Van~Handel.
\newblock {\em Probability in {H}igh {D}imension}.
\newblock APC 550 Lecture Notes, Princeton University, 2016.

\bibitem{zhu2023online_cov_matr}
Xi~Chen Wanrong~Zhu and Wei~Biao Wu.
\newblock Online {C}ovariance {M}atrix {E}stimation in {S}tochastic {G}radient
  {D}escent.
\newblock {\em Journal of the American Statistical Association},
  118(541):393--404, 2023.

\bibitem{watkins:qlearn:1992}
C.topher~J.C.H. Watkins and P.~Dayan.
\newblock Technical note: Q-learning.
\newblock {\em Machine Learning}, 8(3):279--292, May 1992.

\bibitem{welch1987relationship}
Peter~D Welch.
\newblock On the relationship between batch means, overlapping means and
  spectral estimation.
\newblock In {\em Proceedings of the 19th conference on Winter simulation},
  pages 320--323, 1987.

\bibitem{wu2024statistical}
Weichen Wu, Gen Li, Yuting Wei, and Alessandro Rinaldo.
\newblock {S}tatistical {I}nference for {T}emporal {D}ifference {L}earning with
  {L}inear {F}unction {A}pproximation.
\newblock {\em arXiv preprint arXiv:2410.16106}, 2024.

\bibitem{wu2025uncertainty}
Weichen Wu, Yuting Wei, and Alessandro Rinaldo.
\newblock Uncertainty quantification for {M}arkov chains with application to
  temporal difference learning.
\newblock {\em arXiv preprint arXiv:2502.13822}, 2025.

\bibitem{yu1994rates}
Bin Yu.
\newblock Rates of convergence for empirical processes of stationary mixing
  sequences.
\newblock {\em The Annals of Probability}, pages 94--116, 1994.

\bibitem{zhong2023online}
Yanjie Zhong, Todd Kuffner, and Soumendra Lahiri.
\newblock Online {B}ootstrap {I}nference with {N}onconvex {S}tochastic
  {G}radient {D}escent {E}stimator.
\newblock {\em arXiv preprint arXiv:2306.02205}, 2023.

\end{thebibliography}


\newpage
\appendix
\section{Constants}
\label{sec:constants}
\begin{table}[!ht]
\begin{tabular}{l l l}
\hline
\bfseries Constant name & \bfseries Description & \bfseries Reference \\
\hline
$\bConst{\sf{Rm}, 1} = 60 \rme$, $\bConst{\sf{Rm}, 2} = 60$ & Martingale Rosenthal constants & \cite[Th.~4.1]{pinelis_1994} \\
$\ConstD_{\ref{lem:auxiliary_rosenthal},1} = (16/3) \bConst{\sf{Rm}, 1}, \,\, \ConstD_{\ref{lem:auxiliary_rosenthal},2} = 8 \bConst{\sf{Rm}, 2}$ & Simple Rosenthal under \Cref{assum:UGE} & \Cref{lem:auxiliary_rosenthal} \\
\hline
\end{tabular}
\caption{Absolute constants appearing in our main results}
\label{tab:univ_constants}
\end{table}

\section{Extended version of \Cref{assum:step-size}}
\label{sec:Extended_version_of_A3}
\setcounter{assumprime}{2}
\begin{assumprime}
\label{assum:step-size_extended}
The sequence of step sizes $\{\alpha_{k}\}_{k \in \nset}$ has a form $\alpha_{k} = c_{0} / (k+k_0)^{\gamma}$, where $\gamma \in [1/2;1)$ and, with 
\begin{equation}
\label{eq:h_block_size_def}
h = \biggl\lceil\frac{16 \taumix \qcond^{1/2} \bConst{A}}{a}\biggr\rceil,
\end{equation}
it holds that 
\begin{equation}
\label{eq:c_0_bound_optimized}
c_{0} \leq \frac{1}{2a}\eqsp.
\end{equation}
Moreover, we assume that 
\begin{align}
\begin{aligned}
\label{eq:a3-sample-size}
 &n^{1-\gamma}\geq \max(\frac{2C'_{\infty}}{\lambda_{\min}(\Sigma_{\infty})}, \rme^{2(1-\gamma)})\eqsp,\\
 &k_0\geq \max\biggl\{\bigl(\frac{24}{ac_0}\bigr)^{1/(1-\gamma)}, c_0^{1/\gamma}, \biggl(c_0(h+1) \max(\alpha_{\infty}, \qcond^{1/2}\bConst{A}, 6\rme\qcond\bConst{A}^2/a)\biggr)^{1/\gamma},\\&\qquad \qquad \frac{8 \taumix \qcond^{1/2} \bConst{A}}{a}, \biggl(\frac{12(h+1)(\log n) c_0\bConst{\sigma}^2}{a}\biggr)^{1/\gamma}\biggr\}\eqsp,\\
 &k_0^{1-\gamma} \geq
    \frac{2}{c_{0}b} \biggl(\log\{\frac{c_0}{b(2\gamma-1)2^\gamma} \} + \gamma \log \{k_0 \} \biggr)\eqsp,
 \end{aligned}
\end{align}
\end{assumprime}

\section{Proof of \Cref{th:normal_approximation_markov_with_sigma_n}}
\label{appendix:proofs}
For simplicity, we denote 
\begin{equation}
    \label{def:gnm}
    g_{n:m} = \sum_{k=n}^m k^{-\gamma},~ n \leq m
\end{equation}

\begin{lemma}
\label{prop:gnm_lowerbound}
\; Suppose $n \leq m$. Therefore 
    \begin{equation}
        \frac{(m+1)^{1-\gamma} - n^{1-\gamma}}{1 - \gamma} \leq g_{n:m} \leq  \frac{m^{1-\gamma} - (n-1)^{1-\gamma}}{1-\gamma}
    \end{equation}
\end{lemma}

\begin{lemma}
\label{lem:sum_as_Qell}
Let $b, c_0 > 0$ and $\alpha_\ell = c_0 (\ell+k_0)^{-\gamma}$ for $\gamma \in (1/2, 1)$, $k_0 \geq 0$. Assume that $bc_0 < 1$ and $k_0^{1-\gamma} \geq 2/(bc_0)$. Then, it holds that
\begin{align}
    \sum_{k=\ell}^{n-1} \alpha_\ell\prod_{j=\ell+1}^{k} (1-b\alpha_j) \leq \mathcal{L}_b
\end{align}
where we set 
\begin{equation}
\label{eq:const_L_b_def}
\mathcal{L}_b = c_0 + \frac{2}{b(1-\gamma)}\eqsp.
\end{equation}
\end{lemma}
\begin{proof}
Note that
\begin{equation}
    \sum_{k=\ell}^{n-1}\prod_{j=\ell+1}^{k} (1-b\alpha_j) \leq  \sum_{k=\ell}^{n-1}\exp\biggl\{-b\sum_{j=l+1}^{k}\alpha_j\biggr\} \leq \sum_{k=\ell+k_0}^{n+k_0-1}\exp\biggl\{-\frac{bc_0}{2(1-\gamma)}(k^{1-\gamma}-(l+k_0)^{1-\gamma})\biggr\}
\end{equation}
Applying \Cref{lem:bound_sum_exponent} with $k_0^{1-\gamma} \geq 2/(bc_0)$, we finish the proof.
\end{proof}

\begin{lemma}
\label{prop:Qell:bound}
Assume \Cref{assum:UGE}, \Cref{assum:noise-level}, and \Cref{assum:step-size}. Then for any $\ell \in \nset$ it holds that 
\begin{equation}
\normop{Q_\ell} \leq \mathcal{L}_Q \eqsp,
\end{equation}
where
\begin{equation}
\label{def:Lq}
\mathcal{L}_Q = \qcond^{1/2} \bigl(c_0 + \frac{4}{a (1-\gamma)}\bigr)\end{equation}
and $\qcond$ is defined in \Cref{prop:hurwitz_stability}.
\end{lemma}
\begin{proof}
Using the definition of $Q_{\ell}$ from \eqref{eq:Q_ell_definition_main} and \Cref{prop:hurwitz_stability}, we get
\begin{align}
\norm{Q_\ell} 
&\leq \qcond^{1/2} \alpha_\ell \sum_{k=\ell}^{n-1} \norm{G_{\ell+1:k}}[Q] \leq \qcond^{1/2} \alpha_\ell \sum_{k=\ell}^{n-1}\prod_{j=\ell+1}^{k}\sqrt{1 - a \alpha_{j}} \\
&\leq \qcond^{1/2} \alpha_\ell \sum_{k=\ell}^{n-1}\prod_{j=\ell+1}^{k} (1 - (a/2) \alpha_{j}) \leq \mathcal{L}_Q\eqsp,
\end{align}
where in the last bound we applied \Cref{lem:sum_as_Qell} with $b = a/2$.
\end{proof}

\begin{lemma} 
 \label{repr:qtminusa}
The following identity holds true
    \begin{equation}
    Q_\ell - \bA^{-1} = S_\ell - \bA^{-1} G_{\ell:n-1},~ S_\ell = \sum_{j=\ell+1}^{n-1} (\alpha_
    \ell- \alpha_j) G_{\ell+1:j-1} \eqsp,
\end{equation}
and
\begin{equation}
    \label{repr:sumqtminusa}
    \sum_{i=1}^{n-1} (Q_i - A^{-1}) = -G^{-1} \sum_{j=1}^{n-1} G_{1:j}\eqsp.
\end{equation}
\end{lemma}
\begin{proof}
    See \cite[pp. 26-30]{wu2024statistical}
\end{proof}

\begin{lemma}
    \label{prop:st_bound}
    Let $c_0 \in (0, \alpha_{\infty}]$ and $\ell \in \mathbb{N}$. Then under \Cref{assum:noise-level}, it holds
    \begin{equation}
        \label{bound:stmatrix}
        \normop{S_\ell} \leq \sqrt{\qcond} \cdot C_{\gamma, \beta}^{(S)} \cdot (\ell+k_0)^{\gamma-1} \eqsp, 
    \end{equation}
    where 
    \begin{equation}
    \label{eq:conts_S_l}
    C_{\gamma, a}^{(S)} = 2c_0\exp\biggl\{\frac{a c_0}{2k_0^{\gamma}}\biggr\}\biggl(2^{\gamma/(1-\gamma)}\frac{2}{a c_0} + (\frac{2}{a c_0})^{1/(1-\gamma)}\Gamma(\frac{1}{1-\gamma})\biggr)\eqsp.
    \end{equation}
\end{lemma}
\begin{proof}
    For simplicity we define $m_i^j = \sum_{k=i}^j (k+k_0)^{-\gamma}$ and $\beta = a/2$.
Note that 
\begin{equation}
    \norm{\sum_{j=i+1}^{n-1} (\alpha_i - \alpha_j) G_{i+1:j-1}^{(\alpha)}} \leq \sqrt{\qcond}\sum_{j=i}^{n-2} \frac{c_0}{(j+k_0+1)^\gamma}\biggl(\biggl(\frac{j+k_0+1}{i+k_0}\biggr)^\gamma - 1\biggr) \exp\{-\beta c_0m_{i+1}^{j}\}
\end{equation}
    Following the proof of \cite[Lemma A.7]{wu2024statistical}, we have 
    \begin{equation}
        \biggl(\frac{j+k_0+1}{i+k_0}\biggr)^\gamma - 1 \leq (i+k_0)^{\gamma-1}\biggl(1+ (1-\gamma)m_{i}^{j}\biggr)^{\gamma/(1-\gamma)}
    \end{equation}
    Hence, we obtain 
    \begin{align}
        \frac{\norm{S_i}}{\sqrt{\qcond}} &\leq c_0(i+k_0)^{\gamma-1}\sum_{j=i}^{n-2} \frac{1}{(j+k_0+1)^\gamma}\biggl(1+ (1-\gamma)m_{i}^{j}\biggr)^{\gamma/(1-\gamma)} \exp\{-\beta c_0m_{i+1}^{j}\}\\ &\leq
        c_0(i+k_0)^{\gamma-1}\sum_{j=i}^{n-2} \frac{1}{(j+k_0)^\gamma}\biggl(1+ (1-\gamma)m_{i}^{j}\biggr)^{\gamma/(1-\gamma)}\exp\{\beta c_0(k_0+i)^{-\gamma}\} \exp\{-\beta c_0m_{i}^{j}\}\\&\leq 
         c_0\exp\{\frac{\beta c_0}{k_0^{\gamma}}\}(i+k_0)^{\gamma-1}\sum_{j=i}^{n-2} (m_i^j-m_i^{j-1})\biggl(1+ (1-\gamma)m_{i}^{j}\biggr)^{\gamma/(1-\gamma)} \exp\{-\beta c_0m_{i}^{j}\}\\ &\leq 
         2c_0\exp\{\frac{\beta c_0}{k_0^{\gamma}}\}(i+k_0)^{\gamma-1}\int_{0}^{+\infty}\biggl(1+ (1-\gamma)m\biggr)^{\gamma/(1-\gamma)} \exp\{-\beta c_0m\}\rmd m \\ &\leq 
         2c_0\exp\{\frac{\beta c_0}{k_0^{\gamma}}\}(i+k_0)^{\gamma-1}\biggl(2^{\gamma/(1-\gamma)}\frac{1}{\beta c_0} + (\frac{1}{\beta c_0})^{1/(1-\gamma)}\Gamma(\frac{1}{1-\gamma})\biggr)\eqsp.
    \end{align}
\end{proof}

\begin{lemma}
    \label{prop:Q_ell_diff_bound}
    Let $c_0 \in (0, \alpha_{\infty}]$ and $\ell \in \mathbb{N}$. Then under \Cref{assum:noise-level}, it holds
    \begin{equation}
        \label{bound:Q_ell_diff_bound}
        \normop{Q_{\ell+1}-Q_{\ell}} \leq \mathcal{L}_{Q,2} \cdot \alpha_{\ell+1} \eqsp, 
    \end{equation}
    where 
    \begin{equation}
    \label{eq:conts_Q_ell_diff}
    \mathcal{L}_{Q,2} =\sqrt{\qcond}\biggl(2^{\gamma} + (2\bConst{A} + a/4)(c_0 + \frac{4}{a(1-\gamma)})\biggr) \eqsp.
    \end{equation}
\end{lemma}
\begin{proof}
Using the definition of $Q_{\ell}$ we have 
\begin{align}
       Q_{\ell+1}-Q_{\ell} &= \alpha_{\ell +1}\sum_{k= l+1}^{n-1}G_{\ell+2:k} -\alpha_{\ell}\sum_{k = l}^{n-1}G_{\ell+1:k}=
       \alpha_{\ell +1}\sum_{k= l+1}^{n-1}G_{\ell+2:k+1} -\alpha_{\ell}\sum_{k = l+1}^{n-1}G_{\ell+1:k} - \alpha_{\ell} \\&=
       -\alpha_{\ell}+\sum_{k=l+1}^{n-1}G_{l+2:k}(\alpha_{\ell+1}I - \alpha_{\ell}I+\alpha_{\ell+1}\alpha_{\ell}\bA)\eqsp.
\end{align}
    Hence, using \Cref{lem:bound_ratio_step_size} with $r=a/4$ and \Cref{prop:Qell:bound} we get 
    \begin{align}
          \norm{Q_{\ell+1}-Q_{\ell}}\leq \alpha_{\ell} + \sum_{k=l+1}^{n-1}\alpha_{\ell+1}^2\norm{G_{l+2:k}}(2\bConst{A} + a/4) \leq (2^{\gamma} + \mathcal{L}_{Q}(2\bConst{A} + a/4))\alpha_{\ell+1}\eqsp.
    \end{align}   
\end{proof}

\begin{lemma}
\label{prop:st_diff_bound}
      Let $c_0 \in (0, \alpha_{\infty}]$, $k_0^{1-\gamma}>bc_0$ and $\ell \in \mathbb{N}$. Then under \Cref{assum:noise-level} it holds
    \begin{equation}
        \label{bound:st_diff_matrix}
        \normop{S_{\ell+1}-S_{\ell}} \leq \sqrt{\qcond} \cdot C_{\gamma, a}^{(S,2)} \cdot \alpha_{\ell+1} \eqsp, 
    \end{equation}
    where 
    \begin{equation}
    \label{eq:const_S_l_diff}
        C_{\gamma, a}^{(S,2)}  = (c_0+\frac{4}{a(1-\gamma)})(a/2 + 3\bConst{A})\eqsp.
    \end{equation}
\end{lemma}
\begin{proof}
    Using \Cref{repr:qtminusa} we obtain 
    \begin{align}
       S_{\ell+1}-S_{\ell} &= Q_{\ell+1}-Q_{\ell} + \bA^{-1}(G_{\ell+1:n-1} - G_{\ell:n-1}) \\&= \alpha_{\ell +1}\sum_{k= l+1}^{n-1}G_{\ell+2:k} -\alpha_{\ell}\sum_{k = l}^{n-1}G_{\ell+1:k} + \bA^{-1}(I- (I-\alpha_\ell\bA))G_{\ell+1:n-1}\\&=
       \alpha_{\ell +1}\sum_{k= l}^{n-2}G_{\ell+2:k+1} -\alpha_{\ell}\sum_{k = l}^{n-2}G_{\ell+1:k} \\&=
       (\alpha_{\ell+1}-\alpha_\ell)I + \sum_{k=l+1}^{n-2}(\alpha_{\ell+1}(I-\alpha_{k+1}\bA) - \alpha_{\ell}(I-\alpha_{\ell+1}\bA))G_{l+2:k}\eqsp.
    \end{align}
    Hence, using \Cref{lem:bound_ratio_step_size} with $r=a/4$ we get 
    \begin{align}
     \normop{S_{\ell+1}-S_{\ell}} &\leq (a/4)\alpha_{\ell+1}^2 + \sum_{k=l+1}^{n-2} ((a/4) \alpha_{\ell+1}^2 + \bConst{A} \alpha_{\ell+1}^2 + 2^\gamma\bConst{A}\alpha_{\ell+1}^2)\normop{G_{l+2:k}} \\&\leq (\mathcal{L}_{Q}((a/4) + \bConst{A} + 2^\gamma\bConst{A}) + (a/4))\alpha_{\ell+1}\eqsp,
    \end{align}
    where in the last inequality we applying \cref{prop:Qell:bound}.
    Since $\mathcal{L}_{Q} \geq 1$ and $2^\gamma<2$ we complete the proof.
\end{proof}

\begin{lemma}
    \label{prop:g_t:n_upperbound}
    Let $c_0 \in (0, \alpha_{\infty}]$. Then under \Cref{assum:noise-level} for any $m \in \mathbb{N}$ it holds
    \begin{equation}
        \sum_{t=1}^{n-1} \normop{G_{t:n-1}}^m \leq \frac{\qcond^{m/2}}{1 - (1 - c_0(a/2)(n+k_0-2)^{-\gamma})^m}
    \end{equation}
\end{lemma}
\begin{proof}
Note that 
\begin{align}
    \sum_{t=1}^{n-1} \normop{G_{t:n-1}}^m &\leq \qcond^{m/2}\sum_{l=1}^{n-1}\prod_{i=t}^{n-1}(1-(a/2)\alpha_i)\\& = \frac{\qcond^{m/2}}{(1-(1-(a/2)\alpha_{n-2})^{m})}\sum_{l=1}^{n-1}(1-(1-(a/2)\alpha_{t-1})^{m})\prod_{i=t}^{n-1}(1-(a/2)\alpha_i)^{m} \\&\leq \frac{\qcond^{m/2}}{(1-(1-(a/2)\alpha_{n-2})^{m})}\eqsp.
\end{align}
\end{proof}

\begin{lemma}
\label{lem:bound_norm_diff_variance}
     Let $c_0 \in (0, \alpha_{\infty}]$. Then under \Cref{assum:noise-level} it holds
     \begin{equation}
         |\sigma_n^2(u)-\sigma^2(u)| \leq C'_{\infty}n^{\gamma-1}\eqsp,
     \end{equation}
     where 
     \begin{equation}
         \label{eq:def_C'_infty_const}
         C'_{\infty} = \normop{\noisecov}\frac{\qcond (C_{\gamma,a}^{(S)})^2}{2\gamma-1} +\normop{\Sigma_\infty}\biggl(\frac{2^{\gamma+1}5 + 3}{ac_0(1-\gamma)}k_0^\gamma  + \frac{4\bConst{A}\qcond C_{\gamma,a}^{(S)}}{ac_0}k_0^{2\gamma-1}\biggr)
     \end{equation}
\end{lemma}
\begin{proof}
Note that 
\begin{equation}
    |\sigma^2_n(u)-\sigma^2(u)|\leq \norm{\Sigma_n -\Sigma_{\infty}}\eqsp.
\end{equation}
Then, we express the $\Sigma_n -\Sigma_{\infty}$  in the following form:
    \begin{align}
    \Sigma_n -\Sigma_{\infty} = \underbrace{\frac{1}{n}\sum_{t=2}^{n-1} (Q_t - \bA^{-1})\noisecov \bA^{-\top} + \frac{1}{n} \sum_{t=2}^{n-1} \bA^{-1} \noisecov (Q_t - \bA^{-1})^\top}_{R_1} \\ + \underbrace{\frac{1}{n} \sum_{t=2}^{n-1} (Q_t - \bA^{-1}) \noisecov (Q_t - \bA^{-1})^{\top}}_{R_2} - \frac{2}{n} \Sigma_{\infty}\eqsp. 
    \end{align}
First, we bound $R_1$, using \cref{lem:sum_as_Qell} we obtain 
\begin{align}
    \norm{\frac{1}{n}\sum_{t=2}^{n-1} (Q_t - \bA^{-1})\noisecov \bA^{-\top}} &= \norm{-\frac{1}{n} \bA^{-1} \sum_{j=2}^{n-1} G_{1:j}\noisecov \bA^{-\top}} \\ &\leq \norm{n^{-1} \Sigma_{\infty}  \sum_{j=2}^{n-1} G_{1:j}} \leq n^{-1} \norm{\Sigma_{\infty}} \cdot \sum_{j=1}^{n-1}\norm{ G_{1:j}}\\& \leq n^{-1} \norm{\Sigma_{\infty}} (1+\frac{4}{ac_0(1-\gamma)})(1+k_0)^{\gamma}\\&\leq 2^{\gamma}n^{\gamma-1} \norm{\Sigma_{\infty}} (1+\frac{4}{ac_0(1-\gamma)})k_0^{\gamma}\eqsp.
\end{align}
Hence, we get that 
\begin{equation}
    \normop{R_1} \leq 2^{\gamma+1}n^{\gamma-1} \norm{\Sigma_{\infty}} (1+\frac{4}{ac_0(1-\gamma)})k_0^{\gamma}\eqsp.
\end{equation}
Now, we rewrite term $R_2$ as follows:
\begin{align}
    \label{lemma:R2_expanded_repr}
    &n^{-1}\sum_{t=2}^{n-1}(Q_t - \bA^{-1}) \noisecov (Q_t - \bA^{-1})^\top \\
    &\qquad = n^{-1}\sum_{t=2}^{n-1}\left(S_t - \bA^{-1} \prod_{k=t}^{n-1} (\Id - \alpha_k \bA)\right) \noisecov \left(S_t - \bA^{-1} \prod_{k=t}^{n-1} (\Id - \alpha_k \bA)\right)^\top \\
    &\qquad = \underbrace{n^{-1} \sum_{t=2}^{n-1} S_t \noisecov S_t^\top}_{R_{21}} + \underbrace{n^{-1}\sum_{t=2}^{n-1}\bA^{-1} \prod_{k=t}^{n-1} (\Id - \alpha_k \bA) \noisecov \bA^{-\top} \prod_{k=t}^{n-1} (\Id - \alpha_k \bA)^\top}_{R_{22}} \\ 
    &\qquad\qquad \qquad  -\underbrace{n^{-1} \sum_{t=2}^{n-1} \bA^{-1} \prod_{k=t}^{n-1} (\Id - \alpha_k \bA) \cdot \noisecov S_t^\top}_{R_{23}} - \underbrace{n^{-1}\sum_{t=2}^{n-1} S_t \noisecov \bA^{-\top} \prod_{k=t}^{n-1} (\Id - \alpha_k \bA)^{\top}}_{R_{24}}\eqsp.
\end{align}
To bound $\normop{R_{21}}$ we use 
\Cref{prop:st_bound} and obtain 
\begin{align}
    \label{lemma:R21_bound}
    \normop{R_{21}} &= \norm{n^{-1} \sum_{t=2}^{n-1} S_t \noisecov S_t^\top} \leq n^{-1} \sum_{t=2}^{n-1} \normop{\noisecov} \normop{S_t}^2 \\&\leq n^{-1} \normop{\noisecov} \sum_{t=2}^{n-1} \qcond \left(C_{\gamma,a}^{(S)}\right)^2 (t+k_0)^{2(\gamma-1)} \\ &\leq
    n^{-1} \normop{\noisecov} \qcond \left(C_{\gamma,a}^{(S)}\right)^2 \frac{(n+k_0-1)^{2\gamma - 1} - (k_0+1)^{2\gamma-1}}{2\gamma - 1} \\ &\leq n^{2(\gamma - 1)} \frac{\normop{\noisecov} \qcond \left(C_{\gamma,a}^{(S)}\right)^2}{2\gamma - 1} \eqsp,
\end{align}
where the last inequality holds since $(n+k_0-1)^{2\gamma-1} \leq n^{2\gamma-1}+(k_0+1)^{2\gamma-1}$ for $\gamma\in(1/2,1)$.
 The bound for $R_{22}$ follows from  \Cref{prop:g_t:n_upperbound} and simple inequality $(n+k_0-2)^{\gamma} \leq (k_0n)^{\gamma}$:
    \begin{align} 
        \label{lemma:R22_bound}
        \normop{R_{22}} &= \normop{n^{-1}\sum_{i=2}^{n-1} G_{i:n-1} \bA^{-1} \noisecov \bA^{-\top}  G_{i:n-1}^\top} \leq n^{-1} \normop{\Sigma_\infty} \sum_{i=1}^{n-1} \normop{G_{i:n-1}}^2 \\ 
        &\leq 
        n^{-1} \frac{\normop{\Sigma_\infty}}{ c_0a(n+k_0-2)^{-\gamma} - c_0^2(a^2/4)(n+k_0-2)^{-2\gamma}} \leq 
         2\normop{\Sigma_\infty} k_0^{\gamma}\frac{n^{\gamma-1}}{c_0a}\eqsp.
    \end{align}
    Since $R_{23} = R_{24}^\top$, we concentrate on $\normop{D_{24}}$. \Cref{prop:g_t:n_upperbound}  immediately imply 
    \begin{align}
        \normop{R_{24}} & \leq n^{-1} \normop{\noisecov \bA^{-\top}} \sum_{i=1}^{n-1} \normop{S_i} \normop{G_{i:n-1}} \\ &\leq n^{-1} \normop{\bA\Sigma_{\infty}}\qcond C_{\gamma,a}^{(S)}\sum_{i=1}^{n-1}(i+k_0)^{\gamma-1}\prod_{k=i}^{n-1}(1-\frac{ac_0}{2(k+k_0)^\gamma}) \\ &\leq n^{-1} \normop{\bA\Sigma_{\infty}}\qcond C_{\gamma,a}^{(S)}\sum_{i=1}^{n-1}(i+k_0)^{2\gamma-1}(i+k_0)^{-\gamma}\prod_{k=i+1}^{n-1}(1-\frac{ac_0}{2(k+k_0)^\gamma})\\ &\leq \bConst{A}\normop{\Sigma_{\infty}}\qcond C_{\gamma,a}^{(S)}k_0^{2\gamma-1}\frac{2n^{2(\gamma-1)}}{ac_0}
    \end{align}
    By combining all the inequalities, we complete the proof. 
\end{proof}

\begin{lemma}
    \label{lem:bound_sigma_n}
    Let $c_0 \in (0, \alpha_{\infty}]$ and $C'_{\infty}n^{\gamma-1} \leq \lambda_{\min}(\Sigma_{\infty})/2$. Then under \Cref{assum:noise-level} it holds
    \begin{equation}
        \sigma_n^{-1} \leq C_{\Sigma}\eqsp,
    \end{equation}
    where $C_{\Sigma} = \sqrt{\frac{2}{\lambda_{\min}(\Sigma_{\infty})}}$.
\end{lemma}
\begin{proof}
Note that $\sigma_n^2 >\lambda_{\min}(\Sigma_n)$. Using Lidskii’s inequality, we obtain
\begin{equation}
    \lambda_{\min}(\Sigma_n) \geq \lambda_{\min}(\Sigma_{\infty}) - \norm{\Sigma_n-\Sigma_{\infty}} \geq \lambda_{\min}(\Sigma_{\infty})/2\eqsp,
\end{equation}
where in the last inequality we use $C'_{\infty}n^{\gamma-1} \leq \lambda_{\min}(\Sigma_{\infty})/2$.
\end{proof}

Recall that $D = D_1 + D_2$, where
\begin{align}
D_1 &=  \frac{1}{\sqrt n} \sum_{k = 0}^{n-1} \ProdB_{1:k} (\theta_0 - \thetas) +\frac{1}{\sqrt{n}} \sum_{k=1}^{n-1} H_k^{(0)}\funnoisew(\State_\ell)\eqsp,  \\
D_2 &= -\frac{1}{\sqrt n} Q_1 \hat \funnoisew (\State_1) +\frac{1}{\sqrt n}  Q_{n-1} \MKQ \hat \funnoisew (\State_{n-1}) + \sum_{\ell=1}^{n-2}(Q_{\ell}- Q_{\ell+1})\MKQ\hat \funnoisew(\State_{\ell})\eqsp.
\end{align}
\begin{proposition}
\label{prop:D_term_bound}
Assume \Cref{assum:UGE}, \Cref{assum:noise-level}, and \Cref{assum:step-size} and $k_0\geq \max(\bigl(\frac{24}{ac_0}\bigr)^{1/(1-\gamma)}, c_0^{1/\gamma})$. Then for any $2 \leq p \leq \log{n}$, $u \in \sphere_{d-1}$, and any initial distribution $\xi$ on $(\Zset,\Zsigma)$, it holds that
\begin{align}
\begin{aligned}
\label{eq:D_term_bound}
\PE_{\xi}^{1/(p+1)}\left[ \bigl| (u^{\top} D)/\sigma_n \bigr|^{p}\right]& \lesssim (\Constupd{1,1}\|\theta_0 - \thetas \| + \Constupd{1,2})n^{p/(2p+2)} + \Constupd{1,3}\{p^2n^{1/2-\gamma}\}^{p/(p+1)} \\& \qquad \qquad + \Constupd{1,4}\{p^2\sqrt{\log n} n^{1/2-\gamma}\}^{p/(p+1)}\eqsp,
\end{aligned}
\end{align}
where 
\begin{equation}
\label{eq:D_term_bound_constants}
\begin{split}
\Constupd{1,1} &= C_{\Sigma}C_\Gamma d^{1/{\log{n}}}\frac{k_0^\gamma}{ac_0(1-\gamma)}\ \eqsp, \\
\Constupd{1,2} &= C_{\Sigma}\supconsteps\mathcal{L}_{Q}\taumix\eqsp, \\
\Constupd{1,3} &= 
\frac{C_{\Sigma}c_0\taumix}{1-\gamma}\biggl(\mathcal{L}_{Q,2}\supconsteps+ \sqrt{\log \frac{k_0^{\gamma}}{c_0}}\ConstDM_{3}\biggr)\eqsp,\\
\Constupd{1,4} &=\frac{C_{\Sigma}\ConstDM_{3}c_0\taumix\sqrt{\gamma}}{1-\gamma}\eqsp.
\end{split}
\end{equation}
Moreover setting $p = \log n$ we have
\begin{align}
\begin{aligned}
\label{eq:D_term_bound_p=logn}
\PE_{\xi}^{1/(p+1)}\left[ \bigl| (u^{\top} D)/\sigma_n \bigr|^{p}\right]& \lesssim (\Constupd{1}\|\theta_0 - \thetas \| + \Constupd{2})n^{-1/2} + \Constupd{3}(\log n)^2n^{1/2-\gamma} \\& \qquad \qquad + \Constupd{4}(\log n)^{5/2} n^{1/2-\gamma}\eqsp,
\end{aligned}
\end{align}
and 
\begin{equation}
\label{eq:D_term_bound_constants_p=logn}
    \Constupd{i} = \sqrt{\rme}\Constupd{1,i} \qquad \text{for } i\in \{1, 2, 3, 4\}
\end{equation}
\end{proposition}
\begin{proof} 
Using \Cref{lem:bound_sigma_n} we get $\PE_{\xi}^{1/p}\left[ \bigl| (u^{\top} D)/\sigma_n \bigr|^{p}\right] \leq C_{\Sigma}\PE^{1/p}_{\xi}\bigl[|u^\top D|^{p}\bigr]$
In order to bound $\PE^{1/p}_{\xi}\bigl[|u^\top D|^{p}\bigr]$, we bound the terms from $D_1$ to $D_2$ separately. 
Denote
\begin{align}
D_{11} = \frac{1}{\sqrt n} \sum_{k = 0}^{n-1} u^{\top}\ProdB_{1:k} (\theta_0 - \thetas), \, D_{12} = \frac{1}{\sqrt{n}} \sum_{k=1}^{n-1}  u^{\top}\Hnalpha{k}{0} \eqsp. 
\end{align}
Minkowski's inequality, \Cref{prop:products_of_matrices_UGE}, \Cref{lem:bounds_on_sum_step_sizes} and \Cref{lem:bound_sum_exponent} imply
\begin{align}
\begin{aligned}
\label{eq:bound_D11}
\PE_\xi^{1/p}[\|D_{11}\|^p] &\lesssim 
  C_\Gamma d^{1/{\log{n}}} \frac{1}{\sqrt n} \sum_{k = 0}^{n-1} \exp\biggl\{ -(a / 12) \sum_{j=1}^{k} \alpha_{j} \biggr\} \|\theta_0 - \thetas \| \\& \lesssim C_\Gamma d^{1/{\log{n}}} \frac{1}{\sqrt n} \sum_{k = 0}^{n-1} \exp\biggl\{ -\frac{a c_0}{24(1-\gamma)}((k+k_0)^{1-\gamma}-k_0^{\gamma})  \biggr\} \|\theta_0 - \thetas \|
  \\&\lesssim  C_\Gamma d^{1/{\log{n}}}\frac{1}{\sqrt n}\frac{k_0^\gamma}{ac_0(1-\gamma)}\|\theta_0 - \thetas \|\eqsp.
  \end{aligned}
\end{align}
It follows from the Minkowski inequality and \Cref{prop:H_n_0_bound_Markov}
that
\begin{align}
\begin{aligned}
\label{eq:bound_D22}
\PE_\xi^{1/p}[\|D_{12}\|^p] &\le \frac{1}{\sqrt n} \sum_{k=1}^{n-1}  \PE_\xi^{1/p}[\|u^\top \Hnalpha{k}{0}\|^p] \le  \frac{\ConstDM_{3} \taumix  p^{2}}{\sqrt n} \sum_{k = 1}^{n-1} \alpha_{k} \sqrt{\log{(1/\alpha_k)}} \\& \lesssim
\frac{\ConstDM_{3} \taumix  p^{2}c_0}{1-\gamma}\sqrt{\log\frac{(n+k_0-1)^\gamma}{c_0}}\frac{((n-1+k_0)^{1-\gamma}-k_0^{1-\gamma})}{\sqrt{n}}
\\&\lesssim  \frac{\ConstDM_{3} \taumix  p^{2}c_0}{1-\gamma}(\sqrt{\gamma\log n} + \sqrt{\frac{k_0^{\gamma}}{c_0}})n^{1/2-\gamma}\eqsp.
\end{aligned}
\end{align}

Denote 
\begin{align}
    D_{21} &= -\frac{1}{\sqrt n} u^{\top}Q_1 \hat \funnoisew (\State_1)\eqsp, \\
    D_{22} & = \frac{1}{\sqrt n} u^{\top} Q_{n-1} \MKQ \hat \funnoisew (\State_{n-1}) \eqsp,\\
    D_{32} & = \sum_{\ell=1}^{n-2}u^{\top}(Q_{\ell}- Q_{\ell+1})\MKQ\hat \funnoisew(\State_{\ell})\eqsp. \\
\end{align}
For $D_{21}$ and $D_{22}$ we use \cref{prop:Qell:bound} and $\normop{\hat \funnoisew (z)} \leq (8/3)\taumix \supconsteps$ and obtain 
\begin{equation}
  \PE^{1/p}_{\xi}[|D_{21}]^p] + \PE^{1/p}_{\xi}[|D_{22}]^p] \lesssim \frac{\taumix\supconsteps\mathcal{L}_{Q}}{\sqrt{n}}
\end{equation}
Finally, The bound of $\PE_\xi^{1/p}[|D_{23}|^p]$ follows from \Cref{prop:Q_ell_diff_bound}, $\normop{\hat \funnoisew (z)} \leq (8/3)\taumix \supconsteps$  and Minkovski's inequality 
\begin{align}
\PE_\xi^{1/p}[|D_{23}|^p] &\lesssim \frac{\mathcal{L}_{Q,2} \taumix \supconsteps}{\sqrt n} \sum_{\ell=2}^{n-1}\alpha_{\ell}\\ &\lesssim \frac{\mathcal{L}_{Q,2} \taumix \supconsteps c_0}{1-\gamma}n^{1/2-\gamma} \eqsp.
\end{align}

\end{proof}

\subsection{Proof of \Cref{cor:normal_approximation_markov}
}
\label{sec:cor:normal_approximation_markov}
\begin{proof}
 Let $\Phi(x)$ is the c.d.f. of the standard normal law $\mathcal{N}(0,1)$, $\Phi_{\sigma}(x)$ is the c.d.f. of the normal law $\mathcal{N}(0,\sigma(u))$ and $\Phi_{\sigma_n}(x)$ is the c.d.f. of the normal law $\mathcal{N}(0,\sigma_n(u))$. Then we have
    \begin{align}
    \kolmogorov{\sqrt{n} u^\top (\bar{\theta}_{n} - \thetas) / \sigma(u)} &= \sup_{x \in \rset}|\PP(\sqrt{n} u^\top (\bar{\theta}_{n} - \thetas)/\sigma(u) \leq x) - \Phi(x)| \\&=
    \sup_{x \in \rset}|\PP(\sqrt{n} u^\top (\bar{\theta}_{n} - \thetas) \leq x) - \Phi_\sigma(x)| \\ &\leq 
     \sup_{x \in \rset}|\PP(\sqrt{n} u^\top (\bar{\theta}_{n} - \thetas) \leq x) - \Phi_{\sigma_n}(x)| + \sup_{x \in \rset}|\Phi_{\sigma}(x)- \Phi_{\sigma_n}(x)|\\&= \kolmogorov{\sqrt{n} u^\top (\bar{\theta}_{n} - \thetas) / \sigma_n(u)} + \sup_{x \in \rset}|\Phi_{\sigma}(x)- \Phi_{\sigma_n}(x)|
    \end{align}
     Using \Cref{lem:bound_norm_diff_variance} and \Cref{lem:bound_sigma_n} we obtain 
     \begin{equation}
         |\sigma^2(u)/\sigma^2_n(u) -1| \leq \frac{C'_\infty C_{\Sigma}^2}{n^{1-\gamma}}\eqsp.
     \end{equation}
     Setting 
     \begin{equation}
        \label{def:C_infty_const}
            C_{\infty} = \frac{3}{2}C'_\infty \frac{2}{\lambda_{\min}(\Sigma_{\infty})}\eqsp, \qquad \text{where }C'_{\infty} \text{ is defined in \eqref{eq:def_C'_infty_const}}
    \end{equation}
     and applying \Cref{lem:Pinsker} we conclude the proof.
\end{proof}

\section{Last iterate bound}
\label{appendix:last_iterate}
\label{sec:last_iterate_bound_proof}
The main result of this section is the following bound on the $p$-th moment of the LSA iterate $\theta_{k} - \thetas$.
\begin{proposition}
\label{prop:last_iterate_bound}
Assume \Cref{assum:UGE}, \Cref{assum:noise-level}, and \Cref{assum:step-size}. Then for any $2 \leq p \leq \log{n}$, $ k \geq 1$, $u \in \sphere_{d-1}$, and any initial distribution $\xi$ on $(\Zset,\Zsigma)$, it holds that
\begin{equation}
\label{eq:p-th-moment-markov}
\PE_{\xi}^{1/p}\left[ \bigl| u^{\top}(\theta_k - \thetas) \bigr|^{p}\right] \lesssim \ConstDM_{1} \taumix \sqrt{p \alpha_{k}} + C_\Gamma d^{1/\log n}\exp\biggl\{ -(a / 12) \sum_{\ell=1}^{k} \alpha_{\ell} \biggr\} \norm{\theta_0 - \thetas}\eqsp,
\end{equation}
where the constant $\ConstDM_1$ is given by 
\begin{equation}
\label{eq:definition:ConstDM_1}
\ConstDM_1 = d^{1/2+1/\log n} \bigl( C_\Gamma \bConst{A} \ConstDM_{2} / a\bigr) \supconsteps  +  \|\varepsilon\|_{\infty} (\qcond/a)^{1/2} (3 + 4\bConst{A}/a)\eqsp.
\end{equation}
\end{proposition}
 We provide the result above only for $1$-dimensional projections of the error $u^{\top}(\theta_n - \thetas)$, as previously considered in \cite{mou2020linear}. Note that the scaling of the right-hand side of \eqref{eq:p-th-moment-markov} with a $\taumix$ factor can be suboptimal due to our usage of McDiarmid's inequality \Cref{lem:bounded_differences_norms_markovian}. However, this scaling corresponds to the second-order (in $n$) terms. Tighter analysis of this term using the Rosenthal-type inequality should reveal a $\sqrt{\taumix}$ dependence of the leading term, however, we leave this improvement for future work.
 
Our proof of the last iterate bound is based on the perturbation-expansion framework \cite{aguech2000perturbation}, see also \cite{durmus2022finite}. Within this framework, we expand the error recurrence \eqref{eq:main_recurrence_1_step}, using the notation $\ProdBa_{m:k}$ for the product of random matrices:
\begin{equation} 
\label{eq:definition-Phi} 
\ProdBa_{m:k}  = \prod_{i=m}^k (\Id - \alpha_{i} \funcA{Z_i} ) \eqsp, \quad m,k \in\nset, \quad m \leq k \eqsp,
\end{equation}
with the convention that $\ProdBa_{m:k}=\Id$ for $m > k$. Using the recurrence \eqref{eq:main_recurrence_1_step}, we arrive at the following decomposition of the LSA error:
\begin{equation} 
\label{eq:decomp_main_error}
u^{\top}(\theta_{k} - \thetas) = u^{\top}\utheta_{k} + u^{\top}\vtheta_{k}\eqsp,
\end{equation}
where we have defined
\begin{equation}
\label{eq:LSA_recursion_expanded}
\utheta_{k} =  u^{\top} \ProdBa_{1:k} \{ \theta_0 - \thetas \} \eqsp, \quad \vtheta_{k} = - \sum_{j=1}^k u^{\top} \ProdBa_{j+1:k} \alpha_{j} \funcnoise{Z_j}\eqsp.
\end{equation}
Here $\utheta_{k}$ is a transient term (reflecting the forgetting of the initial condition) and $\vtheta_{k}$ is a fluctuation term (reflecting misadjustement noise). We treat the $\utheta_{k}$ and $\vtheta_{k}$ terms separately. In particular, we control $\utheta_{k}$ using \Cref{prop:products_of_matrices_UGE}. For estimating $\vtheta_{k}$ we use the decomposition
\begin{equation}
\label{eq:decomp_fluctuation}
u^{\top} \vtheta_{k} = u^{\top}\Jnalpha{k}{0}+ u^{\top}\Hnalpha{k}{0} \eqsp,
\end{equation}
where the latter terms are defined by the following pair of recursions
\begin{align}
\label{eq:jn0_main}
&\Jnalpha{k}{0} =\left(\Id - \alpha_{k} \bA\right) \Jnalpha{k-1}{0} - \alpha_{k} \funcnoise{\State_{k}}\eqsp, && \Jnalpha{0}{0}=0\eqsp, \\[.1cm]
\label{eq:hn0_main}
&\Hnalpha{k}{0} =\left( \Id - \alpha_{k} \funcA{\State_{k}} \right) \Hnalpha{k-1}{0} - \alpha_{k} \zmfuncA{\State_{k}} \Jnalpha{k-1}{0}\eqsp, && \Hnalpha{0}{0}=0\eqsp.
\end{align}
For notation convenience we introduce, for $m \leq k$, the deterministic counterpart of the product of random matrices \eqref{eq:definition-Phi}, that is,
\begin{equation}
\label{eq:prod_determ_matrix}
G_{m:k} = \prod_{i=m}^k (\Id - \alpha_{i} \bA) \eqsp,
\end{equation}
keeping the convention $G_{m:k} = \Id$ if $m > k$. Thus we obtain that 
\begin{equation}
\label{eq:jh_recur_main}
J_{k}^{(0)} = -\sum_{j=1}^{k} \alpha_j G_{j+1:k} \funcnoise{\State_j}, \quad H_{k}^{(0)} = - \sum_{j=1}^{k} \alpha_j \Gamma_{j+1:k} \zmfuncA{Z_j} J_{j-1}^{(0)} \eqsp,
\end{equation}
and we analyze these terms above separately. We first bound the term $\Jnalpha{n}{0}$:
\begin{proposition}
\label{prop:J_n_0_bound_Markov}
Assume \Cref{assum:UGE}, \Cref{assum:noise-level}, and \Cref{assum:step-size}. Let $k_0 \geq \{\frac{16\gamma}{ac_0}\}^{1/{1-\gamma}}$. Then, for any $p \geq 2$, initial probability measure $\xi$ on $(\Zset,\Zsigma)$, and $k \geq 1$, it holds that
\begin{equation}
\label{eq:J_n_0_bound_Markov}
\PE^{1/p}_{\xi}\bigl[\bigl|u^{\top}\Jnalpha{k}{0}\bigr|^{p}\bigr]  \leq \ConstDM_{2} \taumix \sqrt{p \alpha_{k}} \eqsp,
\end{equation}
where 
\begin{equation}
\label{eq:const_D_2_Markov}
\ConstDM_{2} =  (32/3) 
\|\varepsilon\|_{\infty} (\qcond/a)^{1/2} (3 + 4\bConst{A}/a) \eqsp.
\end{equation}
\end{proposition}
\begin{proof}
Note that $J_{k}^{(0)}$ is an additive functional of $\{ \funcnoise{Z_j} \}_{j=1}^{k}$. Using the representation \eqref{eq:jh_recur_main}, we obtain that 
\[
u^{\top} J_{k}^{(0)} = -\sum_{j=1}^{k} \alpha_j u^{\top} G_{j+1:k} \funcnoise{\State_j}\eqsp.
\]
Applying \Cref{prop:hurwitz_stability}, we get that  
\[
\normop{G_{j+1:k}} \leq \qcond^{1/2} \prod_{\ell=j+1}^{k} \sqrt{1 - a \alpha_\ell}\eqsp.
\]
Using \Cref{lem:auxiliary_rosenthal_weighted} with  $A_j = \alpha_{j} G_{j+1:k}$, we obtain that
\begin{align}
    \PE_{\xi}^{1/p}[|u^TJ_{k}^{(0)}|^p] &\leq (16/3)p^{1/2}\taumix\|\varepsilon\|_{\infty}\bigl(\sum_{j=2}^{k}\alpha_j^2\norm{G_{j+1:k}}^2\bigr)^{1/2} \\&\qquad +(8/3)\taumix(\alpha_1\norm{G_{2:k}} + \alpha_k + \sum_{j=1}^{k-1}\norm{\alpha_jG_{j+1:k}-\alpha_{j-1}G_{j:k}})\|\varepsilon\|_{\infty}\eqsp.
\end{align}
Using \Cref{lem:sum_alpha_k_squared} with $b = a$, we get 
\[
\sum_{j=1}^{k} \alpha_{j}^2 \normop{G_{j+1:k}}^2 \leq \qcond \sum_{j=1}^{k} \alpha_{j}^2 \prod_{\ell=j+1}^{k} (1 - a \alpha_\ell) \leq  4 (\qcond /a) \alpha_{k} \eqsp.
\]
Applying \Cref{lem:prod_alpha_k} with $b = a/2$, we get $\alpha_1\norm{G_{2:k}} \leq \alpha_{k}$. Moreover, using \Cref{lem:bound_ratio_step_size} with $r = a/4$, we get that 
\[
\norm{\alpha_jG_{j+1:k}-\alpha_{j-1}G_{j:k}} = \norm{G_{j+1:k}(\alpha_j \Id - \alpha_{j-1}\Id + \alpha_{j-1}\alpha_j \bA)} \leq \alpha_{j}^2 \norm{G_{j+1:k}}  (2\bConst{A} + a/4)\eqsp.
\]
Then, applying \Cref{lem:sum_alpha_k_squared} with $b = a/2$, we get
\begin{align}
\alpha_1\norm{G_{2:k}} + \alpha_k + \sum_{j=2}^k\norm{\alpha_jG_{j+1:k}-\alpha_{j-1}G_{j:k}} &\leq 2\alpha_k + \qcond^{1/2}(a/4 + 2\bConst{A})\sum_{j=1}^k\alpha_j^2\prod_{l=j+1}^{k}(1 -a\alpha_\ell/2) \\
&\leq \alpha_k(2 + 8\qcond^{1/2}(a/4 + 2\bConst{A})/a)\eqsp.
\end{align}
Combining the above results and using that $p \geq 2$ yields \eqref{eq:const_D_2_Markov}.
\end{proof}

\subsection{Proof of \Cref{prop:last_iterate_bound}}
Proceeding as in \eqref{eq:decomp_main_error} and \eqref{eq:decomp_fluctuation}, we obtain that 
\begin{equation}
\label{eq:theta_n_bound_Markov_general}
\PE_\xi^{1/p}\left[\bigl| u^{\top}(\theta_{k} - \thetas) \bigr|^{p}\right]
\leq \PE^{1/p}_\xi\left[\bigl|u^{\top} \ProdBa_{1:k} \{ \theta_0 - \thetas \} \bigr|^{p}\right] + \PE^{1/p}_\xi\left[\bigl|u^{\top} \Jnalpha{k}{0} \bigr|^{p}\right]
+\PE^{1/p}_\xi\left[\bigl|u^{\top} \Hnalpha{k}{0} \bigr|^{p}\right]\eqsp.
\end{equation}
The first two terms are bounded using \Cref{prop:products_of_matrices_UGE} and \Cref{prop:J_n_0_bound_Markov}, respectively. In order to bound the term, corresponding to $\Hnalpha{k}{0}$, we apply Minkowski's inequality together with H\"older's inequlity  to \eqref{eq:jh_recur_main}:
\begin{equation}
\PE^{1/p}_\xi\left[\bigl|u^{\top} \Hnalpha{k}{0} \bigr|^{p}\right] 
\leq \sum_{j=1}^{k} \alpha_{j}\bConst{A} \bigr\{\PE_{\xi}\bigl[\normop{\ProdBa_{j+1:k}}^{2p}\bigr]\bigr\}^{1/2p} \bigl\{ \PE_{\xi}\bigl[\bigl|\Jnalpha{j-1}{0}\bigr|^{2p}\bigr]\bigr\}^{1/2p} \eqsp.
\end{equation}
Note that 
\begin{align}
\begin{aligned}
\label{eq:bound_norm_through_scalar_norm}
     \bigl\{ \PE_{\xi}\bigl[\bigl\|\Jnalpha{j-1}{0}\bigr\|^{2p}\bigr]\bigr\}^{1/p} &=  \bigl\{ \PE_{\xi}\bigl[(\sum_{r=1}^d |e_r^T\Jnalpha{j-1}{0}|^2)^{p}\bigr]\bigr\}^{1/p} \\&\leq \sum_{r=1}^d\bigl\{\PE_{\xi}\bigl[|e_r^T\Jnalpha{j-1}{0}|^{2p} \bigr]\bigr\}^{1/p}\eqsp.
\end{aligned}
\end{align}
Using \Cref{prop:J_n_0_bound_Markov} and \cref{prop:products_of_matrices_UGE} with a simple inequality $\rme^{-x} \leq 1 - x/2$, valid for $x \in [0,1]$,  we get
\begin{align}
\PE^{1/p}_\xi\left[\bigl|u^{\top} \Hnalpha{k}{0} \bigr|^{p}\right]  
&\lesssim C_\Gamma\bConst{A} \ConstDM_{2} \supconsteps 
 \taumix \sqrt{p} d^{1/2+1/\log n} \sum_{j=1}^{k} \alpha_{j}^{3/2} \prod_{\ell = j+1}^k(1-a\alpha_\ell/24) \\
&\lesssim   d^{1/2+1/\log n} \bigl( C_\Gamma \bConst{A} \ConstDM_{2} / a\bigr) \supconsteps \taumix \sqrt{p \alpha_{k}}\eqsp,
\end{align}
where in the last inequality we use \Cref{lem:sum_alpha_k_squared}.

\subsection{Proof of $J_n^{(1)}$ bound}
\label{sec:J_n_1_bound_proof}
In order to proceed further, we need to obtain the tighter bound on $\Hnalpha{k}{0}$, which requires a tighter moment bound on the quality $\Jnalpha{k}{1}$.

\begin{proposition}
\label{prop:H_n_0_bound_Markov}
Assume \Cref{assum:UGE}, \Cref{assum:noise-level}, and \Cref{assum:step-size}. Then, for any $p \geq 2$, initial probability measure $\xi$ on $(\Zset,\Zsigma)$, and $k \geq 1$, it holds that
\begin{equation}
\label{eq:H_n_0_bound_Markov}
\PE^{1/p}_{\xi}\bigl[\bigl|u^{\top}\Hnalpha{k}{0}\bigr|^{p}\bigr]  \lesssim \ConstDM_{3} \taumix p^{2} \alpha_{k} \sqrt{\log{(1/\alpha_k)}} \eqsp,
\end{equation}
where 
\begin{equation}
\label{eq:const_D_3_Markov}
\ConstDM_{3} = \ConstDM_{4} (1 + \frac{d^{1/2+1/\log n} C_\Gamma \bConst{A}}{a})\eqsp,
\end{equation}
and the constant $\ConstDM_{4}$ is defined in \eqref{eq:const_D_4_Markov}.
\end{proposition}
\begin{proof}
It is known (see e.g. \cite{aguech2000perturbation} and \cite{durmus2022finite}), that the term $\Hnalpha{k}{0}$ can be further decomposed as follows:
\begin{equation}
\label{eq:error_decomposition_LSA}
\Hnalpha{k}{0} = \sum_{\ell=1}^{L}\Jnalpha{k}{\ell} + \Hnalpha{k}{L}\eqsp.
\end{equation}
Here the parameter $L \geq 1$ control the depth of expansion, and the terms $\Jnalpha{k}{\ell}$ and $\Hnalpha{k}{\ell}$ are given by the following recurrences:
\begin{equation}
\label{eq:jn_allexpansion_main}
\begin{aligned}
&\Jnalpha{k}{\ell} =\left(\Id - \alpha_{k} \bA\right) \Jnalpha{k-1}{\ell} - \alpha_{k} \zmfuncA{\State_{k}} \Jnalpha{k-1}{\ell-1}\eqsp,
&& \Jnalpha{0}{\ell}=0  \eqsp, \\
& \Hnalpha{k}{\ell} =\left( \Id - \alpha_{k} \funcA{\State_{k}} \right) \Hnalpha{k-1}{\ell} - \alpha_{k} \zmfuncA{\State_{k}} \Jnalpha{k-1}{\ell} \eqsp, && \Hnalpha{0}{\ell}=0 \eqsp.
\end{aligned}
\end{equation}
The expansion depth $L$ here controls the desired approximation accuracy. For our further results it is enough to take $L = 1$ and estimate the respective terms $\Jnalpha{k}{1}$ and $\Hnalpha{k}{1}$. Now the rest of the proof follows from \Cref{prop:J_n_1_bound_Markov} (see the bounds \eqref{eq:J_n_1_bound_Markov} and \eqref{eq:H_n_1_bound_Markov} for $\PE^{1/p}_{\xi}\bigl[\bigl|u^{\top}\Jnalpha{k}{1}\bigr|^{p}\bigr]$ and $\PE^{1/p}_{\xi}\bigl[\bigl|u^{\top}\Hnalpha{k}{1}\bigr|^{p}\bigr]$, respectively), and Minkoski's inequality.
\end{proof}

\begin{proposition}
\label{prop:J_n_1_bound_Markov}
Assume \Cref{assum:UGE}, \Cref{assum:noise-level}, and \Cref{assum:step-size}. Then, for any $p \geq 2$, initial probability measure $\xi$ on $(\Zset,\Zsigma)$, and $k \geq 1$, it holds that 
\begin{equation}
\label{eq:J_n_1_bound_Markov}
\PE^{1/p}_{\xi}\bigl[\bigl|u^{\top}\Jnalpha{k}{1}\bigr|^{p}\bigr]  \lesssim \ConstDM_{4} \taumix p^{2} \alpha_{k} \sqrt{\log{(1/\alpha_k)}}\eqsp,
\end{equation}
where 
\begin{equation}
\label{eq:const_D_4_Markov}
\ConstDM_{4} = \frac{\qcond^{3/2} \bConst{A} \supconsteps}{a} + \frac{\qcond^{3/2} \bConst{A} \supconsteps}{a^{3/2}}3^{\gamma/2} +  \frac{\qcond^{3/2} \bConst{A} \supconsteps}{a} 3^{\gamma}\biggl(\frac{8\gamma}{ac_0}\biggr)^{\gamma/(1-\gamma)}\eqsp.
\end{equation}
Moreover, 
\begin{equation}
\label{eq:H_n_1_bound_Markov}
\PE^{1/p}_{\xi}\bigl[\bigl|u^{\top}\Hnalpha{k}{1}\bigr|^{p}\bigr]  \lesssim \ConstDM_{5} \taumix p^{2} \alpha_{k} \sqrt{\log{(1/\alpha_k)}} \eqsp,
\end{equation}
where
\begin{equation}
\label{eq:const_D_5_Markov}
\ConstDM_{5} = \frac{d^{1/2+1/\log n}\ConstDM_{4} C_\Gamma \bConst{A}}{a}\eqsp.
\end{equation}
\end{proposition}
\begin{proof}
Solving the recursion in \eqref{eq:jn_allexpansion_main} yields the double summation:
\[
\Jnalpha{k}{1}= -\sum_{\ell=1}^{k} \alpha_{\ell} G_{\ell+1:k} \zmfuncA{\State_{\ell}} \Jnalpha{\ell-1}{0} =  \sum_{\ell=1}^{k} \alpha_\ell \sum_{j=1}^{\ell-1}\alpha_j G_{\ell+1:k} \zmfuncA{\State_{\ell}}   G_{j+1:\ell-1} \funcnoise{\State_j}\eqsp.
\]
Changing the order of summation yields
\begin{equation}
\label{eq:J1alternative}
\Jnalpha{k}{1} = \sum_{j=1}^{k-1} \alpha_j \bigg\{\sum_{\ell=j+1}^{k} \alpha_{\ell} G_{\ell+1:k} \zmfuncA{\State_{\ell}}   G_{j+1:\ell-1} \bigg \}  \funcnoise{\State_j} = \sum_{j=1}^{k-1} \alpha_j S_{j+1:k} \funcnoise{\State_j},
\end{equation}
where for $j \leq k$ we have defined
\begin{equation}
\label{eq:S_j_k_def}
S_{j:k}:= \sum_{\ell=j}^{k} \alpha_{\ell} G_{\ell+1:k} \zmfuncA{\State_{\ell}} G_{j:\ell-1}\eqsp.
\end{equation}
Fix a constant $m \in \nset$, $m \leq k$ (to be determined later). Then we can rewrite $S_{j+1:k}$ as
\begin{align}
S_{j+1:k}&= \sum_{\ell=j+1}^{j+m} \alpha_\ell G_{\ell+1:k} \zmfuncA{\State_{\ell}}   G_{j+1:\ell-1} +  \sum_{\ell=j+m+1}^{k} \alpha_{\ell} G_{\ell+1:k} \zmfuncA{\State_{\ell}}   G_{j+1:\ell-1} \\
&=G_{j+m+1:k} S_{j+1:j+m} + S_{j+m+1:k} G_{j+1:j+m}\eqsp.
\end{align}
Let $N: = \round{k/m}$. In these notations, we can express $\Jnalpha{k}{1}$ as a sum of three terms:
\begin{align}
  \Jnalpha{k}{1} &= \underbrace{\sum_{j=1}^{m(N-1)} \alpha_j G_{j+m+1:k} S_{j+1:j+m} \funcnoise{\State_j}}_{T_1} + \underbrace{\sum_{j=1}^{m(N-1)} \alpha_j S_{j+m+1:k} G_{j+1:j+m} \funcnoise{\State_j}}_{T_2} \\
  & \quad + \underbrace{\sum_{j=m(N-1) + 1}^{k-1} \alpha_j S_{j+1:k} \funcnoise{\State_j}}_{T_3}\eqsp.
\end{align}
Consider the first term $T_1$. Applying Minkowski's inequality, \Cref{lem:S_j_k_def_matrices}, and \Cref{lem:sum_alpha_k_squared}, we get that 
\begin{equation}
\label{eq:T1}
\begin{split}
    \PE_{\xi}^{1/p}[\bigl|u^{\top} T_1 \bigr|^{p}] 
    & \overset{(a)}{\lesssim}   \qcond^{3/2} \bConst{A} p^{1/2} \taumix^{1/2} \supconsteps \sum_{j=1}^{(m-1)N} \alpha_j \bigl(\sum_{r=j+1}^{j+m}\alpha_{r}^2\bigr)^{1/2} \prod_{\ell = j+1}^{k}  \sqrt{1 - a \alpha_\ell} \\
    & \lesssim  \qcond^{3/2} \bConst{A} \sqrt{m p \taumix} \supconsteps \sum_{j=1}^{(m-1)N} \alpha_j^2  \prod_{\ell = j+1}^{k}\sqrt{1 - a \alpha_\ell} \\
    & \overset{(b)}{\lesssim} \Const{1} \sqrt{m p \taumix}  \alpha_{k}\eqsp,
    \end{split}
\end{equation}
where we set 
$$
\Const{1} = (\qcond^{3/2}/a) \bConst{A} \supconsteps\eqsp.
$$
In the line (a) above we used \Cref{lem:S_j_k_def_matrices}, and in (b) we used \Cref{lem:sum_alpha_k_squared}. Similarly, using the same lemmas, we bound the term $T_3$:
\begin{equation}
\label{eq:T3}
\PE_{\xi}^{1/p}\bigl[\bigl|u^{\top} T_3 \bigr|^{p}\bigr] \lesssim \Const{1} \sqrt{m p} \taumix \alpha_{k}\eqsp.
\end{equation}
Note that the second term $T_2$ is non-zero only for $N \geq 2$ and it can be rewritten as $T_2 = T_{21} + T_{22}$, where
\begin{align}
T_{21}&:= \sum_{j=0}^{N-2} \sum_{i=1}^{m} \alpha_{jm + i} S_{(j+1)m + i + 1:k} G_{jm+i+1:(j+1)m+i} \funcnoise{\State_{jm+i}^{*}}, \\
T_{22}&:=\sum_{j=0}^{N-2} \sum_{i=1}^{m} \alpha_{jm + i} S_{(j+1)m + i + 1:k} G_{jm+i+1:(j+1)m+i} (\funcnoise{\State_{jm+i}} - \funcnoise{\State_{jm+i}^{*}})\eqsp.
\end{align}
The set of random variables $\State_{jm+i}^{*}$ is constructed for each $i \in [1, m]$, with $\{\State_{jm+i}^{*}\}_{j=0}^{N-2}$ having the following properties:
\begin{equation} 
\label{eq:zstar}
    \begin{split}
    & \text{1.~~$\State_{jm+i}^{*}$ is independent of $\mathfrak F_{(j+1)m + i}^{k}: = \sigma\{Z_{(j+1)m+i}, \ldots, Z_{k}\}$}; \\
    & \text{2.~~}\P_{\xi}(\State_{jm+i}^{*} \neq \State_{jm+i}) \leq 2\, (1/4)^{\lceil m/\taumix \rceil}; \\
    & \text{3.~~$\State_{jm+i}^{*}$ and $\State_{jm+i}$ have the same distribution},
    \end{split}
\end{equation}
The existence of the random variables $Z_{jm+i}^*$ is guaranteed by Berbee's lemma, see e.g~\cite[Lemma 5.1]{riobook}, together with the fact that uniformly geometrically ergodic Markov chains are a special instance of $\beta$-mixing processes. We control $\beta$-mixing coefficient via total variation distance, see \cite[Theorem F.3.3]{douc:moulines:priouret:soulier:2018}. In order to analyze the term $T_{21}$ we use Minkovski's and Burkholder's inequality \cite[Theorem~8.6]{osekowski:2012}, and obtain that:
\begin{align}
\PE_{\xi}^{1/p}[\bigl| u^{\top} T_{21}\bigr|^p]  &\leq \sum_{i=1}^{m}  \PE_{\xi}^{1/p}\bigg[ \biggl| \sum_{j=0}^{N-2} \alpha_{jm + i} u^{\top} S_{(j+1)m + i + 1:k} G_{jm+i+1:(j+1)m+i} \funcnoise{\State_{jm+i}^{*}} \biggr| ^p  \bigg] \\
&\leq p \sum_{i=1}^{m} \bigg(\sum_{j=0}^{N-2} \alpha_{jm + i }^2 \PE^{2/p}[\bigl|u^{\top} S_{(j+1)m + i + 1:k} G_{jm+i+1:(j+1)m+i} \funcnoise{\State_{jm+i}^{*}}  \bigr|^{p}] \bigg)^{1/2} \\
&\leq p\sqrt{m} \bigg(\sum_{j=1}^{k} \alpha_{j}^2 \PE^{2/p}\bigl[\bigl|u^{\top} S_{j+m+1:k} G_{j+1:j+m} \funcnoise{\State_{j}^{*}}  \bigr|^{p}\bigr] \bigg)^{1/2}\eqsp.
\end{align}
Applying now \Cref{lem:S_j_k_def_matrices}, we arrive at the bound
\begin{equation}
\label{eq:T21}
\PE_{\xi}^{1/p}[\bigl| u^{\top} T_{21}\bigr|^p] \lesssim \qcond^{3/2} \bConst{A} \supconsteps \taumix^{1/2} p^{3/2} \sqrt{m} \bigg( \sum_{j=1}^{k} \alpha_{j}^2 \bigl(\sum_{\ell=j+m+1}^{k}\alpha_{\ell}^2\bigr) \prod_{\ell=j+1}^{k}(1-\alpha_{\ell} a)\biggr)^{1/2}\eqsp.
\end{equation}
In order to bound the term $T_{22}$ we first note that 
\begin{multline}
\PE_{\xi}^{1/p}[\bigl| u^{\top} S_{(j+1)m + i + 1:k} G_{jm+i+1:(j+1)m+i} (\funcnoise{\State_{jm+i}} - \funcnoise{\State_{jm+i}^{*}})\bigr|^p] \\
\leq \PE_{\xi}^{1/p}\bigl[\norm{G_{jm+i+1:(j+1)m+i}\{\funcnoise{\State_{jm+i}} - \funcnoise{\State_{jm+i}^{*}}\}}^{p} \sup_{u,v \in \sphere_{d-1},\xi'} \PE_{\xi'}[\bigl| u^{\top} S_{(j+1)m + i + 1:k} v \bigr|^p]\bigr]\eqsp.
\end{multline}
Hence, using Minkowski's inequality for $T_{22}$, we obtain that 
\begin{align}
\PE_{\xi}^{1/p}[\bigl| u^{\top} T_{22}\bigr|^p] &\leq \sqrt{\qcond}\sum_{j=0}^{N-2} \sum_{i=1}^{m} \alpha_{jm + i} \sup_{u,v \in \sphere_{d-1},\xi'} \PE_{\xi'}^{1/p}[\bigl| u^{\top} S_{(j+1)m + i + 1:k}  v \bigr|^{p}] \PE_{\xi}^{1/p}[\norm{\funcnoise{\State_{jm+i}} - \funcnoise{\State_{jm+i}^{*}}}^{p}] \\
&\qquad\qquad\qquad \times \prod_{\ell=jm+i+1}^{(j+1)m+ i}(\sqrt{1-\alpha_{\ell}a}).
\end{align}
Using the definition of $\State_{km+i}^*$ and the Cauchy-Schwartz inequality,
\begin{equation}
\label{eq: zminuszstar}
\begin{split}
\PE_{\xi}^{1/p}[\norm{\funcnoise{\State_{km+i}} - \funcnoise{\State_{km+i}^{*}}}^{p}] 
&=  \PE_{\xi}^{1/p}[\norm{\funcnoise{\State_{km+i}} - \funcnoise{\State_{km+i}^{*}}) \mathbbm{1} \{\State_{km+i}^{*} \neq \State_{km+i}  \}}^{p}] \\
& \overset{(a)}{\lesssim} \supconsteps  (1/4)^{m/(2p \taumix)}\eqsp.
\end{split}
\end{equation}
where in (a) we used \eqref{eq:zstar}. The last two inequalities together with \Cref{lem:S_j_k_def_matrices} imply 
\begin{equation}
\label{eq:T22}
\PE_{\xi}^{1/p}[\bigl| u^{\top} T_{22}\bigr|^p] \lesssim  \qcond^{3/2} \bConst{A} \supconsteps \taumix^{1/2} p^{1/2} (1/4)^{m/(2p \taumix)} \sum_{j=1}^{k} \alpha_{j} \bigl(\sum_{\ell=j+1}^{k} \alpha_{\ell}^2 \bigr)^{1/2} \prod_{\ell=j+1}^{k} \sqrt{1 -  a \alpha_\ell}\eqsp.
\end{equation}
Combining now \eqref{eq:T21} and \eqref{eq:T22}, we obtain 
\begin{align}
\PE^{1/p}_{\xi}\bigl[\bigl|u^{\top}\Jnalpha{k}{1}\bigr|^{p}\bigr] 
&\lesssim  \Const{1} \sqrt{m p} \taumix^{1/2} \alpha_{k} \\
&+ \qcond^{3/2} \bConst{A} \supconsteps \taumix^{1/2} p^{3/2} \sqrt{m} \bigg( \sum_{j=1}^{k} \alpha_{j}^2 \bigl(\sum_{\ell=j+1}^{k}\alpha_{\ell}^2\bigr) \prod_{\ell=j+1}^{k}\sqrt{1-\alpha_{\ell} a} \biggr)^{1/2} \\
&+ \qcond^{3/2} \bConst{A} \supconsteps \taumix^{1/2} p^{1/2} (1/4)^{m/(2p \taumix)} \sum_{j=1}^{k} \alpha_{j} \bigl(\sum_{\ell=j+1+m}^{k}\alpha_{\ell}^2\bigr)^{1/2} \prod_{\ell=j+1}^{k} \sqrt{1 -  a \alpha_\ell}\eqsp.
\end{align}
Applying now \Cref{lem:sum_alpha_k_squared_new}, we arrive at the bound:
\begin{align}
\PE^{1/p}_{\xi}\bigl[\bigl|u^{\top}\Jnalpha{k}{1}\bigr|^{p}\bigr] 
&\lesssim \Const{1} \sqrt{m p} \taumix^{1/2} \alpha_{k} \\
&+ (\qcond^{3/2} \bConst{A} \supconsteps / a) \taumix^{1/2} p^{3/2} \sqrt{m}  3^{\gamma}\biggl(\frac{8\gamma}{ac_0}\biggr)^{\gamma/(1-\gamma)} \alpha_{k} \\
& (\qcond^{3/2} \bConst{A} \supconsteps / a^{3/2}) \taumix^{1/2} p^{1/2} (1/4)^{m/(4p \taumix)} 3^{\gamma/2} \sqrt{\alpha_{k}}\eqsp.
\end{align}
Hence, it remains to set the block size $m$ as 
\begin{equation}
\label{eq:block_size_m_def}
(1/4)^{m/(4p \taumix)} \leq \sqrt{\alpha_{k}}, \text{ i.e. } m = \bigg \lceil \frac{2 p \taumix \log(1/\alpha_{k})}{\log{4}} \bigg \rceil\eqsp.
\end{equation}
With this choice of $m$ we obtain from the above inequality that 
\begin{align}
\PE^{1/p}_{\xi}\bigl[\bigl|u^{\top}\Jnalpha{k}{1}\bigr|^{p}\bigr] &\lesssim  \Const{1} p\taumix \sqrt{\log(1/\alpha_{k})} \alpha_{k} \\
&+ (\qcond^{3/2} \bConst{A} \supconsteps / a) \taumix p^{2} \sqrt{\log(1/\alpha_{k})}3^{\gamma}\biggl(\frac{8\gamma}{ac_0}\biggr)^{\gamma/(1-\gamma)}  \alpha_{k} \\
&+(\qcond^{3/2} \bConst{A} \supconsteps / a^{3/2}) \taumix^{1/2} p^{1/2} 3^{\gamma/2}\alpha_{k} \\ &\lesssim \Const{3}p^2\taumix\sqrt{\log(1/\alpha_k)}\alpha_k
\end{align}
where we have defined 
\begin{equation}
\Const{3} = \frac{\qcond^{3/2} \bConst{A} \supconsteps}{a} + \frac{\qcond^{3/2} \bConst{A} \supconsteps}{a^{3/2}}3^{\gamma/2} +  \frac{\qcond^{3/2} \bConst{A} \supconsteps}{a} 3^{\gamma}\biggl(\frac{8\gamma}{ac_0}\biggr)^{\gamma/(1-\gamma)}\eqsp,
\end{equation}
and the bound \eqref{eq:J_n_1_bound_Markov} holds. To estimate $H_k^{(1)}$ we rewrite it as follows
\[
\Hnalpha{k}{1} = -\sum_{j=1}^{k} \alpha_j \ProdBa_{j+1:k} \zmfuncA{\State_{j}} \Jnalpha{j-1}{1}\eqsp.
\]
Using  Minkowski's inequality together with H\"older's inequlity, we get 
\[
\PE_{\xi}^{1/p}[\bigl|u^{\top} \Hnalpha{k}{1}\bigr|^p] \leq \sum_{j = 1}^{k} \bConst{A} \alpha_{j} \bigl\{ \PE_{\xi}\bigl[\normop{\ProdBa_{j+1:k}}^{2p}\bigr]\bigr\}^{1/2p} \bigl\{ \PE_{\xi}\bigl[\norm{\Jnalpha{j-1}{1}}^{2p}\bigr]\bigr\}^{1/2p} \eqsp.
\]
Applying similar argument to \eqref{eq:bound_norm_through_scalar_norm} together with 
\Cref{prop:products_of_matrices_UGE} and an elementary inequality $\rme^{-x} \leq 1 - x/2$, valid for $x \in [0,1]$, we obtain the bound \eqref{eq:J_n_1_bound_Markov}, we get
\begin{align}
\PE_{\xi}^{1/p}[\bigl|u^{\top} \Hnalpha{k}{1}\bigr|^p] 
&\lesssim d^{1/2+1/\log n} C_\Gamma \bConst{A} \ConstDM_{4} \taumix p^{2} \sum_{j = 1}^{k} \alpha_{j}^{2} \sqrt{\log{(1/\alpha_j)}}   \prod_{\ell=j+1}^{k} (1 - a \alpha_{\ell} / 24) \\
&\lesssim d^{1/2+1/\log n} C_\Gamma \bConst{A} \ConstDM_{4} \taumix p^{2} \sqrt{\log{(1/\alpha_k)}} \sum_{j = 1}^{k} \alpha_{j}^{2} \prod_{\ell=j+1}^{k} (1 - a \alpha_{\ell} / 24) \\
& \overset{(a)}{\lesssim} \ConstDM_{5} \taumix p^{2} \alpha_{k} \sqrt{\log{(1/\alpha_k)}}\eqsp,
\end{align}
where in (a) we have additionally used \Cref{lem:sum_alpha_k_squared} and used the definition of $\ConstDM_{5}$ from \eqref{eq:const_D_5_Markov}.
\end{proof}

\subsection{Technical bounds related to $\Jnalpha{k}{1}$.}
Recall that $S_{\ell+1:\ell+m}$ is defined, for $\ell,m \in \nset$, as 
\begin{equation}
\label{eq:S_ell_n_def_tech_markov}
S_{\ell+1:\ell+m} = \sum_{k = \ell+1}^{\ell+m} \alpha_{k} \funcBw_k(\State_{k}) \eqsp, \text{ with } \funcBw_k(z) =  G_{k+1:\ell+m} \zmfuncA{z} G_{\ell+1:k-1}\eqsp. 
\end{equation}

\begin{lemma}
\label{lem:S_j_k_def_matrices}
Under the assumptions of \Cref{prop:J_n_1_bound_Markov}, it holds for any vector $u \in \sphere_{d-1}$, $\ell, m \in \nset$, and any initial distribution $\xi$ on $(\Zset,\Zsigma)$, that 
\begin{equation}
\label{eq:S_j_k_bound}
\PE_{\xi}^{1/p}\bigl[\bigl|u^{\top} S_{\ell+1:\ell+m}\funcnoise{Z_{\ell}}\bigr|^{p}\bigr] \leq C_{S} \bigl(\sum_{r=\ell+1}^{\ell+m}\alpha_{r}^2\bigr)^{1/2} \prod_{k = \ell+1}^{\ell+m}\sqrt{1- a \alpha_{k}} \eqsp,
\end{equation}
where 
\[
C_{S} = 7 \qcond \bConst{A} p^{1/2} \taumix^{1/2} \supconsteps\eqsp.
\]
\end{lemma}
\begin{proof}
We first prove the auxiliary inequality for deterministic vectors $u,v \in \S_{d-1}$. Indeed, using \eqref{eq:S_ell_n_def_tech_markov}, we obtain that 
\[
u^{\top} S_{\ell+1:\ell+m} v = \sum_{k = \ell+1}^{\ell+m} h_{k}(\State_{k})\eqsp, \quad \text{where } h_{k}(z) := \alpha_{k} u^{\top} \funcBw_k(z) v\eqsp. 
\]
It is easy to check that $u^{\top} S_{\ell+1:\ell+m} v$ satisfies the bounded differences property, since for any $z,z' \in \Zset$, and $r \in \{\ell+1,\ldots,\ell+m\}$, it holds that 
\begin{equation}
\bigl| h_{r}(z) - h_{r}(z') \bigr| \leq 2 \qcond \bConst{A} \alpha_{r} \prod_{k = \ell+1, k \neq r}^{\ell+m}\sqrt{1- a \alpha_{k}} \leq 2^{3/2} \qcond \bConst{A} \alpha_{r} \prod_{k = \ell+1}^{\ell+m}\sqrt{1- a \alpha_{k}}\eqsp.
\end{equation}
In the last inequality we have additionally used the fact that $\alpha_{k} a \leq 1/2$ for any $k \in \nset$. Applying the bounded differences inequality from \cite{paulin_concentration_spectral}[Corollary 2.11], we get that for any $t \geq 0$, 

\begin{align}
\PP_{\xi}\bigl(\bigl|u^{\top} S_{j+1:j+m}v\bigr| \geq t\bigr) \leq 2\exp\biggl\{-\frac{2(t-|\PE_{\xi}[u^{\top} S_{j+1:j+m}v]|)^2}{72 \taumix \qcond^2 \bConst{A}^2 \bigl(\sum_{r=\ell+1}^{\ell+m}\alpha_{r}^2\bigr) \prod_{k = \ell+1}^{\ell+m}(1- a \alpha_{k})}\biggr\}\eqsp.
\end{align}

It remains to upper bound $\PE_{\xi}[u^{\top} S_{j+1:j+m}v]$. Note that 
\begin{align}
   (\PE_{\xi}[u^{\top} S_{j+1:j+m}v])^2 \leq  \PE_{\xi}[(\sum_{k=l+1}^{l+m} h_k(Z_k))^2] = \sum_{k=l+1}^{l+m}\PE_{\xi}[h_k(Z_k)^2] + 2\sum_{k=l+1}^{l+m}\sum_{j=1}^{l+m-k}\PE_{\xi}[h_k(Z_k)h_{k+j}(Z_{k+j})]\eqsp.
\end{align}
Using that  $\pi(h_k)=0$ and \Cref{assum:UGE}, we obtain
\begin{align}
    |\PE_{\xi}[h_k(Z_k)h_{k+j}(Z_{k+j})]|& = |\int_{\Zset}h_k(z)(\MKQ^jh_{k+j}(z) - \pi(h_{k+j}))\xi\MKQ^k(\rmd z)|\leq \|h_k\|_{\infty}\|h_{k+j}\|_{\infty}\dobrush(\MKQ^j)\\ &\leq 2\alpha_k\alpha_{k+j}\qcond^2(2\bConst{A})^2\prod_{t=l+1}^{l+m}(1-a\alpha_t)(1/4)^{\lceil j/\taumix\rceil}\eqsp,
\end{align}
and 
\begin{align}
    \PE_{\xi}[h_k(Z_k)^2] \leq 2\alpha_k^2\qcond^2(2\bConst{A})^2\prod_{t=l+1}^{l+m}(1-a\alpha_t)\eqsp.
\end{align}
Combining inequalities above, we obtain 
\begin{align}
    (\PE_{\xi}[u^{\top} S_{j+1:j+m}v])^2 \leq \frac{32}{3}\taumix\qcond^2\bConst{A}^2(\sum_{k=l+1}^{l+m}\alpha_k^2)\prod_{t=l+1}^{l+m}(1-a\alpha_t)\eqsp.
\end{align}
Note that for $t\geq |\PE_{\xi}[u^{\top} S_{j+1:j+m}v]|$, we have 
\begin{align}
\PP_{\xi}\bigl(\bigl|u^{\top} S_{j+1:j+m}v\bigr| \geq t\bigr) \leq 2\exp\biggl\{-\frac{t^2}{49 \taumix \qcond^2 \bConst{A}^2 \bigl(\sum_{r=\ell+1}^{\ell+m}\alpha_{r}^2\bigr) \prod_{k = \ell+1}^{\ell+m}(1- a \alpha_{k})}\biggr\}\eqsp.
\end{align}
And for $t\leq |\PE_{\xi}[u^{\top} S_{j+1:j+m+m}v]|$, the right side of the inequality is greater than 1, so this inequality is also true for $t\leq |\PE_{\xi}[u^{\top} S_{j+1:j+m+m}v]|$.
Hence, using \Cref{lem:bound_subgaussian}, we obtain that 
\begin{align}
\label{eq:aux_bound_moment_s_ell}
\PE_{\xi}^{1/p}\bigl[\bigl|u^{\top} S_{j+1:j+m}v\bigr|^p] \leq 7 p^{1/2}\qcond \bConst{A} \taumix^{1/2} \bigl(\sum_{r=\ell+1}^{\ell+m}\alpha_{r}^2\bigr)^{1/2} \prod_{k = \ell+1}^{\ell+m}\sqrt{1- a \alpha_{k}}\eqsp.
\end{align}
Now, with $\mcf_\ell = \sigma\{\State_{j}, j \leq \ell\}$, it holds that
\begin{align}
\PE^{1/p}_{\xi}[\bigl|u^{\top} S_{\ell+1:\ell+m} \funnoisew(\State_{\ell})\bigr|^{p}]
&= \PE_{\xi}^{1/p}\bigl[\norm{\funnoisew(\State_{\ell})}^{p} \CPE{ \bigl|u^{\top} S_{\ell+1:\ell+m} \funnoisew(\State_{\ell})\bigr|^{p}/\norm{\funnoisew(\State_{\ell})}^p}{\mcf_\ell}\bigr] \\
&\leq \supconsteps \sup_{u,v \in \sphere_{d-1}, \,  \xi^{\prime} \in \mathcal{P}(\msz)}\PE^{1/p}_{\xi^{\prime}}\bigl[\bigl| u^{\top} S_{\ell+1:\ell+m} v\bigr|^{p}\bigr]\eqsp.
\end{align}
Combining the above bounds with \eqref{eq:aux_bound_moment_s_ell} and \Cref{assum:noise-level} yields the statement.

\end{proof}

\section{Proofs for stability of matrix products}
\label{appendix:tehnical}

\subsection{Stability}
We first provide a  result on the product of  dependent  random matrices. Our proof technique is based on the approach of  \cite{huang2020matrix} and the results, previously obtained in \cite{durmus2022finite}. Let $(\Omega, \mathfrak F, \sequence{\mathfrak{F}}[\ell][\nset], \P)$ be a filtered probability space. For the matrix $\MatB \in \rset^{d \times d}$ we denote by $( \sigma_\ell(\MatB) )_{ \ell=1 }^d$ its singular values. For $\qexponent \geq 1$, the Shatten $\qexponent$-norm is denoted by $\norm{\MatB}[\qexponent] = \{\sum_{\ell=1}^d \sigma_\ell^\qexponent (\MatB)\}^{1/\qexponent}$. For $\qexponent, \ppexponent \geq 1$ and a random matrix $\X$ we write $\norm{\X}[\qexponent,\ppexponent] = \{ \PE[\norm{\X}[\qexponent]^\ppexponent] \}^{1/\ppexponent}$. The main result of this section is stated below:
\begin{proposition}
\label{prop:products_of_matrices_UGE}
Assume \Cref{assum:UGE}, \Cref{assum:noise-level}, and \Cref{assum:step-size}. Then, for any $2 \leq p \leq \log{n}$, $n \in \nset$, and probability distribution $\xi$ on $(\Zset,\Zsigma)$, it holds that 
\begin{equation}
\label{eq:concentration UGE}
\PE_{\xi}^{1/p}\left[ \normop{\ProdBa_{j:n}}^{p} \right]
\leq  C_\Gamma d^{1/{\log{n}}} \exp\biggl\{ - (a / 12) \sum_{k=j}^{n} \alpha_{k} \biggr\}\eqsp,
\end{equation}
where 
\[
C_\Gamma = \sqrt{\qcond} \rme^2\eqsp.
\]
\end{proposition}
The proof is given in \Cref{sec:matrix_product_matrix}.
It is based on a simplification of the arguments in \cite{durmus2021stability} together with a new result about the matrix concentration for the product of random matrices, using a proof method introduced in \cite{huang2020matrix}. We first state a result from \cite{durmus2022finite}:
\begin{proposition}[Proposition 15 in \cite{durmus2022finite}]
\label{th:general_expectation UGE}
Let $\sequence{\Y}[\ell][\nset]$ be a sequence of random matrices adapted to the filtration \(\sequence{\mathfrak{F}}[\ell][\nset]\)   and $\MatP$ be a positive definite matrix. Assume that for each $\ell \in \nsets$ there exist $\mtt_{\ell} \in \ocintLine{0,1}$  and $\sigma_{\ell} > 0$ such that
\begin{equation}
  \label{eq:form_bound_ass_prod_mat}
  \norm{\PE^{\mathfrak{F}_{\ell-1}}[\Y_\ell]}[\MatP]^2  \leq 1 - \mtt_{\ell} \text{  and } \norm{\Y_\ell - \PE^{\mathfrak{F}_{\ell-1}}[\Y_\ell]}[\MatP]  \leq \sigma_\ell \quad \text{ $\PP$-a.s.}  \eqsp.
\end{equation}
Define $\Zbf_n = \prod_{\ell = 0}^n \Y_\ell= \Y_n \Zbf_{n-1}$, for $n \geq 1$. Then, for any $2 \le \ppexponent \le \qexponent$ and $n \geq 1$,
\begin{equation}
\label{eq:gen_expectation}
\textstyle \norm{\Zbf_n}[\qexponent,\ppexponent]^2 \leq \kappa_P \prod_{\ell=1}^n (1- \mtt_{\ell} + (\qexponent-1)\sigma_{\ell}^2) \norm{\MatP^{1/2}\Zbf_0 \MatP^{-1/2}}[\qexponent, \ppexponent]^2 \eqsp,
\end{equation}
where $\kappa_P = \lambda_{\sf max}( \MatP )/\lambda_{\sf min}( \MatP )$ and $\lambda_{\sf max}( \MatP ),\lambda_{\sf min}( \MatP )$ correspond to the largest and smallest eigenvalues of $\MatP$.
\end{proposition}

Now we fix some $N \in \nset$, block size $h \in \nset$ defined in \eqref{eq:h_block_size_def}, and a sequence $0 = j_{0} < j_{1} < \ldots < j_{N} = 2n$,  where $j_{\ell} = h \ell$, $\ell \leq N-1$. Note that it is possible that $j_{N} - j_{n-1} < h$.
Now we set
\[
\Y_{\ell} = \prod_{i=j_{\ell-1}}^{j_{\ell}-1} (\Id - \alpha_{i}
\funcA{\State_{i}})\eqsp.
\]
Then the following lemma holds:
\begin{lemma}
\label{lem:proof_uge_1}
Assume \Cref{assum:UGE}, \Cref{assum:noise-level}, and \Cref{assum:step-size}. Then for any $\ell \in \{1,\ldots,N-1\}$, and any probability measure $\xi$ on $(\Zset,\Zsigma)$, it holds that 
\begin{equation}
\label{eq:h_markov_def}
\norm{\PE_{\xi}[\Y_{\ell}]}[Q] \leq 1 -  \bigl(\sum_{k=j_{\ell-1}}^{j_{\ell}-1} \alpha_{k}\bigr) a / 6 \eqsp.
\end{equation}
\end{lemma}
\begin{proof}
We decompose the matrix product $\Y_{\ell}$ as follows:
\begin{equation} \label{eq:split_main}
\textstyle
\Y_{\ell}  = \Id - \bigl(\sum_{k=j_{\ell-1}}^{j_{\ell}-1} \alpha_{k} \bigr) \bA   - \Mat{S}_{\ell}  + \Mat{R}_{\ell} \eqsp.
\end{equation}
Here $\Mat{S}_\ell = \sum_{k = j_{\ell-1}}^{j_{\ell}-1} \alpha_{k} \bigl\{\funcA{\State_{k}} - \bA\bigr\}$ is a linear statistics in $\{\funcA{\State_{k}}\}_{k=j_{\ell-1}}^{j_{\ell}-1}$, and the remainder term $\Mat{R}_{\ell}$ is defined as
\begin{equation}
\label{eq:RlRlbar_def}
\textstyle
    \Mat{R}_{\ell} = \sum_{r=2}^{h}(-1)^{r}  \sum_{(i_1,\dots,i_r)\in\msi_r^\ell}\prod_{u=1}^{r} \{\alpha_{i_u}\funcA{\State_{i_u}}\}\eqsp,
\end{equation}
where $\msi_r^{\ell} = \{(i_1,\ldots,i_r) \in \{1,\ldots,h\}^r\, : \, i_1 < \cdots < i_r \}$. Since $\norm{M}[Q] = \norm{Q^{1/2} M Q^{-1/2}}$, it is straightforward to check that $\PP$-a.s. it holds
\begin{equation}
\label{eq:10}
\textstyle
\norm{\Mat{R}_{\ell}}[Q] \leq \frac{(\sum_{k=j_{\ell-1}}^{j_{\ell}-1} \alpha_{k})^2 \qcond \bConst{A}^2}{2} \exp\bigl\{ \qcond^{1/2} \bConst{A} \sum_{k=j_{\ell-1}}^{j_{\ell}-1} \alpha_{k}\bigr\} =: T_2 \eqsp.
\end{equation}
On the other hand, using \Cref{assum:UGE} and \Cref{assum:noise-level}, we have for any $k \in \nset$, that
\begin{align}
\normop{\PE_{\xi}[\funcA{\State_{k}} - \bA]} = \sup_{u,v \in \sphere^{d-1}}[\PE_{\xi}[u^{\top} \funcA{\State_{k}} v] - u^{\top} \bA v] \leq \bConst{A} \dobru{\MKQ^k}\eqsp.
\end{align}
Hence, with the triangle inequality we obtain that 
\begin{align}
\norm{\PE_{\xi}[\Mat{S}_{\ell}]}[Q] 
&\leq \qcond^{1/2} \sum_{k = j_{\ell-1}}^{j_{\ell}-1} \alpha_{k} \normop{\PE_{\xi}[\funcA{\State_{k}} - \bA]} \leq \qcond^{1/2} \bConst{A} \sum_{k = j_{\ell-1}}^{j_{\ell}-1} \alpha_{k} \dobru{\MKQ^k} \\
&\leq (4/3) \alpha_{j_{\ell-1}}  \taumix \qcond^{1/2} \bConst{A} =: T_1\eqsp.
\end{align}
This result combined with \eqref{eq:10} in \eqref{eq:split_main} implies that
\begin{equation}
\label{eq: T_2 def markov product}
\norm{\PE_{\xi}[ \Y_{\ell}]}[Q] \leq \norm{\Id - \sum_{k=j_{\ell-1}}^{j_{\ell}-1} \alpha_k \bA}[Q] + T_1 + T_2 \eqsp.
\end{equation}
First, by definition  \eqref{eq:h_block_size_def} of $h$ (see details in \Cref{lem:technical_sum_check}), we have
\begin{equation}
\label{eq:bound_T_1_product_uge}
T_1 \leq \bigl(\sum_{k=j_{\ell-1}}^{j_{\ell}-1} \alpha_{k}\bigr) a/6\eqsp.
\end{equation}
Second, with the definition of $c_{0}$ in \eqref{eq:c_0_bound_optimized}, we obtain that 
\begin{equation}
\label{eq:bound_T_2_product_uge}
T_2 \leq (\qcond^{1/2} \bConst{A} \sum_{k=j_{\ell-1}}^{j_{\ell}-1} \alpha_{k})^2 \rme \leq \bigl(\sum_{k=j_{\ell-1}}^{j_{\ell}-1} \alpha_{k}\bigr) a/6 \eqsp.
\end{equation}
Finally, \Cref{prop:hurwitz_stability} implies that, for $\sum_{k=j_{\ell-1}}^{j_{\ell}-1} \alpha_{k} \leq \alpha_{\infty}$, it holds that 
\begin{equation}
\label{eq:bound_main_uge}
\textstyle \norm{\Id -  \bigl(\sum_{k=j_{\ell-1}}^{j_{\ell}-1} \alpha_{k} \bigr) \bA}[Q] \leq 1 - (a/2) \sum_{k=j_{\ell-1}}^{j_{\ell}-1} \alpha_{k}\eqsp.
\end{equation}
Combining \eqref{eq:bound_T_1_product_uge}, \eqref{eq:bound_T_2_product_uge}, and \eqref{eq:bound_main_uge} yield that
\[
\norm{\PE_{\xi}[ \Y_{1}]}[Q] \leq 1 - \bigl(\sum_{k=j_{\ell-1}}^{j_{\ell}-1} \alpha_{k} \bigr) a /6\eqsp,
\]
and the statement follows.
\end{proof}

\begin{lemma}
\label{lem:technical_sum_check}
Assume \Cref{assum:step-size}. Then, for the block size $h$ defined in \eqref{eq:h_block_size_def}, and any $\ell \in \{1,\ldots,2n-h\}$, it holds that 
\begin{align}
\label{eq:technical-bounds-first-line}
&\sum_{k=\ell}^{\ell+h} \alpha_{k} \leq \alpha_{\infty}\eqsp, \quad \qcond^{1/2} \bConst{A} \sum_{k=\ell}^{\ell+h} \alpha_{k}  \leq 1\eqsp, \quad \qcond \bConst{A}^2 \rme \sum_{k=\ell}^{\ell+h} \alpha_{k} \leq a/6\eqsp, \\
\label{eq:technical-bounds-second-line}
&(4/3) \alpha_{\ell} \taumix \qcond^{1/2} \bConst{A} \leq \bigl(\sum_{k=\ell}^{\ell+h} \alpha_{k}\bigr) a / 6\eqsp, \\
& (\log{n}) \bConst{\sigma}^2 \sum_{k=\ell}^{\ell+h} \alpha_{k} \leq \frac{a}{12}\eqsp.
\label{eq:technical-bounds-third-line}
\end{align}
\end{lemma}
\begin{proof}
First three inequalities above are easy to check, since 
\begin{align}
\sum_{k=\ell}^{\ell+h} \alpha_{k} \leq  \sum_{k=1}^{1+h}\frac{c_{0}}{{(k+k_0)}^{\gamma}} &\leq c_{0} \int_{k_0}^{1+k_0+h}\frac{dx}{x^{\gamma}} = \frac{c_{0}\{(1+h+k_0)^{1-\gamma} - k_0^{1-\gamma}\}}{1-\gamma} \\&= \frac{c_{0}k_0^{1-\gamma}((1+\frac{h+1}{k_0})^{1-\gamma}-1)}{1-\gamma}\leq c_0k_0^{-\gamma}(h+1)\eqsp,
\end{align}
hence, in order to satisfy these inequalities, it is enough to choose $h$ in such a manner that 
\begin{align}
k_0^\gamma\geq c_0(h+1) \max(\alpha_{\infty}, \qcond^{1/2}\bConst{A}, 6\rme\qcond\bConst{A}^2/a)\eqsp,
\end{align}
which is guaranteed by our choice of $c_0$ in \Cref{assum:step-size} and \eqref{eq:h_block_size_def}. Now note that 
\begin{align}
\sum_{k=\ell}^{\ell+h} \alpha_{k} \geq c_{0} \int_{\ell+k_0}^{\ell+k_0+h+1}\frac{\rmd x}{x^{\gamma}} = \frac{c_{0}\{(\ell+k_0+h+1)^{1-\gamma} - (\ell+k_0)^{1-\gamma}\}}{1-\gamma}\eqsp,
\end{align}
hence, in order to check \eqref{eq:technical-bounds-second-line}, we need to ensure that
\begin{align}
\frac{4 c_0 \taumix \qcond^{1/2} \bConst{A}}{3 (\ell+k_0)^{\gamma}} \leq \frac{c_{0} a \{(\ell+k_0+h+1)^{1-\gamma} - (\ell+k_0)^{1-\gamma}\}}{6(1-\gamma)}\eqsp. 
\end{align}
Equivalently, it is enough to set $h$ in such a way, that 
\begin{align}
(\ell+k_0+h+1)^{1-\gamma} - (\ell+k_0)^{1-\gamma} \geq \frac{C_{1}}{(\ell+k_0)^{\gamma}}\eqsp,
\end{align}
where we set $C_{1} = 8(1-\gamma)\taumix \qcond^{1/2} \bConst{A} / a$. Hence, \eqref{eq:technical-bounds-second-line} will be satisfied if 
\begin{align}
h \geq (\ell+k_0) \biggl( 1 + \frac{C_{1}}{\ell+k_0}\biggr)^{1/(1-\gamma)} - (\ell+k_0+1)\eqsp.
\end{align}
Note that, with $\alpha > 1$, it holds that $(1+x)^{\alpha} \leq 1 + 2\alpha x$ for $0 < x \leq 1/\alpha$. Hence, provided that $k_0>C_1/(1-\gamma)$, or, equivalently,
\[
k_0 \geq \frac{8 \taumix \qcond^{1/2} \bConst{A}}{a} \eqsp,
\]
it is enough to set  
\[
h \geq (\ell+k_0) \bigl(1 + \frac{2C_1}{(\ell+k_0)(1-\gamma)}\bigr) - (\ell+k_0+1) = \frac{16 \taumix \qcond^{1/2} \bConst{A}}{a}-1\eqsp.
\]
Now it remains to check \eqref{eq:technical-bounds-third-line}, that is,
\[
(\log{n}) \bConst{\sigma}^2 \sum_{k=\ell}^{\ell+h} \frac{c_0}{k^{\gamma}}\leq \frac{a}{12}\eqsp.
\]
Since 
\[
\sum_{k=\ell}^{\ell+h} \frac{1}{(k+k_0)^{\gamma}} \leq \sum_{k=1}^{h+1} \frac{1}{(k+k_0)^{\gamma}} \leq \int_{k_0}^{h+1+k_0}\frac{\rmd x}{x^{\gamma}} = \frac{(h+k_0+1)^{1-\gamma} - (k_0)^{1-\gamma}}{1-\gamma}\eqsp,
\]
it is enough to choose $k_0$ in such a way that 
\[
(\log{n}) \bConst{\sigma}^2 c_{0} \frac{(h+k_0+1)^{1-\gamma} - (k_0)^{1-\gamma}}{1-\gamma} \leq \frac{a}{12} \eqsp,
\]
or, equivalently, 
\begin{equation}
\label{eq:technical-bounds-second-line-1}
(k_0)^{1-\gamma}((1+\frac{h+1}{k_0})^{1-\gamma} - 1) \leq \frac{a(1-\gamma)}{12(\log n )\bConst{\sigma}^2 c_{0}}\eqsp.
\end{equation}
Since $(1+x)^{1-\gamma} \leq 1 + (1-\gamma)x$ for $ x \geq -1$, \eqref{eq:technical-bounds-second-line-1} holds if
\[
k_0^{-\gamma}(1-\gamma)(h+1) \leq \frac{a(1-\gamma)}{12(\log n )\bConst{\sigma}^2 c_{0}}\eqsp,
\]
or, equivalently, 
\begin{align}
k_0 \geq \biggl(\frac{12(h+1)(\log n) c_0\bConst{\sigma}^2}{a}\biggr)^{1/\gamma}
\end{align}
which is guaranteed by the condition \eqref{eq:a3-sample-size}.
\end{proof}
\begin{lemma}
\label{lem:proof_uge_2}
Assume \Cref{assum:UGE}, \Cref{assum:noise-level}, and \Cref{assum:step-size}. Then, for any probability $\xi$ on $(\Zset,\Zsigma)$, and any $\ell \in \{1,\ldots,N-1\}$, we have
\begin{equation}
\label{eq:norm_bound_h_lemma}
\norm{\Y_{\ell} - \PE_{\xi}[ \Y_{\ell}]}[Q] \leq \bConst{\sigma} \bigl(\sum_{k=j_{\ell-1}}^{j_{\ell}-1} \alpha_{k}\bigr) \eqsp, \text{ where } \bConst{\sigma} = 2(\qcond^{1/2}\bConst{A} + a/6)\eqsp,
\end{equation}
and $h$ is given in \eqref{eq:h_block_size_def}.
\end{lemma}
\begin{proof}
Using the decomposition \eqref{eq:split_main}, we obtain
\begin{equation}
\textstyle
\norm{\Y_{\ell} - \PE_{\xi}[ \Y_{\ell}]}[Q] \leq \sum_{k = j_{\ell-1}}^{j_{\ell}-1} \alpha_{k} \norm{\funcA{\State_k} - \PE_{\xi}[\funcA{\State_k}]}[Q] + \norm{\Mat{R}_{\ell} - \PE_{\xi}[\Mat{R}_{\ell}]}[Q]\eqsp.
\end{equation}
Applying the definition of $\Mat{R}_{\ell}$ in \eqref{eq:10}, the definition of $h$,$\alpha_\infty^{(\Markov)}$, and $T_2$ in \eqref{eq:bound_T_2_product_uge}, we get from the above inequalities that 
\begin{equation}
\textstyle \norm{\Y_{\ell} - \PE_{\xi}[ \Y_{\ell}]}[Q] \leq 2 \qcond^{1/2}\bConst{A} \bigl(\sum_{k=j_{\ell-1}}^{j_{\ell}-1} \alpha_{k}\bigr) +  (a/3)\bigl(\sum_{k=j_{\ell-1}}^{j_{\ell}-1} \alpha_{k}\bigr)\eqsp,
\end{equation}
and the statement follows.
\end{proof}

We have now all ingredients required to prove \Cref{prop:products_of_matrices_UGE}.
\subsection{Proof of \Cref{prop:products_of_matrices_UGE}}
\label{sec:matrix_product_matrix}
Denote by $h \in \nset$ a block length, the value of which is determined later. Define the sequence $j_0 = j, \, j_{\ell+1} = \min(j_\ell + h, n)$. By construction $j_{\ell+1} - j_{\ell} \leq h$. Let $N = \lceil(n-j)/h\rceil$. Now we introduce the decomposition
\begin{equation}
\label{eq:decomp_Gamma_proof_main}
\ProdBa_{j:n} = \prod_{\ell=1}^N \Y_\ell\eqsp, \quad \text{where} \quad \Y_\ell = 
\begin{cases}
   \prod_{i=j_{\ell-1}}^{j_\ell-1} (\Id - \alpha
\funcA{\State_{i}}) \eqsp, ~~\ell \in \{1,\ldots,N-1\} \eqsp, \\
\prod_{i=j_{N-1}}^{n} (\Id - \alpha
\funcA{\State_{i}}) \eqsp, ~~\ell=N\eqsp. \\
\end{cases}
\end{equation}
The last block, $Y_{N}$, can be of smaller size than $h$. Now we apply the bound 
\[
\norm{\Y_{N}} \leq \prod_{k=j_{N-1}}^{n}(1 + \alpha_{k} \bConst{A}) \leq \exp\bigl\{ \bConst{A} \sum_{k=j_{N-1}}^{N} \alpha_{k}  \bigr\} \leq \rme\eqsp,
\]
where the last bound follows from the relation \eqref{eq:c_0_bound_optimized}.  Hence, substituting into \eqref{eq:decomp_Gamma_proof_main}, we get the following bound:
\begin{equation}
\textstyle \PE^{1/p}_{\xi}[\normop{\ProdBa_{j:n}}^{p}] \leq \rme \PE_{\xi}^{1/p}[ \normopLigne{\prod_{\ell=1}^{N-1} \Y_{\ell} }^{p}]\eqsp.
\end{equation}
Now we bound $\PE_{\xi}^{1/p}[\normopLigne{\prod_{\ell=1}^{N-1} \Y_{\ell} }^{p}]$ using the results from \Cref{th:general_expectation UGE}. To do so, we define, for $\ell \in \{1,\ldots,N-1\}$, the filtration $\mathcal{H}_{\ell} = \sigma(Z_k \,: \, k \leq j_{\ell})$ and establish  almost sure bounds on $\norm{\CPE[\xi]{\Y_\ell}{\mathcal{H}_{\ell-1}}}[Q]$ and $\norm{\Y_{\ell}-\CPE[\xi]{\Y_\ell}{\mathcal{H}_{\ell-1}}}[Q]$ for $\ell \in \{1,\ldots,N-1\}$. More precisely, by the Markov property, it is sufficient to show that there exist $\mtt \in \ocintLine{0,1}$ and $\sigma > 0$ such that for any probabilities $\xi, \xi'$ on $(\Zset,\Zsigma)$,
\begin{equation}
\label{eq:form_bound_ass_prod_mat_v2_Z}
\norm{\PE_{\xi'}[\Y_{\ell}]}[Q]^2  \leq 1 - \mtt_{\ell} \text{  and } \norm{\Y_{\ell} - \PE_{\xi'}[\Y_{\ell}]}[Q] \leq \sigma_{\ell} \eqsp, \quad \text{ $\PP_{\xi}$-a.s.}\eqsp.
\end{equation}
Such bounds require the blocking procedure, since \eqref{eq:form_bound_ass_prod_mat_v2_Z} not necessarily holds with $h = 1$. Setting $h$ as in equation \eqref{eq:h_block_size_def}, and applying \Cref{lem:proof_uge_1} and \Cref{lem:proof_uge_2}, we show that \eqref{eq:form_bound_ass_prod_mat_v2_Z} hold with 
\[
\mtt_{\ell} = \bigl(\sum_{k=j_{\ell-1}}^{j_{\ell}-1} \alpha_{k}\bigr) a / 6 \eqsp, \quad \sigma_{\ell} = \bConst{\sigma} \bigl(\sum_{k=j_{\ell-1}}^{j_{\ell}-1} \alpha_{k}\bigr)\eqsp,
\]
where $\bConst{\sigma}$ is given in \eqref{eq:norm_bound_h_lemma}. Then, applying \Cref{th:general_expectation UGE} with $q = \log{n}$, we get
\begin{align}
\PE_{\xi}^{1/p}\left[ \normop{\ProdBa_{1:n}}^{p} \right]
&\leq \PE_{\xi}^{1/q}\left[ \normop{\ProdBa_{1:n}}^{q}\right] \\
&\leq \sqrt{\qcond} d^{1/q} \rme \prod_{\ell=1}^{N-1}\biggl( 1 - \bigl(\sum_{k=j_{\ell-1}}^{j_{\ell}-1} \alpha_{k}\bigr) a / 6 + (\log{n}) \bConst{\sigma}^2 \bigl(\sum_{k=j_{\ell-1}}^{j_{\ell}-1} \alpha_{k}\bigr)^2 \biggr)  \nonumber \\
& \leq \sqrt{\qcond} d^{1/q} \rme \prod_{\ell=1}^{N-1} \exp\biggl\{ - \bigl(\sum_{k=j_{\ell-1}}^{j_{\ell}-1} \alpha_{k}\bigr) a / 6 + (\log{n}) \bConst{\sigma}^2 \bigl(\sum_{k=j_{\ell-1}}^{j_{\ell}-1} \alpha_{k}\bigr)^2 \biggr\} \nonumber \\
& \overset{(a)}{\leq} \sqrt{\qcond} \rme^2 d^{1/q} \exp\biggl\{ - \bigl(\sum_{k=j}^{n} \alpha_{k}\bigr) a / 12 \biggr\} \eqsp.\nonumber
\end{align}
Here in (a) we used the fact that 
\[
(\log{n}) \bConst{\sigma}^2 \bigl(\sum_{k=j_{\ell-1}}^{j_{\ell}-1} \alpha_{k}\bigr) \leq \frac{a}{12}\eqsp,
\]
which holds due to \Cref{assum:step-size}.
\endproof

\subsection{Proof of \Cref{prop:hurwitz_stability}}
First part of the statement (existence of $Q$) follows from \cite[Proposition 1]{samsonov2024gaussian}. For the second part, we note that for any non-zero vector $x \in \rset^{d}$, we have
\begin{align}
\frac{x^{\top}(\Id - \alpha \bA)^{\top}Q(\Id - \alpha \bA)x}{x^{\top} Q x} 
&= 1 - \alpha \frac{x^{\top}(\bA^{\top}Q + Q\bA)x}{x^{\top} Q x} + \alpha^2 \frac{x^{\top} \bA^{\top} Q \bA x}{x^{\top} Q x} \\
&= 1 - \alpha \frac{x^{\top} P x}{x^{\top} Q x} + \alpha^2\, \frac{x^{\top} \bA^{\top} Q \bA x}{x^{\top} Q x} \\
&\leq 1 - \alpha \frac{\lambda_{\min}(P)}{\normop{Q}} + \alpha^2\, \norm{\bA}[Q]^2 \\
&\leq 1 - \alpha a\eqsp,
\end{align}
where we set 
\[
a = \frac{1}{2} \frac{\lambda_{\min}(P)}{\normop{Q}}\eqsp,
\]
and used the fact that $\alpha \leq \alpha_{\infty}$, where $\alpha_{\infty}$ is defined in \eqref{eq:alpha_infty_def}.

\section{Proofs of \Cref{sec:bootstrap}}
\label{appendix:bootstrap}
\subsection{Proof of \Cref{prop:block_bootstrap_batch_mean}}
\label{sec:block_bootstrap_batch_mean_proof} 
Note that for any $0 \le t \le n - b_n$
\begin{align}
\theta_{k+t} - \thetas = \ProdB_{t+1:t+k}(\theta_t - \thetas) - \sum_{\ell = t+1}^{k+t}\alpha_{\ell}\ProdB_{\ell+1:t+k} \funcnoise{\State_{\ell}}
\end{align}
We first note that
\begin{align}
\bar{\theta}_{n, b_n}(u) &= \frac{\sqrt{b_n}}{\sqrt{n-b_n+1}} \sum_{t = 0}^{n-b_n} w_t (\bar{\theta}_{b_n,t} - \bar{\theta}_{n})^\top u \\
& = \frac{\sqrt{b_n}}{\sqrt{n-b_n+1}} \sum_{t = 0}^{n-b_n} w_t (\bar{W}_{b_n,t} - \bar{W}_{n})^\top u + \frac{\sqrt{b_n}}{\sqrt{n-b_n+1}} \sum_{t = 0}^{n-b_n} w_t (\bar{D}_{b_n,t} - \bar{D}_{n})^\top u\eqsp, 
\end{align}
where 
\begin{align}
  \bar{W}_{b_n,t} &= -\frac{1}{b_n}\sum_{k=1}^{b_n-1} \bA^{-1} \funcnoise{Z_{k+t}}, \\
  \bar{W}_{n} & = -\frac{1}{n} \sum_{k=1}^{n-1} \bA^{-1} \funcnoise{Z_{k}}, \\
  \bar{D}_{b_n,t} &= \frac{1}{b_n}\sum_{k=1}^{b_n-1} \ProdB_{t+1:t+k} (\theta_t - \thetas) + \frac{1}{b_n} \sum_{k=1}^{b_n-1} H_{k,t,b_n}^{(0)} - \frac{1}{b_n} \sum_{k=1}^{b_n-1} (Q_{k,t,b_n} - \bA^{-1}) \funnoisew(\State_{k+t})    \\
  \bar{D}_n &= \frac{1}{n} \sum_{k = 0}^{n-1} \ProdB_{1:k} (\theta_0 - \thetas) +\frac{1}{n} \sum_{k=1}^{n-1} H_k^{(0)} -  \frac{1}{n} 
\sum_{\ell=1}^{n-1} (Q_\ell-\bA^{-1}) \funnoisew(\State_\ell)\eqsp,
\end{align}
where
\begin{align}
    Q_{k,t,b_n} &= \alpha_{k+t} \sum_{\ell=k}^{b_n-1} G_{k+t+1:\ell+t}\eqsp, \\
    H_{k,t,b_n}^{(0)} & = -\sum_{\ell=t+1}^{k+t}\alpha_\ell \ProdB_{\ell+1:t+k}\zmfuncA{Z_\ell} J_{\ell-1, t, b_n}^{(0)}\\
    J^{(0)}_{k, t, b_n}& = -\sum_{\ell=t+1}^{k+t} \alpha_\ell G_{\ell+1:k+t} \funcnoise{\State_\ell}\eqsp.
\end{align}
Then
\begin{align}
\begin{aligned}
    \hat \sigma^2_\theta(u) &= \hat \sigma^2_\varepsilon(u) + \RemCov_{var}(u)
\end{aligned}
\end{align}
where 
\begin{equation}
\label{def:remcov_car}
    \begin{split}
        \RemCov_{var}(u) &= D_1^b + D_2^b\\
        D_1^b &= \frac{b_n}{n-b_n+1} \sum_{t = 0}^{n-b_n} ( (\bar{D}_{b_n,t} - \bar{D}_{n})^\top u)^2 \\
        D_2^b &= \frac{2b_n}{n-b_n+1} \sum_{t = 0}^{n-b_n} u^\top (\bar{W}_{b_n,t} - \bar{W}_{n}) (\bar{D}_{b_n,t} - \bar{D}_{n})^\top  u
    \end{split}
\end{equation}

\begin{lemma}
\label{lem:bound_barD_b_t}
Assume \Cref{assum:UGE}, \Cref{assum:noise-level} and \Cref{assum:step-size}. Then it holds 
\begin{align}
\begin{aligned}
\PE_{\xi}^{1/p}[| u^{\top}\bar{D}_{b_n,t} |^p]  &\lesssim \frac{\Constupdboot{1,1}p^{1/2}}{b_n\sqrt{{\alpha_t}}}+ \frac{\Constupdboot{1,2}}{b_n} + \Constupdboot{1,3}p^2\sqrt{\log (1/\alpha_{b_n+t-1})}\frac{\sum_{k=1+t}^{b_n+t-1}\alpha_k}{b_n}\\ &+\Constupdboot{1,4}\frac{ p^{1/2}}{b_n}(\sum_{k=2}^{b_n-1}(k+k_0+t)^{2\gamma-2})^{1/2}+ \frac{\Constupdboot{1,5}p^{1/2}b_n^{-1}}{\sqrt{\alpha_{t+b_n-2}}} + \Constupdboot{1,6}\frac{(k_0+t-1)^{\gamma-1}}{b_n}\eqsp,
\end{aligned}
\end{align}
where 
\begin{equation}
    \begin{split}
        \Constupdboot{1,1}& = \frac{C_{\Gamma}d^{1/\log n+1/2}\ConstDM_{1} \taumix }{a(1-\gamma)}\\
         \Constupdboot{1,2}& = \frac{k_0^{\gamma}}{ ac_0(1-\gamma)} C_{\Gamma}^2d^{2/\log n+1/2} \norm{\theta_0 - \thetas} + \taumix\supconsteps (\sqrt{\qcond}C_{\gamma,a}^{(S)}k_0^{\gamma-1} + \sqrt{\qcond}(\norm{\bA^{-1}} + (8/a))\\
         \Constupdboot{1,3}& = \ConstDM_{3}\taumix  + \taumix\supconsteps\sqrt{\qcond}C_{\gamma,a}^{(S,2)}\\
        \Constupdboot{1,4}&=\taumix \supconsteps\sqrt{\qcond}C_{\gamma,a}^{(S)}\\
        \Constupdboot{1,5}& = \frac{\normop{A^{-1}}\taumix\supconsteps\sqrt{\qcond}}{\sqrt{a}}\\
        \Constupdboot{1,6} & = \taumix\supconsteps\sqrt{\qcond}C_{\gamma,a}^{(S)}
 \end{split}
\end{equation}
\end{lemma}
\begin{proof}
    We split $\bar{D}_{b_n,t}$ into three parts:
    \begin{equation}
        \begin{split}
            T_1 &= \frac{1}{b_n}\sum_{k=1}^{b_n-1} u^{\top}\ProdB_{t+1:t+k} (\theta_t - \thetas)\eqsp, \\T_2 &=  \frac{1}{b_n} \sum_{k=1}^{b_n-1}u^{\top} H_{k,t,b_n}^{(0)} \eqsp,\\ 
            T_3 &=- \frac{1}{b_n} \sum_{k=1}^{b_n-1} u^{\top}(Q_{k,t,b_n} - \bA^{-1}) \funnoisew(\State_{k+t}) \eqsp.
        \end{split}
    \end{equation}
       We start from $T_1$. Using Minkovski's and H\"older's inequality, we obtain 
    \begin{equation}
        \PE_{\xi}^{1/p}[| u^{\top}T_1 |^p] \leq \frac{1}{b_n}\sum_{k=1}^{b_n-1}(\PE^{1/p}_{\xi}[\normop{\ProdB_{t+1:t+k}}^{2p}])^{1/(2p)}(\PE^{1/p}_{\xi}[\normop{(\theta_t - \thetas)}^{2p}])^{1/(2p)}\eqsp.
    \end{equation}
    Applying similar arguments as in \eqref{eq:bound_norm_through_scalar_norm} together with \Cref{prop:products_of_matrices_UGE} and \Cref{prop:last_iterate_bound} we get 
    \begin{align}
        \begin{aligned}
            &\PE_{\xi}^{1/p}[| u^{\top}T_1 |^p] \\& \lesssim \frac{1}{b_n}\sum_{k=1}^{b_n-1} C_{\Gamma}d^{1/\log n+1/2}\exp\biggl\{ -(a/12)\sum_{\ell=t+1}^{t+k}\alpha_\ell\biggr\}\ConstDM_{1} \taumix \sqrt{p \alpha_{t}} \\& + \frac{1}{b_n}\sum_{k=1}^{b_n-1} C_{\Gamma}^2d^{2/\log n+1/2}\exp\biggl\{ -(a/12)\sum_{\ell=1}^{t+k}\alpha_\ell\biggr\} \norm{\theta_0 - \thetas}\eqsp.
        \end{aligned}
    \end{align}
    Using \Cref{lem:bounds_on_sum_step_sizes}, \Cref{lem:bound_sum_exponent} and that $k_0^{1-\gamma} > \frac{24}{ac_0}$ we get 
    \begin{align}
        \begin{aligned}
            &\PE_{\xi}^{1/p}[| u^{\top}T_1 |^p] \\& \lesssim \frac{1}{b_n\sqrt{{\alpha_t}}a(1-\gamma)} C_{\Gamma}d^{1/\log n+1/2}\ConstDM_{1} \taumix \sqrt{p} \\& + \frac{k_0^{\gamma}\exp\{-\frac{ac_0}{24(1-\gamma)}t^{1-\gamma}\}}{b_n ac_0(1-\gamma)} C_{\Gamma}^2d^{2/\log n+1/2} \norm{\theta_0 - \thetas}\eqsp.
        \end{aligned}
    \end{align}
    For $T_2$ we use \Cref{prop:H_n_0_bound_Markov} and Minkovski's inequality and get 
    \begin{align}
\PE_\xi^{1/p}[|T_2|^p] &\le \frac{1}{b_n} \sum_{k=1}^{b_n-1}  \PE_\xi^{1/p}[\|u^\top \Hnalpha{k,t,b_n}{0}\|^p] \le  \frac{\ConstDM_{3} \taumix  p^{2}}{ b_n} \sum_{k = 1+t}^{b_n+t-1} \alpha_{k} \sqrt{\log{(1/\alpha_k)}} \\& \lesssim
\ConstDM_{3} \taumix  p^{2}\sqrt{\log (1/\alpha_{b_n+t-1})}\frac{\sum_{k=1+t}^{b_n+t-1}\alpha_k}{b_n}\eqsp.
\end{align}

We first use \Cref{repr:qtminusa} to show that  
\begin{align}
T_3 = - \frac{1}{b_n}\sum_{\ell=1}^{b_n-1} u^{\top}S_{\ell,t,b_n} \funnoisew(\State_{\ell+t})  +  \frac{1}{b_n}\sum_{\ell=1}^{b_n-1}  u^{\top}\bA^{-1} G_{\ell+t:b_n+t-1} \funnoisew(\State_{\ell+t}) = T_{31} + T_{32} \eqsp,
\end{align}
where $S_{\ell,t,b_n}= \sum_{j = \ell+1}^{b_n-1}(\alpha_{k+t}-\alpha_{j+t})G_{l+t+1:j+t-1}$.
It is easy to see that $T_{31}$ is a weighed linear statistics of MC. We apply Markov property together with \Cref{lem:auxiliary_rosenthal_weighted} and obtain 
\begin{align}
\begin{aligned}
\PE_\xi^{1/p}[|T_{31}|^p] &\lesssim \frac{\taumix p^{1/2}\supconsteps}{b_n}(\sum_{k=2}^{b_n-1}\norm{S_{k,t,b_n}}^2)^{1/2}\\&\quad+\frac{\taumix \supconsteps}{b_n}\bigl(\|S_{1,t,b_n}\| + \|S_{b_n-1,t,b_n}\| + \sum_{k=1}^{b_n-2}\|S_{k+1,t,b_n}-S_{k,t,b_n}\| \bigr)\eqsp.
\end{aligned}
\end{align}
Applying \Cref{prop:st_bound} and \Cref{prop:st_diff_bound}, we get 
\begin{align}
\begin{aligned}
    \PE_\xi^{1/p}[|T_{31}|^p] &\lesssim \frac{\taumix p^{1/2}\supconsteps\sqrt{\qcond}C_{\gamma,a}^{(S)}}{b_n}(\sum_{k=2}^{b_n-1}(k+k_0+t)^{2\gamma-2})^{1/2}\\&\qquad+\frac{\taumix\supconsteps}{b_n} \bigl(\sqrt{\qcond}C_{\gamma,a}^{(S)}((k_0+t+1)^{\gamma-1} + (b_n+k_0+t-1)^{\gamma-1})  \\& \qquad \qquad +\sqrt{\qcond}C_{\gamma,a}^{(S,2)}\sum_{k=1}^{b_n-2}\alpha_{k+1+t} \bigr)\\&
\end{aligned}
\end{align}
For term $T_{32}$ we also apply \Cref{lem:auxiliary_rosenthal_weighted} and get
\begin{align}
\begin{aligned}
    &\PE_\xi^{1/p}[|T_{32}|^p] \lesssim \frac{\norm{\bA^{-1}}\taumix p^{1/2}\supconsteps}{b_n}(\sum_{k=2}^{b_n-1}\norm{G_{k+t:b_n+t-1}}^2)^{1/2}\\&+\frac{\taumix \supconsteps}{b_n}\bigl(\norm{\bA^{-1}}(\| G_{1+t:b_n+t-1}\| + \| G_{b_n+t-1:b_n+t-1}\|) + \sum_{k=1}^{b_n-2}\|\bA^{-1}( G_{k+1+t:b_n+t-1}- G_{k+t:b_n+t-1})\| \bigr)\eqsp.
\end{aligned}
\end{align}
Using \Cref{prop:g_t:n_upperbound}, \Cref{lem:prod_alpha_k} and \Cref{lem:sum_alpha_k_squared} with $b = a/2$ we get 
\begin{align}
\begin{aligned}
    &\PE_\xi^{1/p}[|T_{32}|^p] \lesssim \frac{\norm{\bA^{-1}}\taumix p^{1/2}\supconsteps\sqrt{\qcond}}{\sqrt{ac_0}b_n}(b_n+t+k_0-2)^{\gamma/2}\\&\quad+\frac{\taumix \supconsteps}{b_n}\bigl(\norm{\bA^{-1}}(\alpha_{b_n+t-1}/\alpha_t + \sqrt{\qcond}(1-(a/2)\alpha_{b_n+t-1})) + \sqrt{\qcond}\sum_{k=1}^{b_n-1}\alpha_{k+t}\prod_{i=k+t+1}^{b_n+t-1}(1-(a/2)\alpha_i) \bigr)\\&\qquad \qquad \qquad \lesssim 
    \frac{\norm{\bA^{-1}}\taumix p^{1/2}\supconsteps\sqrt{\qcond}}{\sqrt{a}b_n\sqrt{\alpha_{b_n+t-2}}}\\&\quad+\frac{\taumix \supconsteps \sqrt{\qcond}}{b_n}\bigl(\norm{\bA^{-1}} + (8/a)\bigr) \eqsp.
\end{aligned}
\end{align}
\end{proof}

\begin{lemma}
    \label{lem:bound_barD_n}
    Assume \Cref{assum:UGE}, \Cref{assum:noise-level} and \Cref{assum:step-size}. Then it holds 
    \begin{align}
    \begin{aligned}
        \PE_{\xi}^{1/p}[|u^{\top}\bar D_n|^p]   &\lesssim \frac{\Constupdboot{2,1}}{n} + \Constupdboot{2,2}p^2n^{-\gamma} + \Constupdboot{2,3}p^{1/2}n^{\gamma-3/2} \\& + \Constupdboot{2,4}p^{1/2}n^{\gamma/2-1}+\Constupdboot{2,5}p^2\sqrt{\log n}n^{-\gamma}\eqsp,
    \end{aligned}
    \end{align}
    where 
    \begin{equation}
        \begin{split}
            \Constupdboot{2,1} &=  C_\Gamma d^{1/{\log{n}}}\frac{k_0^\gamma}{ac_0(1-\gamma)}\|\theta_0 - \thetas \| + \taumix \supconsteps (\sqrt{\qcond}(\norm{\bA^{-1}} + 1/a)+\sqrt{\qcond}C_{\gamma,a}^{(S)}k_0^{\gamma-1})\\
            \Constupdboot{2,2}& = \frac{\ConstDM_{3} \taumix \sqrt{\log \frac{k_0^{\gamma}}{c_0}}}{c_0(1-\gamma)} + \frac{\sqrt{\qcond}\taumix \supconsteps C_{\gamma,a}^{(S,2)}c_0}{1-\gamma} \\
            \Constupdboot{2,3}& = \frac{\taumix \supconsteps\sqrt{\qcond}C_{\gamma,a}^{(S)}}{\sqrt{2\gamma-1}}\\      \Constupdboot{2,4}&=\frac{\norm{\bA^{-1}}\taumix\supconsteps\sqrt{\qcond}k_0^{\gamma/2}}{\sqrt{ac_0}}\\
            \Constupdboot{2,5}&=\frac{\ConstDM_{3} \taumix c_0\sqrt{\gamma}}{1-\gamma}
        \end{split}
    \end{equation}
\end{lemma}
\begin{proof}
    Using \eqref{eq:bound_D11} and \eqref{eq:bound_D22} we get 
    \begin{align}
        \begin{aligned}
            \PE_{\xi}^{1/p}[u^{\top}( \frac{1}{n} \sum_{k = 0}^{n-1} \ProdB_{1:k} (\theta_0 - \thetas) +\frac{1}{n} \sum_{k=1}^{n-1} H_k^{(0)})|^p] &\lesssim C_\Gamma d^{1/{\log{n}}}\frac{1}{ n}\frac{k_0^\gamma}{ac_0(1-\gamma)}\|\theta_0 - \thetas \| \\& +
             \frac{\ConstDM_{3} \taumix  p^{2}c_0}{1-\gamma}(\sqrt{\gamma\log n} + \sqrt{\log \frac{k_0^{\gamma}}{c_0}})n^{-\gamma}\eqsp.
        \end{aligned}
    \end{align}
For simplicity we define 
\begin{equation}
T_4 =  -  u^{\top}\frac{1}{n} 
\sum_{\ell=1}^{n-1} (Q_\ell-\bA^{-1}) \funnoisew(\State_\ell)
\end{equation}

We first use \Cref{repr:qtminusa} to show that  
\begin{align}
T_{4} = - \frac{1}{ n}\sum_{\ell=1}^{n-1} u^{\top}S_\ell \funnoisew(\State_\ell)  +  \frac{1}{ n}\sum_{\ell=1}^{n-1}  u^{\top}\bA^{-1} G_{\ell:n-1} \funnoisew(\State_\ell) = T_{41} + T_{42} \eqsp. 
\end{align}
It is easy to see that $T_{41}$ is a weighed linear statistics of MC. We apply \Cref{lem:auxiliary_rosenthal_weighted}, 
\begin{align}
\PE_\xi^{1/p}[|T_{41}|^p] &\lesssim \frac{\taumix p^{1/2}\supconsteps}{n}(\sum_{k=2}^{n-1}\norm{S_k}^2)^{1/2}\\&\quad+\frac{\taumix \supconsteps}{n}\bigl(\|S_1\| + \|S_{n-1}\| + \sum_{k=1}^{n-2}\|S_{k+1}-S_{k}\| \bigr)\eqsp.
\end{align}
Applying \Cref{prop:st_bound} and \Cref{prop:st_diff_bound}, we get 
\begin{align}
\begin{aligned}
    \PE_\xi^{1/p}[|T_{41}|^p] &\lesssim \frac{\taumix p^{1/2}\supconsteps\sqrt{\qcond}C_{\gamma,a}^{(S)}}{\sqrt{2\gamma-1}}n^{\gamma-3/2}\\&\quad+\frac{\taumix\supconsteps}{n} \bigl(\sqrt{\qcond}C_{\gamma,a}^{(S)}((k_0+1)^{\gamma-1} + (n+k_0-1)^{\gamma-1}) + \frac{\sqrt{\qcond}C_{\gamma,a}^{(S,2)}c_0}{1-\gamma}n^{1-\gamma} \bigr)\\& \lesssim
    \frac{\taumix p^{1/2}\supconsteps\sqrt{\qcond}C_{\gamma,a}^{(S)}}{\sqrt{2\gamma-1}}n^{\gamma-3/2} + \taumix \supconsteps\sqrt{\qcond}C_{\gamma,a}^{(S)}k_0^{\gamma-1}n^{-1} \\& \quad +\frac{\sqrt{\qcond}\taumix \supconsteps C_{\gamma,a}^{(S,2)}c_0}{1-\gamma}n^{-\gamma} 
\end{aligned}
\end{align}
For term $T_{42}$ we also apply \Cref{lem:auxiliary_rosenthal_weighted} and get
\begin{align}
\begin{aligned}
    &\PE_\xi^{1/p}[|T_{42}|^p] \lesssim \frac{\norm{\bA^{-1}}\taumix p^{1/2}\supconsteps}{n}(\sum_{k=2}^{n-1}\norm{G_{k:n-1}}^2)^{1/2}\\&+\frac{\taumix \supconsteps}{n}\bigl(\norm{\bA^{-1}}(\| G_{1:n-1}\| + \| G_{n-1:n-1}\|) + \sum_{k=1}^{n-2}\|\bA^{-1}( G_{k+1:n-1}- G_{k:n-1})\| \bigr)\eqsp.
\end{aligned}
\end{align}
Using \Cref{prop:g_t:n_upperbound}, \Cref{lem:prod_alpha_k} and \Cref{lem:sum_alpha_k_squared} with $b = a/2$ we get 
\begin{align}
\begin{aligned}
    \PE_\xi^{1/p}[|T_{32}|^p] &\lesssim \frac{\norm{\bA^{-1}}\taumix p^{1/2}\supconsteps\sqrt{\qcond}}{\sqrt{ac_0}n}(n+k_0-2)^{\gamma/2}\\&\quad+\frac{\taumix \supconsteps}{n}\bigl(\norm{\bA^{-1}}(\alpha_{n-1}/\alpha_0 + \sqrt{\qcond}(1-(a/2)\alpha_{n-1})) + \sqrt{\qcond}(8/a) \bigr) \\& \lesssim\frac{\norm{\bA^{-1}}\taumix p^{1/2}\supconsteps\sqrt{\qcond}k_0^{\gamma/2}}{\sqrt{ac_0}}n^{\gamma/2-1}\\&\quad+\frac{\taumix \supconsteps (\norm{\bA^{-1}} + \sqrt{\qcond})}{n}
    \eqsp.
\end{aligned}
\end{align}
\end{proof}

\begin{lemma}
    \label{lem:bound_barW_and_barW_b}
    Assume \Cref{assum:UGE}, \Cref{assum:noise-level} and \Cref{assum:step-size}. Then in holds 
    \begin{equation}
         \PE_{\xi}^{1/p}[|u^{\top}\bar{W}_{n}|^p] \lesssim \Constupdboot{3,1}p^{1/2}n^{-1/2} + \frac{\Constupdboot{3,2}}{n}\eqsp,
    \end{equation}
    and
    \begin{equation}
          \PE_{\xi}^{1/p}[|u^{\top}\bar{W}_{b_n,t}|^p] \lesssim \Constupdboot{3,1}p^{1/2}b_n^{-1/2} + \frac{\Constupdboot{3,2}}{b_n}\eqsp,
    \end{equation}
    where
    \begin{equation}
        \begin{split}
            \Constupdboot{3,1}& =\taumix \normop{\bA^{-1}}\supconsteps\\
            \Constupdboot{3,2}& =\taumix \normop{\bA^{-1}}\eqsp.
        \end{split}
    \end{equation}
\end{lemma}
\begin{proof}
Using \Cref{lem:auxiliary_rosenthal_weighted}, we get 
\begin{equation}
    \PE_{\xi}^{1/p}[|u^{\top}\bar{W}_{n}|^p] \lesssim \frac{\taumix p^{1/2}\normop{\bA^{-1}}\supconsteps}{\sqrt{n}}  + \frac{\taumix \normop{\bA^{-1}}}{n}
\end{equation}
The boundary for $\bar{W}_{b_n,t}$ is obtained similarly.
\end{proof}
\begin{lemma}
\label{lem:bound_D_1_b}
Assume \Cref{assum:UGE}, \Cref{assum:noise-level} and \Cref{assum:step-size}. Then it holds 
    \begin{align}
    \begin{aligned}
        &\PE_{\xi}^{1/p}[| D_1^b|^p] \lesssim M_{1,1}pn^{\gamma}b_n^{-1}
       +M_{1,2}b_n^{-1} + M_{1,3} p^4(\log n)n^{-1} + 
       M_{1,4}pn^{2\gamma-2}\\&
       + M_{2,1}b_nn^{-2} + M_{2,2}p^4b_n(\log n)n^{-2\gamma} + M_{2,3}pb_nn^{2\gamma-3} + M_{2,4}pb_nn^{\gamma-2}\eqsp,
    \end{aligned}
    \end{align}
    where 
    \begin{equation}
    \begin{split}
        M_{1,1} &=  \frac{(2k_0)^{\gamma+1}((\Constupdboot{1,1})^2 + (\Constupdboot{1,5})^2)}{c_0(\gamma+1)} \\
        M_{1,2} &= (\Constupdboot{1,2})^2+ (\Constupdboot{1,6})^2 \\
        M_{1,3} &= (\Constupdboot{1,3})^2(\gamma + \log \frac{k_0^\gamma}{c_0})\frac{c_0^2}{2\gamma-1}\\
        M_{1,4} &= (\Constupdboot{1,4})^2\frac{2^{2\gamma-1}-1}{2\gamma-1}\\
        M_{2,1} & = (\Constupdboot{2,1})^2\\
        M_{2,2} & = (\Constupdboot{2,2})^2 + (\Constupdboot{2,5})^2\\
        M_{2,3} & = (\Constupdboot{2,3})^2\\
        M_{2,4} & = (\Constupdboot{2,4})^2\eqsp.
    \end{split}
\end{equation}
\end{lemma}
\begin{proof}
Note that using Minkowski's inequality and H\"older's inequality, we have 
\begin{align}
    \PE_{\xi}^{1/p}[|D_1^b|^p] &\leq \frac{b_n}{n-b_n+1}\sum_{t=0}^{n-b_n}\biggl(\PE^{1/p}_{\xi}[|u^{\top}\bar D_{b_n,t}|^{2p} + \PE^{1/p}_{\xi}[|u^{\top}\bar D_{b}|^{2p} \\& \qquad \qquad + 2\PE^{1/(2p)}_{\xi}[|u^{\top}\bar D_{b_n,t}|^{2p} \PE^{1/(2p)}_{\xi}[|u^{\top}\bar D_{b}|^{2p}] \biggr) \\&\leq \frac{2b_n}{n-b_n+1}\sum_{t=0}^{n-b_n}\biggl\{ \PE^{1/p}_{\xi}[|u^{\top}\bar D_{b_n,t}|^{2p} + \PE^{1/p}_{\xi}[|u^{\top}\bar D_{b}|^{2p}\biggr\}
\end{align}
Using \Cref{lem:bound_barD_b_t} and \Cref{lem:bounds_on_sum_step_sizes}
we obtain 
\begin{align}
    \begin{aligned}
       &\frac{2b_n}{n-b_n+1}\sum_{t=0}^{n-b_n} \PE^{1/p}_{\xi}[|u^{\top}\bar D_{b_n,t}|^{2p} \lesssim  \frac{b_n}{n-b_n+1}\biggl\{\sum_{t=0}^{n-b_n} \frac{(\Constupdboot{1,1})^2p}{b_n^2\alpha_t}+ \frac{(\Constupdboot{1,2})^2}{b_n^2} \\&+ (\Constupdboot{1,3})^2p^4\log (1/\alpha_{b_n+t-1})\frac{(\sum_{k=1+t}^{b_n+t-1}\alpha_k)^2}{b_n^2}\\ &+(\Constupdboot{1,4})^2\frac{ p}{b_n^2}\sum_{k=2}^{b_n-1}(k+k_0+t)^{2\gamma-2}+ \frac{(\Constupdboot{1,5})^2pb_n^{-2}}{\alpha_{t+b_n-2}} + (\Constupdboot{1,6})^2\frac{(k_0+t-1)^{2\gamma-2}}{b_n^2}\biggr\} \\& \lesssim
       M_{1,1}pn^{\gamma}b_n^{-1} + M_{1,2}b_n^{-1} + M_{1,3}p^4\log nn^{-1} + M_{1,4}pn^{2\gamma-2}
    \end{aligned}
\end{align}

Using \Cref{lem:bound_barD_n}, we get  
\begin{align}
    \begin{aligned}
       \frac{2b_n}{n-b_n+1}\sum_{t=0}^{n-b_n}\PE^{1/p}_{\xi}[|u^{\top}\bar D_{b}|^{2p} &\lesssim M_{2,1}b_nn^{-2} + M_{2,2}p^4b_n(\log n)n^{-2\gamma} + M_{2,3}pb_nn^{2\gamma-3} + M_{2,4}pb_nn^{\gamma-2}\eqsp.
    \end{aligned}
\end{align}
\end{proof}
\begin{lemma}
    \label{lem:bound_D_2_b}
    Assume \Cref{assum:UGE}, \Cref{assum:noise-level} and \Cref{assum:step-size}. Then it holds 
    \begin{align}
        \begin{aligned}
            \PE_{\xi}^{1/p}[|D_2^b|^p]&\lesssim M_{3,1}p^{1/2}b_n^{1/2}n^{-1} + M_{3,2}pb_n^{1/2}n^{\gamma-3/2} + M_{3,3}pb_n^{1/2}n^{\gamma/2-1} +M_{3,4}p^{5/2}\sqrt{\log n}n^{-\gamma}b_n^{1/2} \\&+ M_{3,5}p^{1/2}b_n^{-1/2}  + M_{3,6}p^{5/2}\sqrt{\log n}b_n^{1/2}n^{-\gamma}+M_{3,7}pb_n^{-1/2}n^{\gamma-1/2} + M_{3,8}pb_n^{-1/2}n^{\gamma/2}\eqsp,
        \end{aligned}
    \end{align}
    where
    \begin{equation}
        \begin{split}
           M_{3,1} &= \Constupdboot{2,1}(\Constupdboot{3,1}+\Constupdboot{3,2})\\
           M_{3,2} &= \Constupdboot{2,3}(\Constupdboot{3,1}+\Constupdboot{3,2}) \\ 
           M_{3,3} &= \Constupdboot{2,4}(\Constupdboot{3,1}+\Constupdboot{3,2}) \\
           M_{3,4} &=(\Constupdboot{2,5}+ \Constupdboot{2,2})(\Constupdboot{3,1}+\Constupdboot{3,2}) 
        \end{split}
    \end{equation}
    \begin{equation}
        \begin{split}
            M_{3,5} &= (\Constupdboot{1,2} + \Constupdboot{1,6})(\Constupdboot{3,1}+\Constupdboot{3,2}) \\
           M_{3,6} &= \frac{\Constupdboot{1,3}}{1-\gamma} (\sqrt{\gamma} + \sqrt{\log\frac{k_0^\gamma}{c_0}})(\Constupdboot{3,1}+\Constupdboot{3,2}) \\
           M_{3,7} &=\frac{\Constupdboot{1,4}(2k_0)^{\gamma+1/2}}{\sqrt{2\gamma-1}(\gamma+1/2)}(\Constupdboot{3,1}+\Constupdboot{3,2})  \\
           M_{3,8} &= \frac{(2k_0)^{\gamma/2+1}(\Constupdboot{1,5}+\Constupdboot{1,1})}{\sqrt{c_0}(\gamma/2 + 1)} (\Constupdboot{3,1}+\Constupdboot{3,2})
        \end{split}
    \end{equation}
\end{lemma}
\begin{proof}
Note that using Minkowski's inequality together with H\"older's inequality we get 
\begin{align}
    \begin{aligned}
         \PE_{\xi}^{1/p}[|D_2^b|^p]&\lesssim  \frac{b_n}{n-b_n+1} \sum_{t = 0}^{n-b_n} \PE_{\xi}^{1/(2p)}[|u^\top \bar{W}_{b_n,t}|^{2p}] (\PE_{\xi}^{1/(2p)}[|u^\top \bar{D}_{b_n,t}|^{2p}] + \PE_{\xi}^{1/(2p)}[|u^\top \bar{D}_{n}|^{2p}]) \\&+ 
    \frac{b_n}{n-b_n+1} \sum_{t = 0}^{n-b_n} \PE_{\xi}^{1/(2p)}[|u^\top \bar{W}_{n}|^{2p}] (\PE_{\xi}^{1/(2p)}[|u^\top \bar{D}_{b_n,t}|^{2p}] + \PE_{\xi}^{1/(2p)}[|u^\top \bar{D}_{n}|^{2p}]) 
    \end{aligned}
\end{align}
 Note that bound on $\PE_{\xi}^{1/(2p)}[|u^\top \bar{W}_{b_n,t}|^{2p}]$, $\PE_{\xi}^{1/(2p)}[|u^\top \bar{W}_{n}|^{2p}]$, $\PE_{\xi}^{1/(2p)}[|u^\top \bar{D}_{n}|^{2p}]$ does not depend upon $t$, hence, we can rewrite formula above as follows 
 \begin{align}
    \begin{aligned}
         \PE_{\xi}^{1/p}[|D_2^b|^p]&\lesssim  \frac{b_n}{n-b_n+1} \biggl\{\sum_{t = 0}^{n-b_n} \PE_{\xi}^{1/(2p)}[|u^\top \bar{D}_{b_n,t}|^{2p}]\biggr\} (\PE_{\xi}^{1/(2p)}[|u^\top \bar{W}_{b_n,t}|^{2p}] + \PE_{\xi}^{1/(2p)}[|u^\top \bar{W}_{n}|^{2p}]) \\&+ 
    b_n \PE_{\xi}^{1/(2p)}[|u^\top \bar{D}_{n}|^{2p}] (\PE_{\xi}^{1/(2p)}[|u^\top \bar{W}_{b_n,t}|^{2p}] + \PE_{\xi}^{1/(2p)}[|u^\top \bar{W}_{n}|^{2p}]) 
    \end{aligned}
\end{align}
Applying \Cref{lem:bound_barD_b_t} and \Cref{lem:bounds_on_sum_step_sizes} we get 
\begin{align}
    \begin{aligned}
         \frac{b_n}{n-b_n+1} &\sum_{t=0}^{n-b_n}\PE_{\xi}^{1/(2p)} [|( u^{\top}\bar D_{b_n,t})|^{2p}\lesssim \frac{b_n}{n-b_n+1} \sum_{t=0}^{n-b_n} \biggl\{\frac{\Constupdboot{1,1}p^{1/2}}{b_n\sqrt{{\alpha_t}}}+ \frac{\Constupdboot{1,2}}{b_n} \\&+ \Constupdboot{1,3}p^2\sqrt{\log (1/\alpha_{b_n+t-1})}\frac{\sum_{k=1+t}^{b_n+t-1}\alpha_k}{b_n}\\ &+\Constupdboot{1,4}\frac{ p^{1/2}}{b_n}(\sum_{k=2}^{b_n-1}(k+k_0+t)^{2\gamma-2})^{1/2}+ \frac{\Constupdboot{1,5}p^{1/2}b_n^{-1}}{\sqrt{\alpha_{t+b_n-2}}} + \Constupdboot{1,6}\frac{(k_0+t-1)^{\gamma-1}}{b_n}\biggr\}
    \end{aligned}
\end{align}
Using \Cref{lem:bound_barD_n} and \Cref{lem:bound_barW_and_barW_b} we get 
\begin{align}
     \PE_{\xi}^{1/p}[|D_2^b|^p]&\lesssim \biggl\{\Constupdboot{2,1}b_nn^{-1} + \Constupdboot{2,3}p^{1/2}b_nn^{\gamma-3/2} + \Constupdboot{2,4}p^{1/2}b_nn^{\gamma/2-1} \\& +(\Constupdboot{2,5}+ \Constupdboot{2,2})p^2\sqrt{\log n}n^{-\gamma}b_n + \Constupdboot{1,2}  + \frac{\Constupdboot{1,3}}{1-\gamma}p^2 (\sqrt{\gamma} + \sqrt{\log\frac{k_0^\gamma}{c_0}})\sqrt{\log n}b_nn^{-\gamma} \\&+\frac{\Constupdboot{1,4}(2k_0)^{\gamma+1/2}}{\sqrt{2\gamma-1}(\gamma+1/2)}p^{1/2}n^{\gamma-1/2}  + \frac{(2k_0)^{\gamma/2+1}(\Constupdboot{1,5}+\Constupdboot{1,1})p^{1/2}(n-b_n+1)^{\gamma/2}}{\sqrt{c_0}(\gamma/2 + 1)}\biggr\}\cdot\\&\cdot(\Constupdboot{3,1}+\Constupdboot{3,2})p^{1/2}(n^{-1/2}+b_n^{-1/2})
\end{align}
\end{proof}
 Finally, from \Cref{lem:bound_D_1_b} and \Cref{lem:bound_D_2_b} it holds that 
 \begin{align}
     \PE_{\xi}^{1/p}\bigl[ \bigl| \RemCov_{var}(u) \bigr|^{p} \bigr] &\lesssim M_1pb_n^{1/2}n^{\gamma/2-1} + M_2p^{4}(\log n)b_n^{1/2}n^{-\gamma} \\&+ M_3pb_{n}^{-1/2}n^{\gamma/2} + M_4p^4(\log n)n^{-1} + M_5pn^{2\gamma-2}\eqsp,
 \end{align}
where 
\begin{equation}
\label{def:const_M_i}
    \begin{split}
        M_1 &= M_{3,1} + M_{3,2} + M_{3,3}+M_{2,3} + M_{2,4}\\
        M_2 &= M_{3,4} + M_{3,6} + M_{2,2}\\
        M_3 &= M_{3,5}+M_{3,7}+M_{3,8}+ M_{1,1} + M_{1,2}\\
        M_4 &= M_{1,3} + M_{2,1}\\
        M_5 &= M_{1,4}\eqsp.
    \end{split}
\end{equation}

\subsection{Proof of \Cref{coro:first-concentration_OBM}}
\label{sec:proof_concentration_OBM}
From \Cref{prop:block_bootstrap_batch_mean} and \Cref{prop:concentration_OBM} we get 
\begin{align}
    \begin{aligned}
        \PE^{1/p}_{\xi}[\bigl|\hat{\sigma}^2_{\theta}(u) - \sigma^2(u)\bigr|^p] &\lesssim \frac{p \taumix^3 \supconsteps^2}{\sqrt{n}} + \frac{p^2 \taumix^2 \sqrt{b_n} \supconsteps^2}{\sqrt{n}} + \frac{p \taumix^{2}\supconsteps^2}{\sqrt{b_n}}\\& \qquad+M_1pb_n^{1/2}n^{\gamma/2-1} + M_2p^{4}(\log n)b_n^{1/2}n^{-\gamma} \\&\qquad + M_3pb_{n}^{-1/2}n^{\gamma/2} + M_4p^4(\log n)n^{-1} + M_5pn^{2\gamma-2}\\&\lesssim
        p^4(\taumix^2\supconsteps^2 + M_2+M_1)(\log n)\frac{b_n^{1/2}}{\sqrt{n}} + p^4(\taumix^2\supconsteps^2 + M_3) \frac{n^{\gamma/2}}{\sqrt{b_n}} \\&\qquad + p^4(\taumix^3\supconsteps^2 + M_4)\frac{(\log n)}{\sqrt{n}} + p^4M_5n^{2\gamma-2}
    \end{aligned}
\end{align}
Hence, applying \Cref{lem:markov_inequality} with $p = \log n$ we obtain that with probability $1-\frac{1}{n}$ it holds
\begin{align}
    \begin{aligned}    \label{eq:high_prob_bound_sigma_theta}
        \bigl|\hat{\sigma}^2_{\theta}(u) - \sigma^2(u)\bigr| &\lesssim \rme (\taumix^2\supconsteps^2 + M_2+M_1) (\log n)^5\frac{b_n^{1/2}}{\sqrt{n}} \\&+ \rme (\log n)^4(\taumix^2\supconsteps^2 + M_3) \frac{n^{\gamma/2}}{\sqrt{b_n}}\\& + \rme(\taumix^3\supconsteps^2 + M_4)\frac{(\log n)^5}{\sqrt{n}}\\& + \rme (\log n)^4M_5n^{2\gamma-2}
    \end{aligned}    
\end{align}
By setting $b_n= b^\beta$, we can optimize the inequality above by $\gamma$ and $\beta$. And by putting $\beta = 3/4$ and $\gamma = 1/2 + \eps$ where $\eps < 1/\log n$ we finally get   that with probability $1-\frac{1}{n}$ it holds
\begin{align}
    \begin{aligned}
        &\bigl|\hat{\sigma}^2_{\theta}(u) - \sigma^2(u)\bigr| \\ & \qquad \lesssim \rme(\taumix^2\supconsteps^2 + M_2+M_1 + \taumix^2\supconsteps^2 + M_3 + \taumix^3\supconsteps^2 + M_4 + \rme^2M_5)(\log n)^5n^{-1/8 + \eps/2}
    \end{aligned}
\end{align}

\subsection{Full version of proof of \Cref{th:bootstrap_validity_main}}
\label{sec:proof_bootstrap_validity_main}
We begin this section with a lemma on Gaussian comparison.
\begin{lemma}[Theorem~1.3 in \cite{devroye2018total}]
\label{lem:Pinsker}
Let $\xi_i \sim \mathcal{N}(0, \sigma_i^2)$, $i = 1,2$. Assume
$$
|\sigma_1^2/\sigma_2^2 - 1| \le \delta,
$$
for some $\delta \geq 0$. Then
$$
\sup_{x \in \rset} |\PP(\xi_1 \le x) - \PP(\xi_2 \le x)| \le \frac{3}{2}\delta.
$$
\end{lemma}
\begin{proof}
    Let $\PP_i = \mathcal N(0, \sigma_i^2)$. The Kullback-Leibler divergence between $\PP_1$ and $\PP_2$ is equal to
    $$
    \operatorname{KL(\PP_1, \PP_2)} = -0.5 \log (\sigma_1^2/\sigma_2^2) + 0.5 (\sigma_1^2/\sigma_2^2 - 1) = 0.5 (\sigma - \log (1 + \sigma)),
    $$
    where $\sigma = \sigma_1^2/\sigma_2^2-1$. If $\sigma > -2/3$ hence,
    $$
    \operatorname{KL(\PP_1, \PP_2)} \le 0.5 \sigma^2 \le 0.5 \delta^2
    $$
    If $\sigma \leq -2/3$ then 
    $$
    \operatorname{TV(\PP_1, \PP_2)}\leq 1 \leq (3/2) |\sigma| \leq (3/2)\delta
    $$
    It remains to apply the Pinsker inequality
    $$
\sup_{x \in \rset} |\PP(\xi_1 \le x) - \PP(\xi_2 \le x)| \le \operatorname{TV(\PP_1, \PP_2)} \le \sqrt{\operatorname{KL(\PP_1, \PP_2)}/2} \le 3/2 \delta. 
$$  
\end{proof}
Using \eqref{eq:high_prob_bound_sigma_theta} together with the inequality $\sigma^{-2}(u) \leq 1/\lambda_{\min}(\Sigma_{\infty})$ we get that with probability $1-1/n$ it holds 
\begin{align}
    \begin{aligned}   
        \bigl|\hat{\sigma}^2_{\theta}(u) - \sigma^2(u)\bigr|/\sigma^2(u) &\lesssim \frac{\rme}{\lambda_{\min}(\Sigma_{\infty})} (\taumix^2\supconsteps^2 + M_2+M_1) (\log n)^5\frac{b_n^{1/2}}{\sqrt{n}} \\&+ \frac{\rme}{\lambda_{\min}(\Sigma_{\infty})}  (\log n)^4(\taumix^2\supconsteps^2 + M_3) \frac{n^{\gamma/2}}{\sqrt{b_n}}\\& + \frac{\rme}{\lambda_{\min}(\Sigma_{\infty})} (\taumix^3\supconsteps^2 + M_4)\frac{(\log n)^5}{\sqrt{n}}\\& + \frac{\rme}{\lambda_{\min}(\Sigma_{\infty})}  (\log n)^4M_5n^{2\gamma-2}
    \end{aligned}    
\end{align}
Now we apply \Cref{lem:Pinsker} and obtain that with probability $1-1/n$ it holds 
\begin{align}
    \begin{aligned}
        \kolmogorov{\mathcal{N}(0, \hat\sigma_{\theta}^2(u)), \mathcal{N}(0, \sigma^2(u))} &\lesssim \frac{1}{\lambda_{\min}(\Sigma_{\infty})} (\taumix^2\supconsteps^2 + M_2+M_1) (\log n)^5\frac{b_n^{1/2}}{\sqrt{n}} \\&+ \frac{1}{\lambda_{\min}(\Sigma_{\infty})}  (\log n)^4(\taumix^2\supconsteps^2 + M_3) \frac{n^{\gamma/2}}{\sqrt{b_n}}\\& + \frac{1}{\lambda_{\min}(\Sigma_{\infty})}(\taumix^3\supconsteps^2 + M_4)\frac{(\log n)^5}{\sqrt{n}}\\& + \frac{1}{\lambda_{\min}(\Sigma_{\infty})}  (\log n)^4M_5n^{2\gamma-2}\eqsp.
    \end{aligned}
\end{align}
Combining this inequality with \Cref{cor:normal_approximation_markov} we obtain with probability $1-1/n$
\begin{align}
    \sup_{x \in \rset} |\PP(\sqrt n (\bar{\theta}_{n} - \thetas)^\top u \le x ) &- \PPb(\bar{\theta}_{n, b_n}(u) \le x)| \lesssim (\ConstK[1]+\ConstK[2]+ \Constupd{1}\|\theta_0 - \thetas \| + \Constupd{2})\frac{\log {n}}{n^{1/4}} \\&+ (\Constupd{4}+\Constupd{3})\frac{(\log n)^{5/2}}{n^{\gamma-1/2}} \\&+ \frac{1}{\lambda_{\min}(\Sigma_{\infty})} (\taumix^2\supconsteps^2 + M_2+M_1) (\log n)^5\frac{b_n^{1/2}}{\sqrt{n}} \\&+ \frac{1}{\lambda_{\min}(\Sigma_{\infty})}  (\log n)^4(\taumix^2\supconsteps^2 + M_3) \frac{n^{\gamma/2}}{\sqrt{b_n}}\\& + \frac{1}{\lambda_{\min}(\Sigma_{\infty})}(\taumix^3\supconsteps^2 + M_4)\frac{(\log n)^5}{\sqrt{n}}\\& + \frac{1}{\lambda_{\min}(\Sigma_{\infty})}  (\log n)^4M_5n^{2\gamma-2}
\end{align}
To complete the proof, it remains to optimize the bound above. Setting $b_n = \lceil n^{4/5} \rceil$ and $\gamma = 3/5$, we obtain with probability $1-1/n$ 
\begin{align}
    \begin{aligned}
          \sup_{x \in \rset} |\PP(\sqrt n (\bar{\theta}_{n} - \thetas)^\top u \le x ) &- \PPb(\bar{\theta}_{n, b_n}(u) \le x)| \lesssim (\ConstK[1]+\ConstK[2]+ \Constupd{1}\|\theta_0 - \thetas \| + \Constupd{2})\frac{\log {n}}{n^{1/4}} \\&+ \frac{1}{\lambda_{\min}(\Sigma_{\infty})}(\taumix^3\supconsteps^2 + M_4)\frac{(\log n)^5}{\sqrt{n}} \\&+ \biggl(\Constupd{3}+\Constupd{4}+\frac{\taumix^3\supconsteps^2 + M_1 + M_2 + M_3 + M_5}{\lambda_{\min}(\Sigma_{\infty})}\biggr)(\log n)^5n^{-1/10} \eqsp.
    \end{aligned}
\end{align}

\section{Applications to the TD learning algorithm}
\label{appendix:td_learning}

Recall that the TD learning algorithm within the framework of linear stochastic approximation (LSA) can be written as
\begin{equation}
\label{eq:LSA_procedure_TD_appendix}
\theta_{k} = \theta_{k-1} - \alpha_{k} (\funcAw_{k} \theta_{k-1} - \funcbw_{k})\eqsp,
\end{equation}
where the matrices $\funcAw_{k}$ and vectors $\funcbw_{k}$ are defined by
\begin{equation}
\label{eq:matr_A_def_appendix}
\begin{split}
\funcAw_{k} &= \varphi(s_k)\{\varphi(s_k) - \lambda \varphi(s_{k+1})\}^{\top}\eqsp, \\
\funcbw_{k} &= \varphi(s_k) r(s_k,a_k)\eqsp.
\end{split}
\end{equation}

Our primary objective is to estimate the agent's \emph{value function}, defined as
\begin{equation}
\label{eq:value_function_def_append}
\textstyle
V^{\pi}(s) = \PE\left[\sum_{k=0}^{\infty}\lambda^{k}r(s_k,a_k)\middle| s_0 = s\right]\eqsp,
\end{equation}
where $a_{k} \sim \pi(\cdot \mid s_k)$ and $s_{k+1} \sim \PMDP(\cdot \mid s_{k}, a_{k})$ for all $k \in \nset$. In this context, the TD learning updates \eqref{eq:LSA_procedure_TD_appendix} correspond to approximating solutions to the deterministic linear system $\bA \thetas = \barb$ (see \cite{tsitsiklis:td:1997}), where the matrix $\bA$ and right-hand side $\barb$ are given by
\begin{align}
\label{eq:system_matrix}
\bA &= \PE_{s \sim \mu, s' \sim \PMDP_{\pi}(\cdot\mid s)} \left[\varphi(s)\{\varphi(s)-\lambda \varphi(s')\}^{\top}\right] \\
\barb &= \PE_{s \sim \mu, a \sim \pi(\cdot\mid s)}\left[\varphi(s) r(s,a)\right]\eqsp.
\end{align}
The rest of the proof of \Cref{prop:assumption_check_TD} reduces to checking the properties of $\bA$ and direct verification of conditions of \Cref{prop:hurwitz_stability}, which is done by the lines of \cite[Proposition 2]{samsonov2024gaussian}

\paragraph{Numerical experiments.} We consider the simple instance of the Garnet problem \cite{archibald1995generation, geist2014off}. This problem is characterized by the number of states $N_s$, number of actions $a$, and branching factor $b$. Here $b$ corresponds to the number of states $s'$, that can be reached when starting from a given state-action pair $(s,a)$. We set the hyperparameter values $N_s = 6$, $a=2$, $b=3$, feature dimension $d = 2$, and discount factor $\lambda = 0.8$. We aim to evaluate the value function of policy $\pi(\cdot|s)$, which is given, for any $a \in \seta = \{1,2\}$, by the expression 
\[
\pi(a|s) = \frac{U_{a}^{(s)}}{\sum_{i=1}^{|\seta|}U_{i}^{(s)}}\eqsp,
\]
where the $U_{i}^{(s)}$ are \iid\ observations with $\mathcal{U}[0,1]$. We consider the problem of policy evaluation in this MDP using the TD learning algorithm with randomly generated feature mapping, that is, we generate the matrix 
\[
\Phi \in \rset^{N_s \times d}
\]
with \iid\ $\mathcal{N}(0,1)$ entries, and then take $\phi(s)$, $s \in \{1,\ldots,|S|\}$, to its $s$-th row, normalized by its euclidean norm: $\phi(s) = \Phi_s / \norm{\Phi_s}$. We run the procedure \eqref{eq:lsa} with the learning rates $\alpha_{k} = c_{0}/(k_0 + k)^{\gamma}$ with $\gamma = 2/3$ with appropriately chosen $c_0$ and $k_0$. We generate random vector $u$ form unitary sphere, and compute coverage probabilities for $u^{\top} \thetas$ for confidence levels $\{0.8, 0.9, 0.95\}$. The detailed setting of the experiments follows \cite{Fang2018}. Results are given in \Cref{tab:obm_coverage_cleaned} and illustrates the consistency of multiplier subsample bootstrap procedure applied on the Garnet problem.

\begin{table}[h!]
\centering
\caption{Coverage probabilities of OBM estimation for the empirical distribution.}
\resizebox{\textwidth}{!}{%
\begin{tabular}{c|ccc|ccc|ccc|c}
\toprule
\multirow{2}{*}{$n$ \, | \, $b_n$} 
& \multicolumn{3}{c|}{\textbf{0.95}} 
& \multicolumn{3}{c|}{\textbf{0.9}} 
& \multicolumn{3}{c|}{\textbf{0.8}} 
& \multirow{2}{*}{\shortstack{stddev \\ $\times 10^3$}} \\
\cmidrule(lr){2-4} \cmidrule(lr){5-7} \cmidrule(lr){8-10}
& \scriptsize $\hat{\sigma}^2_\theta(u)$ 
& \scriptsize $\sigma^2(u)$ 
&  
& \scriptsize $\hat{\sigma}^2_\theta(u)$ 
& \scriptsize $\sigma^2(u)$ 
&  
& \scriptsize $\hat{\sigma}^2_\theta(u)$ 
& \scriptsize $\sigma^2(u)$ 
&  \\
\midrule
20480 \ | \ 250     & 0.873 & 0.881 &  & 0.773 & 0.805 &  & 0.641 & 0.662 &  & 10.89 \\
204800 \ | \ 1200   & 0.935 & 0.945 &  & 0.880 & 0.892 &  & 0.768 & 0.784 &  & 3.49 \\
1024000\ | \ 3600  & 0.942 & 0.948 &  & 0.887 & 0.897 &  & 0.769 & 0.788 &  & 1.56 \\
\bottomrule
\end{tabular}%
}
\label{tab:obm_coverage_cleaned}
\end{table}

Code to reproduce experiments is given in \url{https://anonymous.4open.science/r/markov_lsa_normal_approximation-B85A}. The experiments were conducted on Google Colab using a virtual machine with the following specifications:
\begin{itemize}
    \item CPU: Intel(R) Xeon(R) CPU @ 2.20GHz, 8 virtual cores (1 socket, 4 physical cores, 2 threads per core)
    \item Memory: ~53.5 GB RAM
    \item No GPU was used
\end{itemize}


\section{Probability inequalities }
\label{appendix:Markov_technical}
Denote by $\Phi$ the c.d.f. of a standard Gaussian random variable and set
$$
\kolmogorov{X} = \sup_{x \in \mathbb{R}}|\PP(X \le x) - \Phi(x)|. 
$$
\begin{proposition}
\label{prop:nonlinearapprox}
For any random variables $X, Y$, and any $p \geq 1$,
\begin{equation}
\label{eq:shao_zhang_bound}
\kolmogorov{X+Y} \leq \kolmogorov{X} + 2 \PE^{1/(p+1)}[|Y|^p]\eqsp. 
\end{equation}  
\end{proposition}
\begin{proof}
Let $t \geq 0$. By Markov's inequality
\begin{multline}
\PP(X + Y \le x) \le \PP(X + Y \le x, |Y| \le t) + \frac{1}{ t^{p}} \PE [|Y|^p] \\\leq  \PP(X \leq x + t) - \Phi(x+t) + \Phi(x+t) + \frac{1}{ t^{p}} \PE [|Y|^p] \\ 
    \le \Phi(x) + \sup_{x \in \rset}|\PP(X \le x) - \Phi(x)| + \frac{t}{\sqrt 2\pi} + \frac{1}{ t^{p}} \PE [|Y|^p].  
\end{multline}   
Choosing $t = \PE^{1/(p+1)}[|Y|^p]$ we obtain
\begin{equation}
\sup_{x \in \rset}(\PP(X + Y \le x) - \Phi(x)) \le \sup_{x \in \rset}|\PP(X \le x) - \Phi(x)| + 2 \PE^{1/(p+1)}[|Y|^p]
\end{equation}
Similarly, we may estimate $\sup_{x \in \rset}(\Phi(x) - \PP(X + Y \le x) )$. Hence, \eqref{eq:shao_zhang_bound} holds.
\end{proof}
\begin{remark}
The result similar to \Cref{prop:nonlinearapprox} was previously obtained in \cite[Lemma 1]{ElMachkouriOuchti2007}. It states that for random variables $X$ and $Y$ and any $p \geq 1$, it holds that 
$$
\kolmogorov{X+Y} \leq 2 \kolmogorov{X} + 3 \|\PE[|Y|^{2p}|X] \|_1^{1/(2p+1)}.
$$
\end{remark}

Let $X = (X_1, \ldots, X_n$) be a sequence of real valued random variables which are square integrable and satisfy $\PE[X_i|\mathcal{F}_{i-1}] = 0$, a.s. for $1 \le i \le n$, where $\mathcal F_i = \sigma(X_1, \ldots, X_i)$. 
We denote the class of such sequences of length $n$ by $\mart_n$. Denote 
\begin{align}
\sigma_j^2 &= \PE[X_j^2|\mathcal F_{j-1}], \quad \hat{\sigma}_j^2 = \PE[X_j^2], \\
s_n^2 &= \sum_{j=1}^n \hat{\sigma}_j^2, \\
\|X\|_p &= \max_{1\le j\le n} \|X_j\|_p, \quad 1 \le p \leq \infty\eqsp,
\end{align}
and, with $1 \leq k \leq n$, we have
$$
V_k^2 = \sum_{j=1}^k \sigma_j^2/s_n^2, \quad S_k = \sum_{j=1}^k X_j\eqsp.
$$
\begin{theorem}[Theorem 2 in \cite{bolthausen1982martingale}] Let $0 < \boundconstmart < \infty$. There exists a constant $0 < L(\boundconstmart) < \infty$ depending only on $\boundconstmart$, such that for all $X \in \mart_n, n \geq 2$, satisfying
$$
\|X\|_\infty \leq \boundconstmart, \quad V_n^2 = 1 \text{ a.s.}\eqsp,
$$
the following bound holds
$$
\kolmogorov{S_n/s_n} \le L(\boundconstmart) n \log n/s_n^3.
$$
\end{theorem}

The constant $L(\boundconstmart)$ can be quantified following the work of Adrian R\"ollin, see \cite{ROLLIN2018171}.

\begin{lemma} 
\label{lem:martingale_berry_esseen}
Let $0 < \boundconstmart < \infty$. There exists a constant $0 < L(\boundconstmart) < \infty$ depending only on $\boundconstmart$, such that for all $X \in \mart_n, n \geq 2$, satisfying
$\|X\|_\infty \le \boundconstmart$,
the following bound holds for any $p \geq 1$
\begin{multline}
\kolmogorov{S_n/s_n} \leq \frac{L(\boundconstmart) (2n+1) \log(2n+1)}{s_n^3} \\+ C_1 \sqrt{p} s_n^{-\frac{2p}{2p+1}} \big(\PE|\sum\nolimits_{i=1}^n \sigma_i^2 - s_n^2|^p\big)^{1/(2p+1)} + C_2 s_n^{-\frac{2p}{2p+1}} 
 p  \boundconstmart^{2p/(2p+1)}.
\end{multline}
where we have defined the constants $C_1 = 2\sqrt{2} \bConst{\sf{Rm}, 1}$, $C_2 = 4( \bConst{\sf{Rm}, 1} +  \bConst{\sf{Rm}, 2})$. 
\end{lemma}
\begin{proof} We adapt the arguments from \cite{fan2019exact} based on  \cite{bolthausen1982martingale}. We specify the constants and dependence on $p$ since we need to have possibility to take $p$ of logarithmic order with respect to $n$. Consider the following stopping time $\tau = \sup \{ 0\le k \le n: V_k^2 \le 1\}$. Denote 
$$
r = \lfloor (s_{\tau}^2 - \sum_{j=1}^{\tau} \sigma_j^2)/\boundconstmart^2 \rfloor\eqsp.
$$
Note that $r \le n$. Let $N = 2n+1$. Construct the following sequence $X^\prime = (X_1^\prime, \ldots , X_{N}^\prime)$:
$$
X_i^\prime = X_i, i \le \tau; \quad X_i^\prime = \boundconstmart \eta_i, \tau + 1 \le i \le \tau + r; \quad X_{i}^\prime = (s_{n}^2 - \sum_{i=1}^\tau \sigma_i^2 - r \boundconstmart^2)^{1/2} \eta_i, i = \tau+r+1;
$$
and $X_i^\prime = 0, i \geq \tau + r + 2$. Here $\eta_i$ are Rademacher random variables, which are independent of all other r.v.'s. Then $X^\prime$ is a vector of martingale increments w.r.t. an extended filtration $\mathcal{F}^{\prime}_{i} = \sigma(X_j,\eta_j, j \leq i)$, and, by construction 
$$
\sum_{i=1}^N \PE[(X_i^\prime)^2|\mathcal{F}^{\prime}_{i-1}] = \sum_{i=1}^\tau \sigma_i^2 + r \boundconstmart^2 + s_n^2 -  \sum_{i=1}^\tau \sigma_i^2 - r \boundconstmart^2 = s_n^2. 
$$
Applying \Cref{prop:nonlinearapprox} with $X = S_n/s_n$ and $Y = S_N^\prime/s_n$, we get 
\begin{equation}
\label{eq:S_n_S_N_prime_kolmogorov}
\kolmogorov{S_n/s_n} \leq \kolmogorov{S_N^\prime/s_n} + 2 s_n^{-\frac{2p}{2p+1}} (\PE[|S_n - S_N^\prime|^{2p}])^{1/(2p+1)}. 
\end{equation}
We control the term $\kolmogorov{S_N^\prime/s_n}$ with \cite[Theorem 2]{bolthausen1982martingale}:
\[
\kolmogorov{S_N^\prime/s_n} \leq \frac{L(\boundconstmart) N \log N}{s_n^3}\eqsp.
\]
In order to control the second term in the right-hand side of \eqref{eq:S_n_S_N_prime_kolmogorov}, we notice that
$$
S_n - S_N^\prime = \sum_{i \geq \tau+1}(X_i - X_i^\prime) = \sum_{i=1}^{N} \indi{\tau \leq i-1} (X_i - X_i^\prime)\eqsp.
$$
Since $\tau$ is a stopping time, we can condition on it and get by the Rosenthal's inequality 
$$
\PE[|S_n - S_N^\prime|^{2p}] \le \bConst{\sf{Rm}, 1}^{2p} p^p \PE[|\sum_{i = \tau+1}^N \mathbb E[(X_i - X_i^\prime)^2|\mathcal{F}^{\prime}_{i-1}]|^p] + \bConst{\sf{Rm}, 2}^{2p} p^{2p} \mathbb E[\max_{\tau+1 \le i \le N} |X_i - X_i^\prime|^{2p}]
$$
It is easy to see that 
$$
\sum_{i = \tau+1}^N \mathbb E[(X_i - X_i^\prime)^2|\mathcal{F}^{\prime}_{i-1}] = \sum_{i = \tau+1}^n \mathbb E[X_i^2|\mathcal{F}^{\prime}_{i-1}] + \sum_{i = \tau+1}^N \mathbb E[(X_i^\prime)^2|\mathcal{F}^{\prime}_{i-1}] = \sum_{i=1}^n \sigma_i^2 + s_n^2 - 2 \sum_{i=1}^\tau \sigma_i^2.
$$
Note that 
$$
s_n^2 - \boundconstmart^2 \leq \sum_{i=1}^\tau \sigma_i^2 \le s_n^2\eqsp.
$$
Hence, it holds that
$$
\sum_{i = \tau+1}^N \mathbb E[(X_i - X_i^\prime)^2|\mathcal{F}^{\prime}_{i-1}] \leq \sum_{i=1}^n \sigma_i^2  - s_n^2 + 2 \boundconstmart^2.
$$
Finally,
$$
\PE[|S_n - S_N^\prime|^{2p}] \le 2^{p-1} \bConst{\sf{Rm}, 1}^{2p} p^p \bigl(\PE|\sum_{i=1}^n \sigma_i^2  - s_n^2|^p + 2^p \boundconstmart^{2p}\bigr) + \bConst{\sf{Rm}, 2}^{2p} (2p)^{2p}  \boundconstmart^{2p})^{1/(2p+1)}.
$$
Substituting the above bound into \eqref{eq:S_n_S_N_prime_kolmogorov}, we obtain 
\begin{multline}
\kolmogorov{S_n/s_n} \le  \kolmogorov{S_N^\prime/s_n} \\+ 2 s_n^{-\frac{2p}{2p+1}} (2^{p-1} \bConst{\sf{Rm}, 1}^{2p} p^p (\PE|\sum_{i=1}^n \sigma_i^2  - s_n^2|^p + 2^p\boundconstmart^{2p}) + \bConst{\sf{Rm}, 2}^{2p} (2p)^{2p}  \boundconstmart^{2p})^{1/(2p+1)}\eqsp,
\end{multline}
and the statement follows.
\end{proof}

We now give a lemma that allows us to specify the result of \Cref{lem:martingale_berry_esseen} to the martingale $M$ and its quadratic characteristic $\langle M \rangle_{n}$, given in \eqref{eq:M_quadratic_characteristic}. 
Define 
\begin{equation}
\label{eq:definition-sigma-infty}
\sigma_\infty(u) = \sup_{z \in \Zset} |u^{\top} \bA^{-1} \funnoisew(z)|  \, .
\end{equation}
\begin{lemma}
\label{lem: concentration of quadratic characteristic}
Assume \Cref{assum:UGE}. Then for any $u \in \sphere_{d-1}$, probability measure $\xi$ on $(\Zset,\Zsigma)$, and $p \geq 2$:
\begin{equation}
\PE_{\xi}^{1/p}[|u^\top \langle  M \rangle_n u - n \sigma^2(u) |^p] \leq  13^2\,  \sigma_\infty^2(u) p^{1/2} \taumix^{5/2} n^{1/2}\eqsp.
\end{equation}
\end{lemma}
\begin{proof}
Using the definition \eqref{eq:M_quadratic_characteristic}, we get:
\[
u^\top \langle  M \rangle_n u - n \sigma^2(u) = \sum_{\ell=1}^{n-1}  \{h(Z_\ell) - \pi(h)\} \quad \text{where} \quad  \quad h(z)= u^{\top} \bA^{-1} \tilde{\funnoisew}(z) \bA^{-1} u\eqsp.
\]
We then apply \cite[Lemma~9]{durmus2022finite}, showing that for all $t \geq 0$,
\[
\PP_{\xi}(|\sum_{\ell=1}^{n-1}  \{h(Z_\ell) - \pi(h)\} | \geq t) \leq 2 \exp\left( - \frac{t^2}{2 u_n^2} \right) \, , \text{where}
 \quad  u_n = 16 \sqrt{n} \sqrt{\taumix} \sup_{z \in \Zset}|h(z)|\eqsp.
\]
With the definition of $\tilde{\funnoisew}$ in \eqref{eq:M_quadratic_characteristic}, we get that
\begin{align}
\bigl| u^{\top} \bA^{-1} \widehat{\funnoisew}(z) \bigr| 
&\leq \sum_{k=0}^{\infty} \bigl| \MKQ^{k}\{v^{\top} \funnoisew \}(z)\bigr| \leq 2 \sigma_\infty(u) \sum_{k=0}^{\infty} \dobrush(\MKQ^{k}) = 2 \sigma_\infty(u) \sum_{k=0}^{\taumix-1} \sum_{\ell=0}^{\infty} \dobrush(\MKQ^{k + \taumix \ell}) \\
&\leq (8/3) \sigma_\infty(u) \taumix\eqsp.
\end{align}
Hence, we obtain that
\begin{equation}
\sup_{z \in \Zset}|h(z)| \leq (8/3)^2 \sigma^2_\infty(u) \taumix^2 \eqsp.
\end{equation}
We conclude by using \cite[Lemma~7]{durmus2022finite}.
\end{proof}

\begin{lemma}
\label{lem:weighted_quadratic_characteristic}
Assume \Cref{assum:UGE}. Then for any $u \in \sphere_{d-1}$, probability measure $\xi$ on $(\Zset,\Zsigma)$, and $p \geq 2$:
\begin{equation}
\PE_{\xi}^{1/p}[|u^\top \langle  M \rangle_{n} u - n\sigma_n^2(u) |^p] \leq 32 n^{1/2}p^{1/2} \mathcal{L}_{Q}^2\supconsteps^2\taumix^{5/2}\
\end{equation}
\end{lemma}
\begin{proof}
Using the definition \eqref{eq:M_quadratic_characteristic}, we get:
\[
u^\top \langle  M \rangle_n u - \sigma_n^2(u) = \sum_{\ell=1}^{n-2}  \{h_\ell(Z_\ell) - \pi(h_\ell)\} \quad \text{where} \quad  \quad h_\ell(z)= u^{\top} Q_{\ell+1}\tilde{\funnoisew}(z) Q_{\ell+1}^{\top} u\eqsp.
\]
Note that using \Cref{prop:Qell:bound}, for any $z, z' \in \Zset$ and any $\ell\in [1, \ldots, n-2]$ we have
\begin{align}
    |h_{\ell}(z) - h_{\ell}(z')|&\leq 2\mathcal{L}_{Q}^2\normop{\tilde{\funnoisew}(z)}\leq 4\mathcal{L}_{Q}^2\normop{\hat{\funnoisew}(z)}^2 \leq 8 \mathcal{L}_{Q}^2\supconsteps^2(\sum_{k=0}^{\infty} \dobrush(\MKQ^{k}))^2\\&\leq 8 \mathcal{L}_{Q}^2\supconsteps^2(\sum_{k=0}^{\infty} \sum_{r=0}^{\taumix-1} (1/4)^{\lceil (k\taumix + r)/\taumix\rceil})^{2} \leq 2(8/3)^2\mathcal{L}_{Q}^2\supconsteps^2\taumix^2
\end{align}
We then apply \cite{paulin_concentration_spectral}[Corollary 2.11], showing that for all $t \geq 0$,
\[
\PP_{\xi}(|\sum_{\ell=1}^{n-1}  \{h_{\ell}(Z_\ell) - \pi(h_{\ell})\} | \geq t) \leq 2 \exp\left( - \frac{t^2}{2u_n^2} \right) \, , \text{where}
 \quad  u_n = (64/3)\mathcal{L}_{Q}^2\supconsteps^2\taumix^{5/2}n^{1/2}
\]
We conclude by using \cite[Lemma~7]{durmus2022finite}.
\end{proof}

By combining \Cref{lem:martingale_berry_esseen}  and \Cref{lem:weighted_quadratic_characteristic}, we finally get the following result:

\begin{proposition}
\label{th:normal_approximation_martingales}
Assume \Cref{assum:UGE}. Then for any $u \in \sphere_{d-1}$ and $p \geq 1$:
\begin{multline}
\kolmogorov{\frac{u^\top M}{\sqrt{n} \sigma_n(u)}} \le
\frac{L(\boundconstmart) (2n+1) \log(2n+1)C_{\Sigma}^3}{n^{3/2}} \\+ 4\sqrt{2} C_1 \biggl(\mathcal{L}_{Q}\supconsteps C_{\Sigma}\biggr)^{2p/(2p+1)} \taumix^{5/4} p^{3/4} n^{-p/(2(2p+1))}  +  C_2  n^{-\frac{p}{2p+1}} \biggl(\frac{16\mathcal{L}_{Q}\supconsteps C_{\Sigma}}{3}\biggr)^{2p/(2p+1)}
 p  \taumix\eqsp,
\end{multline}
where $\boundconstmart= (16/3)\supconsteps\mathcal{L}_{Q} \taumix$, and  $C_1$, and $C_2$ are defined in  
\Cref{lem:martingale_berry_esseen}. Moreover, setting $p = \log{n}$, we obtain that 
\begin{equation}
\label{eq:normal_approximation_martingales_optimized}
\kolmogorov{\frac{u^\top M}{\sqrt{n} \sigma_n(u)}} \leq \frac{\ConstK[1] \log^{3/4}{n}}{n^{1/4}} + \frac{\ConstK[2] \log{n}}{n^{1/2}}\eqsp,
\end{equation}
where the constants $\ConstK[1]$ and $\ConstK[2]$ are given by 
\begin{equation}
\label{eq:kolmogorov_const_1_2}
\begin{split}
\ConstK[1] &= 4\sqrt{2} \rme^{1/8} C_{1} \taumix^{5/4} \mathcal{L}_{Q}\supconsteps C_{\Sigma} \eqsp, \\
\ConstK[2] &= \frac{16e^{1/4} \taumix C_{2}\mathcal{L}_{Q}\supconsteps C_{\Sigma}}{3} + 3 L(\boundconstmart)C_{\Sigma}^3 + 3 L(\boundconstmart) \log{3}C_{\Sigma}^3  \eqsp.
\end{split}
\end{equation}
\end{proposition}
\begin{proof}[Proof of \Cref{th:normal_approximation_martingales}]
To complete the proof we apply \Cref{lem:martingale_berry_esseen} with $s_n^2 = n \sigma_n^2(u)$ and we conclude using \Cref{lem: concentration of quadratic characteristic} and \Cref{lem:bound_sigma_n}.
\end{proof}

\subsection{Rosenthal and Burkholder inequalities}
We begin this section with a version of Rosenthal inequality (see the original paper \cite{Rosenthal1970} and the Pinelis version of the Rosenthal inequality \cite{pinelis_1994}). Let $f: \Zset \to \rset$ be a bounded function with $\|f\|_{\infty} < \infty$ and define
\begin{equation}
\label{eq:bar_S_n_defi}
\bar{S}_n = \sum_{k=0}^{n-1}\{f(Z_k) - \pi(f)\}\eqsp.
\end{equation}
We provide below a simplified version of the Rosenthal inequality, where the leading term with respect to number of summands $n$ in $\bar{S}_n$ is governed by the appropriate variance proxy. The result below is proven in \cite[Lemma~2]{durmus2023rosenthal}.
\begin{lemma}
\label{lem:auxiliary_rosenthal}
Assume \Cref{assum:UGE}. Then for any $p \geq 2$ and $f: \Zset \rightarrow \rset$ with $\|f\|_{\infty} < \infty$, it holds that
\begin{equation}
\PE_{\pi}^{1/p}\bigl[\bigr| \bar{S}_n \bigr|^p\bigr] \leq \ConstD_{\ref{lem:auxiliary_rosenthal},1} n^{1/2}\taumix^{1/2}p^{1/2} \|f\|_{\infty} + \ConstD_{\ref{lem:auxiliary_rosenthal},2} \taumix p \|f\|_{\infty}\eqsp,
\end{equation}
where $\ConstD_{\ref{lem:auxiliary_rosenthal},1} = (16/3) \bConst{\sf{Rm}, 1}$, $\ConstD_{\ref{lem:auxiliary_rosenthal},2} = 8 \bConst{\sf{Rm}, 2}$, and $\bConst{\sf{Rm}, 1}$, $\bConst{\sf{Rm}, 2}$ are the constants from martingale version of Rosenthal inequality \cite{pinelis_1994} defined in \Cref{tab:univ_constants}.
\end{lemma}

Based on the inequality above, we can also prove the weighted version of Burkholder's inequality:
\begin{lemma}
\label{lem:auxiliary_rosenthal_weighted} 
Assume \Cref{assum:UGE}. Then for any $p \geq 2$ and $f: \Zset \rightarrow \rset^d$ with $\|f\|_{\infty} < \infty$, any initial distribution $\xi$ on $(\Zset,\Zsigma)$, and any $u\in \sphere_{d-1}, A_i\in\rset^{d\times d}$, it holds that
\begin{equation}
\label{eq:rosenthal_weighted}
\begin{split}
\PE_{\xi}^{1/p}\bigl[\bigr| \sum_{k=1}^{n} u^{\top}A_k (f(Z_k) - \pi(f)) \bigr|^p\bigr] 
&\leq (16/3)\taumix p^{1/2}\|f\|_{\infty}(\sum_{k=2}^{n}\norm{A_k}^2)^{1/2}\\
&\qquad +(8/3) \taumix \bigl(\|A_1\| + \|A_{n}\| + \sum_{k=1}^{n-1}\|A_{k+1}-A_{k}\| \bigr) \|f\|_{\infty}\eqsp.
\end{split}
\end{equation}
\end{lemma}
\begin{proof}
Under \Cref{assum:UGE} the Poisson equation 
\[
g(z) - \MKQ g(z) = f(z) - \pi(f)
\]
has a unique solution for any bounded $f$, which is given by the formula 
\[
g(z) = \sum_{k=0}^{\infty}\{\MKQ^{k}f(z) - \pi(f)\}\eqsp.
\]
Thus, using \Cref{assum:UGE}, we obtain that $g(z)$ is also bounded with
\begin{align}
\| g(z)\| \leq \sum_{k=0}^{+\infty} \|\MKQ^{k}f(z) - \pi(f)\| \leq 2\|f\|_{\infty} \sum_{k=0}^{+\infty} (1/4)^{\lfloor k/\taumix \rfloor} \leq (8/3) \taumix \| f\|_{\infty}\eqsp.
\end{align}
Hence, we can represent
\begin{equation}
\label{eq:repr_weighted_stat}
\begin{split}
\sum_{k=1}^{n} u^\top A_k (f(Z_k) - \pi(f)) 
&= \underbrace{\sum_{k=2}^{n} u^\top A_k (g(Z_{k}) - \MKQ g(Z_{k-1}))}_{T_1} \\
&\quad \underbrace{+\sum_{k=1}^{n-1}u^\top (A_{k+1} - A_{k}) \MKQ g(Z_{k}) + u^\top A_{1} g(Z_1) -u^\top A_{n} \MKQ g(Z_{n})}_{T_2}\eqsp.
\end{split}
\end{equation}
The term $T_2$ can be controlled using Minkowski's inequality:
\begin{equation}
\label{eq:T_2_bound_rosenthal_remainder}
\PE_{\xi}^{1/p}[\bigl| T_{2}\bigr|^{p}] \leq (8/3) \taumix \bigl(\|A_1\| + \|A_{n}\| + \sum_{k=1}^{n-1}\|A_{k+1}-A_{k}\| \bigr) \|f\|_{\infty}\eqsp.
\end{equation}
Now we proceed with $T_1$. Since $\CPE{g(Z_{k}) - \MKQ g(Z_{k-1})}{\F_{k-1}} = 0$ a.s. and $|u^\top A_k(g(Z_k) - \MKQ g(Z_{k-1}))| \leq (16/3) \norm{A_k} \taumix\|f\|_{\infty}$, we get, using the Azuma-Hoeffding inequality \cite[Corollary 3.9]{vanHandel2016}, that 
\begin{equation}
\PP_{\xi}[|T_1| \geq t] \leq 2\exp\biggl\{-\frac{2t^2}{(16/3)^2\taumix^2\|f\|_{\infty}^2\sum_{k=2}^{n}\norm{A_k}^2}\biggr\}\eqsp.
\end{equation}

Hence, applying \Cref{lem:bound_subgaussian} we get
\begin{equation}
\PE^{1/p}_{\xi}[|T_1|^p] \leq (16/3) p^{1/2}\taumix \|f\|_{\infty}(\sum_{k=2}^{n}\norm{A_k}^2)^{1/2}\eqsp.
\end{equation}
\end{proof}

Now we provide a standard moment bounds for sub-Gaussian random variables. Proof of this result can be found in \cite[Lemma~7]{durmus2022finite}.
\begin{lemma}
\label{lem:bound_subgaussian}
Let $X$ be a random variable satisfying $\PP(|X| \geq t) \leq 2 \exp(-t^2/(2\sigma^2))$ for any $t \geq 0$ and some $\sigma^2 >0$. Then, for any $p \geq 2$, it holds that $ \PE[|X|^p] \leq 2 p^{p/2}\sigma^p$.
\end{lemma}
The lemma below is a simple technical statement used to switch between $p$-th moment bounds and high-probability bounds.
\begin{lemma}
\label{lem:markov_inequality}
Fix $\delta \in (0,1/\rme^2)$ and let $Y$ be a positive random variable, such that 
\[
\PE^{1/p}[Y^{p}] \leq p^\upsilon C_{1} 
\]
for any $2 \leq p \leq \log{(1/\delta)}$. Then it holds with probability at least $1-\delta$, that 
\begin{equation}
\label{eq:markov_ineqality_deviation}
Y \leq \rme C_{1}(\log{(1/\delta)})^\upsilon \eqsp.
\end{equation}
\end{lemma}
\begin{proof}
Applying Markov's inequality, for any $t \geq 0$ we get that 
\begin{align}
\PP(Y \geq t) \leq \frac{\PE[Y^{p}]}{t^{p}} \leq \frac{(C_{1} p^\upsilon)^{p}}{t^{p}}\eqsp.
\end{align}
Now we set $p = \log{(1/\delta)}$, $t =  \rme C_{1} (\log{(1/\delta)})^\upsilon$, and aim to check that 
\[
\frac{(C_{1} (\log{(1/\delta)})^\upsilon)^{\log{(1/\delta)}}}{( \rme C_{1} (\log{(1/\delta)})^\upsilon)^{ \log{(1/\delta)}}} \leq \delta\eqsp.
\]
Taking logarithms from both sides, the latter inequality is equivalent to 
\[
-\log{(1/\delta)} \leq \log{\delta}\eqsp,
\]
which turns into exact equality.
\end{proof}

\subsection{McDiarmid inequality}
Now we provide a result, which is a version of McDiarmid's inequality for uniformly geometrically ergodic Markov chains. Before we proceed with this result, we state the following definition from \cite[Chapter~23]{douc:moulines:priouret:soulier:2018}:
\begin{definition}
\label{def:bounded-diff-property}
A measurable function $f: \Zset^{n} \to \rset$ is said to have the bounded difference property if there exist non-negative constants $\bbeta = (\beta_0,\ldots,\beta_{n-1})$, such that for all $(z_0,\ldots,z_{n-1}), (z_0^{\prime},\ldots,z_{n-1}^{\prime}) \in \Zset^{n}$ it holds that
\begin{equation}
\label{eq:bounded-diff-property}
\bigl| f(z_0,\ldots,z_{n-1}) - f(z_0^{\prime},\ldots,z_{n-1}^{\prime}) \bigr| \leq \sum_{i=0}^{n-1}\beta_{i}\indi{z_i \neq z_i^{\prime}}\eqsp.
\end{equation}
\end{definition}
For functions $f$ satisfying the bounded differences property \eqref{eq:bounded-diff-property}, the following version of McDiarmid's inequality holds:
\begin{lemma}[Theorem~23.2.2 in \cite{douc:moulines:priouret:soulier:2018}]
\label{lem:bounded_differences_norms_markovian}
Let $f: \Zset^{n} \to \rset$ be a function, satisfying the bounded difference property with vector $\bbeta = (\beta_0,\ldots,\beta_{n-1})$, such that $\beta_0 \geq \beta_1 \geq \ldots \geq \beta_{n-1}$, and $(Z_{k})_{k \in \nset}$ is a Markov chain with the Markov kernel $\MKQ$ satisfying \Cref{assum:UGE}. Then, for any initial distribution $\xi$ on $(\Zset,\Zsigma)$, and any $t \geq 0$ it holds 
\begin{equation}
\label{eq:mc-diamid}
\PP_{\xi}\bigl(\bigl|f(Z_0,\ldots,Z_{n-1}) - \PE_{\xi}[f(Z_0,\ldots,Z_{n-1})] \bigr| \geq t\bigr) \leq 2\exp\biggl\{-\frac{t^2}{(32/9)\taumix^2 \sum_{i=0}^{n-1}\beta_i^2}\biggr\}\eqsp.
\end{equation}
\end{lemma}
\begin{proof}
The proof directly follows the one of Theorem~23.2.2 in \cite{douc:moulines:priouret:soulier:2018}, together with the appropriate bound on the Dobrushin coefficient implied by \Cref{assum:UGE}. We also note that in \Cref{assum:UGE} we have assumed that $(\Zset,\Zsigma)$ is Polish space.
\end{proof}

\section{Auxiliary results on the sequences of step sizes $\{\alpha_k\}_{k \in \nset}$ under \Cref{assum:step-size}}
\label{appendix:algebra_calculations}
We conclude this section with the auxiliary results on the sequences of step sizes $\{\alpha_k\}_{k \in \nset}$ under \Cref{assum:step-size}. 

\begin{lemma}
    \label{lem:bounds_on_sum_step_sizes}
    Assume \Cref{assum:step-size}. Then the following bounds holds:
    \begin{enumerate}
        \item \label{eq:sum_alpha_k}
        \begin{equation}
        \sum_{i=1}^{k}\alpha_i \leq\frac{c_0}{1-\gamma}((k+k_0)^{1-\gamma}-k_0^{1-\gamma})
        \end{equation}
        \item \label{eq:sum_alpha_k_p} for any $p\geq 2$
        \begin{equation}
        \sum_{i=1}^{k}\alpha_i^p \leq\frac{c_0^p}{p\gamma-1}\eqsp,
        \end{equation}
        \item \label{eq:simple_bound_sum_alpha_k}
        for any $m\in\{0, \ldots, k\}$
        \begin{equation}
        \sum_{i=m+1}^k\alpha_i \geq  \frac{c_0}{2(1-\gamma)}((k+k_0)^{1-\gamma}-(m+k_0)^{1-\gamma})\eqsp,
        \end{equation}
    \end{enumerate}
\end{lemma}
\begin{proof}
    To proof \ref{eq:sum_alpha_k}, note that
    \begin{equation}
        \sum_{i=1}^{k}\alpha_i \leq  c_{0} \int_{k_0}^{k+k_0}\frac{\rmd x}{x^{\gamma}} \leq \frac{c_0}{1-\gamma}((k+k_0)^{1-\gamma}-k_0^{1-\gamma})\eqsp,
    \end{equation}
    To proof \ref{eq:sum_alpha_k_p}, note that 
    \begin{equation}
        \sum_{i=1}^{k}\alpha_i^p \leq  c_{0}^p \int_{1}^{+\infty}\frac{\rmd x}{x^{p\gamma}} \leq \frac{c_0^p}{p\gamma-1}\eqsp,
    \end{equation}
    To proof \ref{eq:simple_bound_sum_alpha_k}, note that for any $i\geq 1$ we have $2(i+k_0)^{-\gamma} \geq (i+k_0-1)^{-\gamma}$. Hence,
    \begin{equation}
        \sum_{i=m+1}^k\alpha_i \geq \frac{1}{2}\sum_{i=m}^{k-1}\alpha_i \geq \frac{c_0}{2}\int_{m+k_0}^{k+k_0}\frac{\rmd x}{x^{\gamma}} = 
        \frac{c_0}{2(1-\gamma)}((k+k_0)^{1-\gamma}-(m+k_0)^{1-\gamma})\eqsp.
        \end{equation}
\end{proof}

\begin{lemma}[Lemma~24 in \cite{durmus2021stability}]
\label{lem:summ_alpha_k}
Let $b > 0$ and $\{\alpha_k\}_{k \geq 0}$ be a non-increasing sequence such that $\alpha_1 \leq 1/b$. Then
\[
\sum_{j=1}^{k} \alpha_j \prod_{l=j+1}^{k} (1 - \alpha_l b) = \frac{1}{b} \left\{1  - \prod_{l=1}^{k} (1 - \alpha_l b) \right\}
\]
\end{lemma}
\begin{proof}
The proof of this statement is given in \cite{durmus2021stability}.
\end{proof}

\begin{lemma}
\label{lem:prod_alpha_k}
Let $b > 0$ and $\alpha_k = \frac{c_0}{(k_0+k)^{\gamma}}$ be a non-increasing sequence such that $c_0 \leq 1/b$ and $k_0\geq \{\frac{\gamma}{2b c_0}\}^{1/(1-\gamma)}$. Then it holds 
\[
 \alpha_j \prod_{l=j+1}^{k} (1 - \alpha_l b) \leq \alpha_k
\]
\end{lemma}
\begin{proof}
    Note that 
    \begin{equation}
    \alpha_j \prod_{l=j+1}^{k} (1 - \alpha_l b) = \alpha_k \prod_{l=j+1}^{k}\frac{\alpha_{l-1}}{\alpha_l} (1 - \alpha_l b)\eqsp.
     \end{equation}
     It remains to note that,
     \begin{equation}
     \frac{\alpha_{l-1}}{\alpha_l} (1 - \alpha_l b) = \biggl\{\frac{k_0+l}{k_0+l-1}\biggr\}^{\gamma}- \frac{bc_0}{(k_0+l-1)^\gamma} \leq 1 + \frac{\gamma}{k_0 +l-1} -\frac{bc_0}{(k_0+l-1)^\gamma} \leq 1-(b/2)\alpha_{l-1}\eqsp,
     \end{equation}
     where the last inequality holds since $k_0\geq \{\frac{2\gamma}{b c_0}\}^{1/(1-\gamma)}$.
\end{proof}

\begin{lemma}
\label{lem:bound_ratio_step_size}
Let $b > 0$, and let $\alpha_{k} = c_{0}/(k+k_0)^{\gamma}$, $\gamma \in (0;1)$, such that $c_{0} \leq 1/b$ and $k_0 \geq \{\frac{2\gamma}{r c_0}\}^{1/{1-\gamma}}$ with some constant $r > 0$. Then it holds that 
\begin{equation}
\frac{\alpha_k}{\alpha_{k+1}} \leq 1+ r \alpha_{k+1}
\end{equation}
\end{lemma}
\begin{proof}
    Note that 
    \begin{equation}
        \frac{\alpha_k}{\alpha_{k+1}} \leq \biggl(1 + \frac{1}{k+k_0}\biggr)^{\gamma} \leq 1+\frac{2\gamma}{k_0+k+1} \leq 1+\frac{r c_0}{(k_0+k+1)^\gamma}\eqsp,
    \end{equation}
    where the last inequality holds since $k_0 \geq \{\frac{2\gamma}{r c_0}\}^{1/{1-\gamma}}$.
\end{proof}

\begin{lemma}
\label{lem:sum_alpha_k_squared}
Let $b > 0$, and let $\alpha_{k} = c_{0}/(k+k_0)^{\gamma}$, $\gamma \in (0;1)$, such that $c_{0}b \leq 1/2$ and $k_0 \geq \{\frac{8\gamma}{bc_0}\}^{1/{1-\gamma}}$. Then for any $q \in (1;3]$, it holds that 
\begin{equation}
\label{eq:sum_squares_bound}
\sum_{j=1}^{k}\alpha_{j}^{q}\prod_{\ell=j+1}^{k}(1-\alpha_{\ell} b) \leq \frac{4}{b}\alpha_{k}^{q-1}\eqsp.
\end{equation}
\end{lemma}
\begin{proof}
Using \Cref{lem:bound_ratio_step_size}, we obtain that  
\begin{align}
\sum_{j=1}^{k} \alpha_{j}^{q}\prod_{\ell=j+1}^{k}(1-\alpha_{\ell} b) 
&= \alpha_{k}^{q-1} \sum_{j=1}^{k} \alpha_{j}\prod_{\ell=j+1}^{k} \biggl(\frac{\alpha_{\ell-1}}{\alpha_{\ell}}\biggr)^{q-1} (1-\alpha_{\ell} b) \\
&\leq \alpha_{k}^{q-1} \sum_{j=1}^{k} \alpha_{j}\prod_{\ell=j+1}^{k}\left(1 + r\alpha_{\ell}\right)^{q-1}(1-\alpha_{\ell} b)\eqsp.
\end{align}
We set $r = \frac{b}{2(q-1)}$.
If $q\in(1,2)$ then we use Bernoulli's inequality and obtain 
\begin{align}
\sum_{j=1}^{k}\alpha_{j}^{q}\prod_{\ell=j+1}^{k}(1-\alpha_{\ell} b) 
&\leq \alpha_{k}^{q-1} \sum_{j=1}^{k} \alpha_{j}\prod_{\ell=j+1}^{k}\left(1 + b\alpha_{\ell}/2\right)(1-\alpha_{\ell} b) \\
&\leq \alpha_{k}^{q-1} \sum_{j=1}^{k} \alpha_{j} \prod_{\ell=j+1}^{k} (1 - b \alpha_{\ell} / 2) \overset{(a)}{\leq} \frac{2}{b}\alpha_{k}^{q-1}\eqsp,
\end{align}
where in (a) we used \Cref{lem:summ_alpha_k}. If $q \in [2,3]$, using that $1 - \alpha_{\ell} b \leq \bigl(1 - b/(q-1)\alpha_{\ell}\bigr)^{q-1}$, we obtain
\begin{align}
    \sum_{j=1}^{k}\alpha_{j}^{q}\prod_{\ell=j+1}^{k}(1-\alpha_{\ell} b) &\leq \alpha_{k}^{q-1} \sum_{j=1}^{k} \alpha_{j}\prod_{\ell=j+1}^{k}\left(1 + \frac{b}{2(q-1)}\alpha_\ell\right)^{q-1}\biggl(1-\frac{b}{q-1}\alpha_{\ell}\biggr)^{q-1} \\
    &\leq \alpha_{k}^{q-1} \sum_{j=1}^{k} \alpha_{j}\prod_{\ell=j+1}^{k}\left(1 - \frac{b}{2(q-1)}\alpha_\ell\right)^{q-1} \\
    &\leq \alpha_{k}^{q-1} \sum_{j=1}^{k} \alpha_{j}\prod_{\ell=j+1}^{k}\left(1 - \frac{b}{2(q-1)}\alpha_\ell\right) \leq \frac{2(q-1)}{b} \alpha_{k}^{q-1}\eqsp,
\end{align}
and the statement follows.
\end{proof}

We conclude with a technical statement on the coefficients $\{\alpha_{j}\}_{j \in \nset}$ under \Cref{assum:step-size}.

\begin{lemma}
\label{lem:sum_alpha_k_squared_new}
Let $b > 0$, and let $\alpha_{\ell} = c_{0}/\{\ell+k_0\}^{\gamma}$, $\gamma \in (1/2;1)$, such that $c_{0} \leq 1/b$. Then, for any $k_0$ satisfying  
\begin{equation}
\label{eq:condition-on-n-squared}
k_0^{1-\gamma} \geq
    \frac{2}{c_{0}b} \biggl(\log\{\frac{c_0}{b(2\gamma-1)2^\gamma} \} + \gamma \log \{k_0 \} \biggr)\eqsp,
\end{equation}
any $s \in (1;2]$, $q \in (0;1]$, and $k \in \nset$, it holds that 
\begin{equation}
\label{eq:sum_alpha_k_squared_new}
\sum_{j=1}^{k}\alpha_{j}^{s} \bigl(\sum_{\ell=j+1}^{k} \alpha_{\ell}^2 \bigr)^{q} \prod_{\ell=j+1}^{k}(1-\alpha_{\ell} b) \leq C(s, q, b) \alpha_{k}^{s+q-1} / b^{1+q} \eqsp,  
\end{equation}
where  $C(s, q, b) = 12\cdot 3^{\gamma(s+q-1)}\biggl(\frac{4(s-1)\gamma}{bc_0}\biggr)^{2\gamma(s-1)/(1-\gamma)}$.
\end{lemma}
\begin{proof}
We denote $t = \lfloor (k+k_0)/2 \rfloor$ and split the sum in \eqref{eq:sum_alpha_k_squared_new} into two parts:
\begin{align}
&\sum_{j=1}^{k} \biggl(\frac{c_0}{(j+k_0)^\gamma}\biggr)^s \biggl(\sum_{\ell=j+1}^{k}  \biggl(\frac{c_0}{(l+k_0)^\gamma}\biggr)^2\biggr)^{q} \prod_{\ell=j+1}^{k}(1- \frac{c_0}{(l+k_0)^\gamma}b)\\& = 
\sum_{j=1+k_0}^{k+k_0} \biggl(\frac{c_0}{j^\gamma}\biggr)^s \biggl(\sum_{\ell=j+1}^{k+k_0}  \biggl(\frac{c_0}{l^\gamma}\biggr)^2\biggr)^{q} \prod_{\ell=j+1}^{k+k_0}(1- \frac{c_0}{l^\gamma}b)\\
&\leq \underbrace{\sum_{j=1}^{t} \biggl(\frac{c_0}{j^\gamma}\biggr)^s \biggl(\sum_{\ell=j+1}^{k+k_0}  \biggl(\frac{c_0}{l^\gamma}\biggr)^2\biggr)^{q} \prod_{\ell=j+1}^{k+k_0}(1- \frac{c_0}{l^\gamma}b)}_{T_1} \\ 
&\qquad + \underbrace{\sum_{j=t+1}^{k+k_0} \biggl(\frac{c_0}{j^\gamma}\biggr)^s \biggl(\sum_{\ell=j+1}^{k+k_0}  \biggl(\frac{c_0}{l^\gamma}\biggr)^2\biggr)^{q} \prod_{\ell=j+1}^{k+k_0}(1- \frac{c_0}{l^\gamma}b)}_{T_2}\eqsp.
\end{align}
For the term $T_2$ we notice that $\frac{c_0}{j^\gamma} \leq 2^{\gamma} \alpha_{k}$ for $j \in \{t+1,\ldots,k+k_0\}$, hence, we can upper bound $T_2$ as follows:
\begin{align}
T_2 
&\leq 2^{\gamma(s+2q)} \sum_{j=t+1}^{k+k_0}\alpha_{k}^{s+2q}(k+k_0-j)^{q} \exp\bigl\{ -b (k+k_0-j) \alpha_{k}\bigr\} \\
&\leq 2^{\gamma(s+2q)+1} \alpha_{k}^{s+2q} \int_{0}^{+\infty}x^{q} \exp\bigl\{ -b  \alpha_{k} x \bigr\}\rmd x \\
&\leq 2^{\gamma(s+2q)+1} \alpha_{k}^{s+q-1} / b^{1+q} \int_{0}^{+\infty}u^{q} \exp\bigl\{ -u \bigr\}\rmd u \\
& \leq 2^{\gamma(s+2q)+1} \alpha_{k}^{s+q-1} / b^{1+q}\eqsp,
\end{align}
where we have used that $\Gamma(q+1) \leq 1$ for $q\in (0,1)$. Now it remains to provide an upper bound for $T_1$. Since $k_0$ satisfies \eqref{eq:condition-on-n-squared}, we get that, for $j \leq t$,  
\begin{equation}
\label{eq:sum_alpha_ell_squared_aux_bound}
\biggl(\sum_{\ell=j+1}^{k+k_0}  \biggl(\frac{c_0}{\ell^\gamma}\biggr)^2\biggr)^{q} \prod_{\ell=j+1}^{k+k_0}(1- \frac{c_0}{\ell^\gamma}b) \leq (\frac{c_0}{bt^\gamma})^{q} \prod_{\ell = j+1}^{t} (1 - 
\frac{c_0}{\ell^{\gamma}}b)\eqsp.
\end{equation}
Then, applying \eqref{eq:sum_alpha_ell_squared_aux_bound}, we get that 
\begin{align}
T_{1} &\leq (\frac{c_0}{bt^\gamma})^{q} \sum_{j=1}^{t} \biggl(\frac{c_0}{j^\gamma}\biggr)^s \prod_{\ell=j+1}^{t}(1- \frac{c_0}{\ell^\gamma} b)\overset{(a)}{\leq} (4/b^{q+1})\biggl(\frac{4(s-1)\gamma}{bc_0}\biggr)^{2\gamma(s-1)/(1-\gamma)} \biggl(\frac{c_0}{t^\gamma}\biggr)^{q+s-1} \\& \leq (4/b^{1+q})\biggl(\frac{4(s-1)\gamma}{bc_0}\biggr)^{2\gamma(s-1)/(1-\gamma)}3^{\gamma(s+q-1)}\alpha_k^{q+s-1}\eqsp,
\end{align}
where in (a) we have additionally used \Cref{lem:sum_alpha_k_squared_without_k_0} and in (b) we used  $\frac{c_0}{t^{\gamma}}\leq 3^{\gamma}\alpha_{k}$. It remains to check the relation \eqref{eq:sum_alpha_ell_squared_aux_bound}, that is, it is enough to obtain an upper bound
\begin{equation}
\label{eq:sum_alpha_ell_squared_aux_bound_1}
\bigl(\sum_{\ell=2}^{k+k_0} \frac{c_0^2}{\ell^{2\gamma}}\bigr)^{q} \prod_{\ell=t+1}^{k+k_0}(1 - \frac{c_0}{\ell^\gamma}b) \leq (\frac{c_0}{bt^\gamma})^{q}\eqsp.
\end{equation}
Since 
\[
\sum_{\ell=2}^{k+k_0}\frac{c_0^2}{\ell^{2\gamma}} \leq c_{0}^2 \int_{2}^{k+k_0}\frac{\rmd x}{x^{2\gamma}} \leq \frac{c_{0}^2}{2\gamma-1}\eqsp.
\]
Hence, \eqref{eq:sum_alpha_ell_squared_aux_bound_1} will follow from 
\[
\frac{c_{0}^2}{2\gamma-1} \exp\biggl\{ - \frac{(k+k_0-t) \alpha_{k}b}{q}\biggr\} \leq \frac{c_0}{t^\gamma b}\eqsp,
\]
which is guaranteed by relations \eqref{eq:condition-on-n-squared}.
\end{proof}

\begin{lemma}
    \label{lem:sum_alpha_k_squared_without_k_0}
    Let $b > 0$, and let $\alpha_{k} = c_{0}/(k)^{\gamma}$, $\gamma \in (1/2;1)$, such that $c_{0}b \leq 1/2$. Then for any $q \in [1;2]$, it holds that 
\begin{equation}
\label{eq:sum_squares_bound_without_k_0}
\sum_{j=1}^{k}\alpha_{j}^{q}\prod_{\ell=j+1}^{k}(1-\alpha_{\ell} b) \leq C_q\alpha_{k}^{q-1}\eqsp,
\end{equation}
where $C_q = \frac{4}{b}\biggl(\frac{4(q-1)\gamma}{bc_0}\biggr)^{2\gamma(q-1)/(1-\gamma)}$.
\end{lemma}
\begin{proof}
    Note that 
\begin{align}
    &\sum_{j=1}^{k}\alpha_{j}^{q}\prod_{\ell=j+1}^{k}(1-\alpha_{\ell} b) = \alpha_k^{q-1}\sum_{j=1}^{k}\alpha_j\prod_{\ell=j+1}^{k}\biggl(\frac{\alpha_{\ell-1}}{\alpha_\ell}\biggr)^{q-1}(1-\alpha_{\ell} b)\\&\quad \quad \leq \alpha_k^{q-1}\sum_{j=1}^{k}\alpha_j \prod_{\ell=j+1}^{k}\biggl(1 + \frac{1}{\ell-1}\biggr)^{\gamma(q-1)}(1-\alpha_{\ell} b) \\ & \quad \quad \quad \quad \leq \alpha_k^{q-1}\sum_{j=1}^{k}\alpha_j \exp\biggl\{\sum_{\ell=j+1}^k\biggl\{ \frac{\gamma(q-1)}{\ell-1}-\frac{bc_0}{\ell^{\gamma}}\biggr\}\biggr\}\eqsp.
\end{align}
Define $\ell_0 = \lceil\bigl(\frac{4(q-1)\gamma}{bc_0}\bigr)^{1/(1-\gamma)}\rceil \geq 2$, then for any $l > l_0$ we have $\frac{\gamma(q-1)}{\ell-1}-\frac{bc_0}{\ell^{\gamma}} \leq -\frac{bc_0}{2\ell^\gamma}$. Hence, we get 
\begin{align}
    \sum_{j=1}^{k}\alpha_{j}^{q}\prod_{\ell=j+1}^{k}(1-\alpha_{\ell} b) \leq \alpha_k^{q-1}\sum_{j=1}^{k}\alpha_j \exp\biggl\{\sum_{\ell=j+1}^k\biggl\{ -\frac{bc_0}{2\ell^{\gamma}}\biggr\}\biggr\}\exp\biggl\{\sum_{\ell=2}^{\ell_0}\biggl\{\frac{\gamma(q-1)}{\ell-1}\biggr\}\biggr\}\eqsp.
\end{align}
Therefore, \Cref{lem:summ_alpha_k} together with the elementary inequality $\rme^{-x} \leq 1 - x/2$ for $x \in (0; 1/2)$, implies that
\begin{align}
    &\sum_{j=1}^{k}\alpha_{j}^{q}\prod_{\ell=j+1}^{k}(1-\alpha_{\ell} b) \leq \alpha_k^{q-1}\sum_{j=1}^{k}\alpha_j \prod_{\ell=j+1}^k(1-\frac{b}{4}\alpha_\ell)\exp\biggl\{\gamma(q-1)(\log(\ell_0 -1) + 1)\biggr\}\biggr\}\\ &\leq \quad \quad \alpha_k^{q-1}\frac{4}{b}\biggl(\frac{4(q-1)\gamma}{bc_0}\biggr)^{2\gamma(q-1)/(1-\gamma)} \eqsp.
\end{align}
\end{proof}

\begin{lemma}
\label{lem:bound_sum_exponent}
    For any $A >0$, any $0 \leq i \leq n-1$  and any $\gamma\in(1/2, 1)$ it holds
   \begin{equation}
        \sum_{j=i}^{n-1}\exp\biggl\{-A(j^{1-\gamma} - i^{1-\gamma})\biggr\} \leq
        \begin{cases}
            1 + \exp\bigl\{\frac{1}{1-\gamma}\bigr\}\frac{1}{A^{1/(1-\gamma)}(1-\gamma)}\Gamma(\frac{1}{1-\gamma})\eqsp, &\text{ if } Ai^{1-\gamma} \leq \frac{1}{1-\gamma} \text{ and } i \geq 1\eqsp;\\
            1 + \frac{1}{A(1-\gamma)^2}i^\gamma\eqsp,  &\text{ if } Ai^{1-\gamma} >\frac{1}{1-\gamma} \text{ and } i \geq 1\eqsp;\\
            1 + \frac{1}{A^{1/(1-\gamma)}(1-\gamma)}\Gamma(\frac{1}{1-\gamma})\eqsp, &\text{ if } i=0 \eqsp. 
        \end{cases}
    \end{equation}
\end{lemma}
\begin{proof}
Note that 
\begin{align}
\sum_{j=i}^{n-1}\exp\biggl\{-A(j^{1-\gamma} - i^{1-\gamma})\biggr\} 
&\leq 1 +\exp\biggl\{A i^{1-\gamma}\biggr\}\int_{i}^{+\infty}\exp\biggl\{-Ax^{1-\gamma} \biggr\}\rmd x \\
&= 1 + \exp\biggl\{A i^{1-\gamma}\biggr\}\frac{1}{A^{1/(1-\gamma)}(1-\gamma)}\int_{Ai^{1-\gamma}}^{+\infty}\rme^{-u} u^{\frac{1}{1-\gamma}-1}\rmd u
\end{align}
      Applying \cite[Theorem 4.4.3]{gabcke2015neue}, we get 
    \begin{equation}
    \int_{Ai^{1-\gamma}}^{+\infty}\rme^{-u} u^{\frac{1}{1-\gamma}-1} \rmd u\leq
        \begin{cases}
            \Gamma(\frac{1}{1-\gamma})\eqsp, &\text{ if } Ai^{1-\gamma} < \frac{1}{1-\gamma};\\
            \frac{1}{1-\gamma}\exp\{-Ai^{1-\gamma}\} A^{\gamma/(1-\gamma)}i^\gamma\eqsp, &\text{ otherwise.}
        \end{cases}
    \end{equation}
    Combining inequities above, for $i \geq 1$ we obtain 
    \begin{equation}
        \sum_{j=i}^{n-1}\exp\biggl\{-A(j^{1-\gamma} - i^{1-\gamma})\biggr\} \leq
        \begin{cases}
            1 + \exp\bigl\{\frac{1}{1-\gamma}\bigr\}\frac{1}{A^{1/(1-\gamma)}(1-\gamma)}\Gamma(\frac{1}{1-\gamma})\eqsp, &\text{ if } Ai^{1-\gamma} < \frac{1}{1-\gamma};\\
            1 + \frac{1}{A(1-\gamma)^2}i^\gamma\eqsp, &\text{ otherwise.}
        \end{cases}
        \eqsp,
    \end{equation}
    and for $i=0$, we have
    \begin{equation}
        \sum_{j=0}^{n-1}\exp\biggl\{-A j^{1-\gamma} \biggr\} \leq 1 + \frac{1}{A^{1/(1-\gamma)}(1-\gamma)}\Gamma\biggl(\frac{1}{1-\gamma}\biggr)\eqsp.
    \end{equation}
\end{proof}

\end{document}